\documentclass[12pt]{alt2021} 

\title[Model Section for Corruption Robust RL]{A Model Selection Approach for Corruption Robust \\ Reinforcement Learning}
\usepackage{times}

\altauthor{%
 \Name{Chen-Yu Wei}\thanks{Research conducted when the author was an intern at Google Research. } \Email{chenyu.wei@usc.edu}\\
 \addr University of Southern California
 \AND
 \Name{Christoph Dann} \Email{cdann@cdann.net}\\
 \addr Google Research 
 \AND
 \Name{Julian Zimmert} 
 \Email{zimmert@google.com}\\
 \addr Google Research
}

\usepackage{xspace}
\usepackage{amsmath}
\usepackage{graphicx}
\usepackage{framed}
\usepackage{bbm}
\LinesNumbered
\usepackage{algorithm}
\usepackage{natbib}
\usepackage[T1]{fontenc}
\usepackage{mathrsfs}
\usepackage{amssymb}
\usepackage{bm}
\usepackage{makecell}
\usepackage{tabulary}
\usepackage{xcolor}
\usepackage{prettyref}
\usepackage{mathtools}
\usepackage{amsmath}
\usepackage{setspace}
\usepackage{multirow}
\usepackage{nicefrac}
\usepackage{amssymb}
\usepackage{pifont}
\definecolor{Green}{rgb}{0.13, 0.65, 0.3}
\usepackage{colortbl}

\newcommand{\jz}[1]{{\color{red}JZ: #1}}

\newcommand{\fillcell}{\cellcolor{blue!25}}

\DeclareMathOperator*{\argmax}{argmax} 

\newcommand{\Reg}{\text{\rm Reg}}

\newcommand{\typea}{type-\am\xspace}
\newcommand{\typer}{type-\rms\xspace}

\newcommand{\Rht}{\widehat{R}}
\newcommand{\Nht}{\widehat{N}}
\newcommand{\calT}{\mathcal{T}}

\newcommand{\calA}{\mathcal{A}}
\newcommand{\calB}{\mathcal{B}}
\newcommand{\calE}{\mathcal{E}}

\newcommand{\calS}{\mathcal{S}}
\newcommand{\calM}{\mathcal{M}}
\newcommand{\calI}{\mathcal{I}}

\newcommand{\calL}{\mathcal{L}}

\newcommand{\calF}{\mathcal{F}}
\newcommand{\calBpi}{\mathcal{B}_{\widehat{\pi}}}
\newcommand{\calW}{\mathcal{W}}
\newcommand{\phaseone}{\text{Phase~1}\xspace}
\newcommand{\phasetwo}{\text{Phase~2}\xspace}
\newcommand{\phasethree}{\text{Phase~3}\xspace}

\newcommand{\hatpi}{\widehat{\pi}}

\newcommand{\tilder}{\widetilde{r}}
\newcommand{\tildep}{\widetilde{p}}

\newcommand{\calR}{\mathcal{R}}

\newcommand{\hatDelta}{\widehat{\Delta}}

\newcommand{\otil}{\widetilde{\order}}

\newcommand{\poly}{\textup{poly}}
\newcommand{\DE}{\textup{dim}_\textup{DE}}
\newcommand{\E}{\mathbb{E}}

\newcommand{\order}{\mathcal{O}}

\newcommand{\thres}{\theta}

\newcommand{\pt}{p_t}

\newcommand{\calTt}{\calT_t}

\newcommand{\calX}{\mathcal{X}}

\newcommand{\one}{\mathbf{1}}
\newcommand{\inner}[1]{\langle#1\rangle}

\newcommand{\alg}{{\small\textsf{\textup{ALG}}}\xspace}

\newcommand{\init}{\text{init}}

\newcommand{\pistar}{\pi^\star}
\newcommand{\term}{\textbf{term}}

\newtheorem{assumption}{Assumption}

\newcommand{\algname}{{\small\textsf{\textup{COBE}}}\xspace}
\newcommand{\gapalgname}{{\small\textsf{\textup{G-COBE}}}\xspace}
\newcommand{\singleepoch}{{\small\textsf{\textup{BASIC}}}\xspace}
\newcommand{\twomodel}{{\small\textsf{\textup{TwoModelSelect}}}\xspace}
\usepackage{tikz}

\newcommand{\nonl}{\renewcommand{\nl}{\let\nl}}

\usepackage{prettyref}
\newcommand{\pref}[1]{\prettyref{#1}}

\newcommand{\savehyperref}[2]{\texorpdfstring{\hyperref[#1]{#2}}{#2}}
\newrefformat{eq}{\savehyperref{#1}{Eq.~\textup{(\ref*{#1})}}}
\newrefformat{eqn}{\savehyperref{#1}{Equation~\ref*{#1}}}
\newrefformat{lem}{\savehyperref{#1}{Lemma~\ref*{#1}}}
\newrefformat{lemma}{\savehyperref{#1}{Lemma~\ref*{#1}}}
\newrefformat{def}{\savehyperref{#1}{Definition~\ref*{#1}}}
\newrefformat{line}{\savehyperref{#1}{Line~\ref*{#1}}}
\newrefformat{thm}{\savehyperref{#1}{Theorem~\ref*{#1}}}
\newrefformat{corr}{\savehyperref{#1}{Corollary~\ref*{#1}}}
\newrefformat{cor}{\savehyperref{#1}{Corollary~\ref*{#1}}}
\newrefformat{sec}{\savehyperref{#1}{Section~\ref*{#1}}}
\newrefformat{subsec}{\savehyperref{#1}{Section~\ref*{#1}}}
\newrefformat{app}{\savehyperref{#1}{Appendix~\ref*{#1}}}
\newrefformat{assum}{\savehyperref{#1}{Assumption~\ref*{#1}}}
\newrefformat{ex}{\savehyperref{#1}{Example~\ref*{#1}}}
\newrefformat{fig}{\savehyperref{#1}{Figure~\ref*{#1}}}
\newrefformat{alg}{\savehyperref{#1}{Algorithm~\ref*{#1}}}
\newrefformat{rem}{\savehyperref{#1}{Remark~\ref*{#1}}}
\newrefformat{conj}{\savehyperref{#1}{Conjecture~\ref*{#1}}}
\newrefformat{prop}{\savehyperref{#1}{Proposition~\ref*{#1}}}
\newrefformat{proto}{\savehyperref{#1}{Protocol~\ref*{#1}}}
\newrefformat{prob}{\savehyperref{#1}{Problem~\ref*{#1}}}
\newrefformat{claim}{\savehyperref{#1}{Claim~\ref*{#1}}}
\newrefformat{que}{\savehyperref{#1}{Question~\ref*{#1}}}
\newrefformat{op}{\savehyperref{#1}{Open Problem~\ref*{#1}}}
\newrefformat{fn}{\savehyperref{#1}{Footnote~\ref*{#1}}}
\newrefformat{tab}{\savehyperref{#1}{Table~\ref*{#1}}}
\newrefformat{fig}{\savehyperref{#1}{Figure~\ref*{#1}}}
\newrefformat{proc}{\savehyperref{#1}{Procedure~\ref*{#1}}}

\usepackage{caption}

\newcommand{\clip}{\text{clip}}

\newcommand{\am}{{\normalfont\textsf{\scalebox{0.9}{a}}}}
\newcommand{\rms}{{\normalfont\textsf{\scalebox{0.9}{r}}}}
\newcommand{\ineff}{^*}
\newcommand{\gap}{\textsc{G}}


\newcounter{parentnumber}
\begin{document}

\maketitle



\begin{abstract}
    We develop a model selection approach to tackle reinforcement learning with adversarial corruption in both transition and reward. For finite-horizon tabular MDPs, without prior knowledge on the total amount of corruption, our algorithm achieves a regret bound of \sloppy$\otil\big(\min\{\frac{1}{\Delta}, \sqrt{T}\}+C\big)$ where $T$ is the number of episodes, $C$ is the total amount of corruption, and $\Delta$ is the reward gap between the best and the second-best policy. This is the first worst-case optimal bound achieved without knowledge of $C$, improving previous results of \cite{lykouris2019corruption, chen2021improved, wu2021reinforcement}. For finite-horizon linear MDPs, we develop a computationally efficient algorithm with a regret bound of $\otil(\sqrt{(1+C)T})$, and another computationally inefficient one with $\otil(\sqrt{T}+C)$, improving the result of \cite{lykouris2019corruption} and answering an open question by \cite{zhang2021robust}. Finally, our model selection framework can be easily applied to other settings including linear bandits, linear contextual bandits, and MDPs with general function approximation, leading to several improved or new results. 
\end{abstract}
\section{Introduction}\label{sec: intro}
Reinforcement learning (RL) studies how an agent learns to behave in an unknown environment with reward feedback. The environment is often modeled as a Markov decision process (MDP). In the standard setting, the MDP is assumed to be static, i.e., the state transition kernel and the instantaneous reward function remain fixed over time. Under this assumption, numerous computationally and statistically efficient algorithms with strong theoretical guarantees have been developed \citep{jaksch2010near, lattimore2012pac, dann2015sample, azar2017minimax, jin2018q, jin2020provably}. However, these guarantees might break completely if the transition or the reward is corrupted by an adversary, even if the corruption is limited to a small fraction of rounds. 


To model adversarial corruptions in MDPs, a framework called \emph{adversarial MDP} has been extensively studied. In adversarial MDPs, the adversary is allowed to choose the reward function arbitrarily in every round, while keeping the transition kernel fixed \citep{neu2010online, gergely2010online, dick2014online, rosenberg2019online, rosenberg2020stochastic,  jin2020learning, neu2020online, lee2020bias, chen2021finding, he2021nearly,  luo2021policy}. Under this framework, strong sub-linear regret bounds can be established, which almost match the bounds for the fixed reward case. Notably, \cite{jin2020simultaneously, jin2021best} developed algorithms that achieve near-minimax regret bound in the adversarial reward case, while 
preserving refined instance-dependent bounds in the static case, showing that adversarial reward can be handled almost without price. 

The situation becomes very different when the transition kernel can also be corrupted. It is first shown by \cite{abbasi2013online} that achieving sub-linear regret in this setting is \emph{computationally hard}, and recently enhanced by \cite{tian2021online} showing that it is even \emph{information-theoretically hard}. To establish meaningful guarantees, previous work aims to achieve a regret bound that smoothly degrades with the amount of corruption, and thus the learner can still behave well when only a small fraction of data is corrupted. When the total amount corruption is given as a prior knowledge to the learner, \cite{wu2021reinforcement} designed an algorithm with a regret upper bound that scales optimally with the amount of corruption. However, this kind of prior knowledge is rarely available in practice. When the total amount of corruption is unknown, efforts were made by \cite{lykouris2019corruption, chen2021improved, cheung2020reinforcement,  wei2021nonstationary, zhang2021robust} to obtain similar guarantees. Unfortunately, all their bounds scale sub-optimally in the amount of corruption. Therefore, the following question remains open:
\textit{When both reward and transition can be corrupted, how can the learner achieve a regret bound that has optimal dependence on the unknown amount of corruption? } 
We address this open problem by designing an efficient algorithm with the desired worst-case optimal bound. Specifically, in tabular MDPs, our regret bound scales with $\sqrt{T}+C$ where $T$ is the number of rounds and $C$ is the total amount of corruption (omitting dependencies on other quantities). This matches the lower bound by \cite{wu2021reinforcement}. In contrast, the bounds obtained by \cite{lykouris2019corruption} and \cite{chen2021improved} are $(1+C)\sqrt{T}+C^2$ and $\sqrt{T}+C^2$ respectively, which are non-vacuous only when $C\leq \sqrt{T}$, a rather limited case. The bounds obtained by \cite{cheung2020reinforcement, wei2021nonstationary, zhang2021corruption} are $(1+C)^{\nicefrac{1}{4}}T^{\nicefrac{3}{4}}$, $\sqrt{T} + C^{\nicefrac{1}{3}}T^{\nicefrac{2}{3}}$, and $\sqrt{(1+C)T}$ respectively. Although these bounds are meaningful for all $C
\leq T$, the dependence on $C$ is \emph{multiplicative} to $T$, which is undesirable.  
 
For tabular MDPs, we further show that the bound can be improved to $\min\{\frac{1}{\Delta}, \sqrt{T}\} + C$, where $\Delta$ is the gap between the expected reward of the best and the second-best policies. This kind of refined instance-dependent regret bound is also established by \cite{lykouris2019corruption} and \cite{chen2021improved}. The bound of \cite{lykouris2019corruption} is $(1+C)\min\{\gap, \sqrt{T}\} + C^2$ for some gap-complexity $\gap\leq \frac{1}{\Delta}$, while \cite{chen2021improved} obtained $\min\{\frac{1}{\Delta}, \sqrt{T}\}+C^2$. It is left as an open question whether the best-of-all-world bound $\min\{\gap, \sqrt{T}\} + C$ is achievable. 

Our method is based on the framework of \emph{model selection} \citep{agarwal2017corralling, foster2019model, arora2021corralling,  abbasi2020regret,  pacchiano2020regret, pacchiano2020model}. In model selection problems, the learner is given a set of \emph{base algorithms}, each with an underlying model or assumption for the world. However, the learner does not know in advance which model fits the real world the best. The goal of the learner is to be comparable to the best base algorithm in hindsight. In our case, a model of the world corresponds to a hypothetical amount of corruption $C$; for a given $C$, there are algorithms with near-optimal bounds (e.g., \cite{wu2021reinforcement}) that can serve as base algorithms. Therefore, the problem of handling unknown corruption can be cast as a model selection problem. To get the bound of $\sqrt{T}+C$, we adopt the idea of \emph{regret balancing} similar to those of \cite{abbasi2020regret,  pacchiano2020regret}, while to get $\min\{\frac{1}{\Delta}, \sqrt{T}\}+C$, we develop another novel \emph{two-model selection} algorithm to achieve the goal (see \pref{sec: gap bound}). 

\paragraph{Extensions to linear and general function approximation} Our model selection framework can be readily extended to the cases of linear contextual bandits and linear MDPs.
However, in \pref{app: non-robust}, we demonstrate that even with at most $C$ corrupted rounds, a straightforward extension of standard algorithms (i.e., OFUL \citep{abbasi2011improved}, LSVI-UCB \citep{jin2020provably}) results in an overall regret of $\Omega(\sqrt{CT})$, a sharp contrast with the $\order(\sqrt{T}+C)$ bound in tabular MDPs. We find that the $\order(\sqrt{T}+C)$ bound can indeed be achieved efficiently in the non-contextual case (i.e., linear bandits with a fixed action set) by using the Phased Elimination (PE) approach developed by \cite{lattimore2020learning, bogunovic2021stochastic}. 
We achieve the same bounds for linear contextual bandits and linear MDPs, but we resort to the idea of \cite{zhang2021varianceaware}, who deal with linear models through a sophisticated and computationally inefficient clipping technique. Note that the original purpose of \cite{zhang2021varianceaware} is to get a variance-reduced bound for linear contextual bandits and a horizon-free bound for linear mixture MDPs, which are very different from our goal here. Their idea being applicable to improve robustness against corruption is surprising and of independent interest. The fact that additive dependence on $C$ is possible under linear settings (though computationally inefficient) partially answers an open question by \cite{zhang2021robust}. 

We further extend our framework to general function approximation settings. We consider the class of MDPs that have low Bellman-eluder dimension \citep{jin2021bellman}, and derive a corruption-robust version of their algorithm (GOLF). The algorithm achieves a regret bound of $\order(\sqrt{(1+C)T})$. Whether the bound of $\order(\sqrt{T}+C)$ is possible is left as an open problem.

In \pref{tab: table of bounds}, we compare our bounds with those in previous works (omitting dependencies other than $C$ and $T$). Note that $C^\textsf{r}$ is best interpreted as $\sqrt{CT}$ in the notation of prior work. More thorough comparisons to the related works are provided in \pref{app: related work}, and more precise bounds (including dependencies other than $C$ and $T$) are provided in \pref{app: base algorithms}. 

\renewcommand{\arraystretch}{1.2}
\begin{table}[t]
\small
    \centering
    \caption{\small $\ineff$ indicates computationally inefficient algorithms. $\gap$ is the \textsf{GapComplexity} defined in \citep{simchowitz2019non}; $\Delta$ is the gap between the expected reward of the best and second-best policy. It holds that  $\gap\leq \frac{1}{\Delta}$. $C^{\textsf{a}}=\sum_t c_t$ and $C^{\textsf{r}}=\sqrt{T\sum_t c_t^2}$, where $c_t$ is the amount of corruption in round $t$. By definition, $C^{\textsf{a}}\leq C^{\textsf{r}}\leq \min\{\sqrt{C^{\textsf{a}} T}, T\max_{t} c_t\}$.  $C^{\textsf{a}}$ is the standard notion of corruption in the literature. \\
    $^\dagger$: The bound reported in \citep{jin2021best} is $\min\{\gap + \sqrt{\gap C^{\textsf{a}}}, \sqrt{T}\}$ under a different definition of regret. \\
    $^\sharp$: Linearized corruption restricts that the corruption on action $a$ equals to $c^\top a$ for some vector $c$ shared among all actions. 
    }\label{tab: table of bounds}
    \begin{tabular}{|c|c|c|c|}
        \hline
        Setting & Algorithm & $\Reg(T)$ in $\otil(\cdot)$ & Restrictions \\
        \hline
        \multirow{4}{*}{Tabular MDP} & \citep{lykouris2019corruption} & $(1+C^{\am})\min\{\gap, \sqrt{T}\} + (C^\am)^2$ & \\
        \cline{2-4} 
         & 
        \citep{chen2021improved} & $\min\{\frac{1}{\Delta}, \sqrt{T}\}+ (C^{\am})^2$ & $\ineff$ \\
        \cline{2-4} 
         & \citep{jin2021best}$^\dagger$ 
        & $\min\{\gap, \sqrt{T}\} + C^\am$  & \makecell{
            \vspace*{-2.5pt}
            \scalebox{0.9}{only for corruption} \\
            \scalebox{0.9}{in reward}  } \\   
        \cline{2-4}
         & \fillcell \gapalgname + UCBVI& \fillcell $\min\{\frac{1}{\Delta}, \sqrt{T}\}+ C^{\am}$ & \fillcell \\
        \hline
        \multirow{5}{*}{Linear bandit}
         & \citep{li2019stochastic} & $\frac{1}{\Delta^2} + \frac{C^\am}{\Delta}$ & \\
        \cline{2-4} 
         & 
        \citep{bogunovic2020corruption} & $(1+C^\am)  \sqrt{T}$ &  \\
         \cline{2-4} 
         & 
        \citep{bogunovic2021stochastic} & $\sqrt{T} + (C^\am)^2$ &  \\
        \cline{2-4}
         &  \citep{lee2021achieving} & $\min\{\frac{1}{\Delta}, \sqrt{T}\} + C^\am$ & \makecell{
             \vspace*{-2.5pt}
             \scalebox{0.9}{only for linearized} \\  \scalebox{0.9}{corruption$^\sharp$} 
        } \\
        \cline{2-4} 
         & \fillcell \gapalgname + PE & \fillcell $\min\{\frac{1}{\Delta}, \sqrt{T}\} + C^\am$ & \fillcell \\  
        \hline
        \multirow{3}{*}{\makecell{Linear contextual \\ bandit}} & \citep{foster2021adapting} & $\sqrt{T} + C^\rms$ & \\
        \cline{2-4}
         & \fillcell \algname + OFUL&\fillcell $\sqrt{T} + C^\rms$ & \fillcell \\
        \cline{2-4}
         & \fillcell \algname + VOFUL& \fillcell $\sqrt{T} + C^\am$ & \fillcell $\ineff$ \\
        \hline 
        \multirow{4}{*}{Linear MDP} & \citep{lykouris2019corruption} & $\sqrt{T} + (C^{\am})^2\sqrt{T}$ & \\ 
        \cline{2-4}
         & \citep{wei2021nonstationary} & $\sqrt{T} + (C^{\am})^{\nicefrac{1}{3}}T^{\nicefrac{2}{3}}$ &  \\
        \cline{2-4}
         & \fillcell \algname + LSVI-UCB& \fillcell $\sqrt{T} + C^\rms$ & \fillcell \\
        \cline{2-4}
         & \fillcell \algname + VARLin& \fillcell $\sqrt{T} + C^\am$ & \fillcell $\ineff$ \\
        \hline
        \makecell{Low BE-dimension} & \fillcell \algname + GOLF& \fillcell $\sqrt{T}+C^\rms$ & \fillcell $\ineff$\\
        \hline
    \end{tabular}
    
    \label{tab:my_label}
\end{table}

\section{Problem Setting}\label{sec: problem setting}

We consider a general decision making framework that covers a wide range of problems. We first describe the \emph{uncorrupted} setting. The learner is given a policy set $\Pi$ and a context set $\calX$. Ahead of time, the environment decides a context-to-expected-reward mapping $\mu^\pi: \calX \rightarrow [0,1]$ for all $\pi\in\Pi$, which are hidden from the learner. In each round $t=1,\ldots, T$, the environment first arbitrarily generates a context $x_t\in\calX$, and generates a noisy reward $r^\pi_t\in[0,1]$ for all $\pi$ such that $\E[r^\pi_t] = \mu^\pi(x_t)$. The context $x_t$ is revealed to the learner. Then the learner chooses a policy $\pi_t$, and receives $r_t\triangleq r_t^{\pi_t}$. The goal of the learner is to minimize the regret defined as 
\begin{align}
    \Reg(T) = \max_{\pi\in \Pi} \sum_{t=1}^T \big(\mu^{\pi}(x_t) - \mu^{\pi_t}(x_t)\big).  \label{eq: pseudo}
\end{align}

In the \emph{corrupted} setting, the protocol is similar, but in each round $t$, an adversary can change the context-to-expected-reward mapping from $\mu^\pi(\cdot)$ to $\mu_t^{\pi}(\cdot)$. Then $r_t^{\pi}$ is drawn such that $\E[r_t^{\pi}]=\mu^\pi_t(x_t)$. We assume that $\mu_t^{\pi}(x_t)$ and $r^{\pi}_t$ still lie in $[0,1]$.  As before, the learner observes $x_t$, chooses $\pi_t$, and receives $r_t=r_t^{\pi_t}$. The goal of the learner remains to minimize the regret defined in \pref{eq: pseudo} (notice that it is defined through the uncorrupted $\mu$). 
The adversary we consider falls into the category of an \emph{adaptive adversary}, an adversary that can decide the corruption in round $t$ based on the history up to round $t-1$. 

We consider the realizable setting, where the following assumption holds: 
\begin{assumption}\label{assum: realizable}
    There exists a policy $\pistar\in\Pi$ such that $\mu^{\pistar}(x_t)\geq \mu^\pi(x_t)$ for all $t$ and all $\pi\in\Pi$. 
\end{assumption}

Below, we instantiate this framework to linear (contextual) bandits and episodic MDPs. For each setting, we define a suitable quantity $c_t$ to measure the amount of corruption in round $t$. It always holds that $\max_{\pi\in\Pi}|\mu^{\pi}(x_t) - \mu_t^{\pi}(x_t)|\leq c_t$, but $c_t$ might be strictly larger than $\max_{\pi\in\Pi}|\mu^{\pi}(x_t) - \mu_t^{\pi}(x_t)|$ in some cases. 

\paragraph{Linear contextual bandits }
In linear contextual bandits, the policy set can be identified as a bounded parameter set $\calW\subset \mathbb{R}^d$ in which the true underlying parameter $w^\star$ lies, and the ``context'' can be identified as the ``action set'' in each round. In the uncorrupted setting, in each round $t$, the environment first generates an action set $\calA_t\subset \mathbb{R}^d$ (the context). Then each policy $w\in\calW$ is associated with an action $a^w(\calA_t)=\argmax_{a\in\calA_t}\inner{w, a}$. The expected reward of policy $w$ under context $\calA_t$ is then given by $\mu^{w}(\calA_t)=\inner{w^\star, a^w(\calA_t)}\in[0,1]$ for all $w$.
Clearly, $w^\star$ is a policy that satisfies \pref{assum: realizable}. 

In the corrupted setting, for any $t$ and $w$, the adversary can choose $\mu_t^w(\cdot)$ to be an arbitrary mapping from $\calA_t$ to $[0,1]$, and $r^w_t$ is generated such that $\E[r^w_t]=\mu^w(\calA_t)$ and $r^w_t\in[0,1]$. We define $c_t=\max_{w\in\calW}|\mu^w(\calA_t) - \mu^{w}_t(\calA_t)|$. 

\paragraph{Linear bandits } Linear bandits can be viewed as a special case of linear contextual bandits with a fixed action set $\calA_t=\calA$ for all $t$. In this case, since every policy chooses the same action in every round, a more direct formulation is to identify the policy set as the action set $\calA$ and ignore the context. The expect reward of policy/action $a\in\calA$ is given by $\mu^a=\inner{w^\star, a}\in[0,1]$ for some unknown $w^\star$. Note that the action $\argmax_{a}\inner{w^\star, a}$ satisfies \pref{assum: realizable}. In the corrupted setting, $\mu_t^a$ can be set to an arbitrary value in $[0,1]$, and $r_t^a$ is generated such that $\E[r_t^a]=\mu_t^a$ and $r_t^a\in[0,1]$. We define $c_t=\max_{a\in\calA}|\mu^a-\mu_t^a|$. 

\paragraph{Episodic Markov decision processes }
An episodic MDP is associated with a state space $\calS$, an action space $\calA$, a number of layers $H$, a transition kernel $p: \calS\times \calA\rightarrow \Delta_{\calS}$, and a reward function $\sigma: \calS\times \calA\rightarrow [0,\frac{1}{H}]$. A policy $\pi=\{\pi_h: \calS\rightarrow \calA, \ h= 1, 2, \ldots, H\}$ consists of mappings $\calS\rightarrow \calA$ for each layer that specifies which action it takes in each state in that layer. The context is identified as the ``initial state.''

In the uncorrupted setting, in each round $t$, the environment arbitrarily generates an initial state $s_{t,1}\in \calS$ (the context). Then the learner decides a policy $\pi_t=\{\pi_{t,h}\}_{h=1}^H$, and interacts with the environment for $H$ steps starting from $s_{t,1}$. On the $h$-th step, she chooses an action $a_{t,h}=\pi_{t,h}(s_{t,h})$, observes a noisy reward $\sigma_{t,h}\in[0,\frac{1}{H}]$ with $\E[\sigma_{t,h}]=\sigma(s_{t,h}, a_{t,h})$, and transitions to the next state $s_{t,h+1}\sim p(\cdot|s_{t,h}, a_{t,h})$.\footnote{Our setting is a scaled version of the standard episodic MDP setting where all rewards are scaled by $1/H$. This scaling does not affect the difficulty of the problem but allows us to unify the presentation with our other settings.} The round ends right after the learner transitions to state $s_{t,H+1}$. With this procedure, the expected reward of policy $\pi$ given the initial state $s$ can be represented as
\begin{align}
    \mu^\pi(s) = \E\left[\sum_{h=1}^H \sigma(s_h, a_h)~\bigg|~s_1=s, \  a_h=\pi_h(s_h),\  s_{h+1}\sim p(\cdot|s_h, a_h),\ \ \forall h=1, \ldots, H \right],   \label{eq: V function}
\end{align}
and $r_t^\pi$ is a realized reward of policy $\pi$ in round $t$, which satisfies $\E[r_t^\pi]=\mu^{\pi}(s_{t,1})$.  

In the corrupted setting, in round $t$, we allow the adversary to change $p$ and $\sigma$ to $\pt$ and $\sigma_t$ respectively. The corrupted expected value $\mu^\pi_t(s)$ is defined similarly to \pref{eq: V function} but with $p$ and $\sigma$ replaced by $p_t$ and $\sigma_t$. We assume that after corruption, $\sigma_{t,h}$ (whose expectation is $\sigma_t(s_{t,h}, a_{t,h})$) still lies in $[0,\frac{1}{H}]$. 

To measure the amount of corruption in round $t$, we define the Bellman operators $\calT: \mathbb{R}^{\calS}\rightarrow \mathbb{R}^{\calS\times \calA}$ under the uncorrupted MDP and $\calT_t$ under the corrupted MDP as
$
    (\calT V)(s,a) \triangleq \sigma(s,a) + \E_{s'\sim p(\cdot|s,a)}[V(s')]$ and $(\calTt V)(s,a) \triangleq \sigma_t(s,a) + \E_{s'\sim \pt(\cdot|s,a)}[V(s')]
$
for $V\in \mathbb{R}^\calS$. 
Then the amount of corruption in round $t$ is defined as 
$c_t \triangleq H \cdot \sup_{s,a} \sup_{V\in [0,1]^\calS} \left|(\calT V - \calTt V)(s,a)\right|$.

\paragraph{Additional note on corruption} In all the above settings, we assume that the corruption is bounded. It holds that $c_{t}\leq c_{\max}$ for some $c_{\max}$ in all $t$. For linear (contextual) bandits, we can set $c_{\max}=1$, while for episodic MDPs, we can set $c_{\max}=2H$. While we assume bounded corruption to keep the exposition clean, our algorithm can actually handle more general scenarios. For example, for linear contextual bandits, we can handle the case where $\mu^{\pi}(\cdot)\in[0,1]$, but $\mu^{\pi}_t(\cdot)$ can be arbitrary, and $r^{\pi}_t - \mu^{\pi}_t(x_t)$ is zero-mean and 1-sub-Gaussian.
For this case, since $r^{\pi}_t - \mu^{\pi}_t(x_t)$ is bounded between $\pm c'\log(1/\delta)$ for some absolute constant $c'$ with high probability (i.e., with probability $1-\order(\delta)$), if we receive a reward $r_t$ that is outside $[-c'\log(1/\delta), 1+c'\log(1/\delta)]$, then with high probability it is caused by corruption. The learner only needs to project the reward back to this range. This essentially reduces the problem to bounded corruption case. 


\paragraph{Other notations} We use $[u, v]$ to denote $\{u, u+1, \ldots, v\}$, and $[u]$ to denote $\{1, 2, \ldots, u\}$. The notations $\otil(\cdot), \widetilde{\Theta}(\cdot)$ hide poly-logarithmic factors. 





\subsection{Two Ways to Compute Aggregated Corruption}\label{sec: two ways}  
In previous works of corruption-robust RL, the total corruption is defined as $C=\sum_{t=1}^T c_t$, where $c_t$ is the per-round corruption defined above. In  \pref{sec: intro} and \pref{app: related work}, we also adopt this definition when comparing with previous works. However, to unify the analysis under different settings, we introduce another notion of total regret defined as  $\sqrt{T\sum_{t=1}^T c_t^2}$\,. To distinguish them, we denote $C^\am = \sum_{t=1}^T c_t$ and $C^\rms=\sqrt{T\sum_{t=1}^T c_t^2}$\,, for that $C^\am$ is $T$ times the \emph{arithmetic mean} of $c_t$'s, while $C^\rms$ is $T$ times the \emph{root mean square} of $c_t$'s.   
By defining $C^{\rms}$, we are able to recover the bounds in the ``model misspecification'' literature, in which the regret bound is often expressed through $T\max_t c_t$, which is an upper bound of $C^\rms$ (see \pref{tab: table of bounds} and \pref{app: related work} for more details). 
We further define 
$C_t^{\am} \triangleq \sum_{\tau=1}^t c_\tau$ and $C^{\rms}_t \triangleq \sqrt{t\sum_{\tau=1}^t c_\tau^2}$\,. 



\section{Gap-Independent Bounds via Model Selection}\label{sec: main alg}
In this section, we develop a general corruption-robust algorithm based on model selection. The regret bound we achieve is of order $\otil(\sqrt{T}+C)$, where $C$ is either $C^\am$ or $C^\rms$ (see \pref{tab: table of bounds} for the choices in different settings). Model selection approaches rely on a meta algorithm learning over a set of base algorithms. We first specify the properties that each base algorithm should satisfy: 
\begin{assumption}[base algorithm, with either $C_t\triangleq C_t^\am$ or $C_t\triangleq C_t^\rms$]\label{assum: regret}
    $\alg$ is an algorithm that takes as input a time horizon $T$, a confidence level $\delta$, and a hypothetical corruption level $\thres$. \alg ensures the following: with probability at least $1-\delta$, for all $t\leq T$ such that $C_{t}\leq \thres$, it holds that
\begin{align*}
    \sum_{\tau=1}^t(r_\tau^{\pistar}-r_\tau) \leq \calR(t,\thres) 
\end{align*}
for some function $\calR(t,\theta)$. 
Without loss of generality, we assume that $\calR(t,\theta)$ is non-decreasing in both $t$ and $\theta$, and that $\calR(t,\theta)\geq \thres$. 
\end{assumption}

If a base algorithm satisfies \pref{assum: regret} with $C_t\triangleq C_t^\am$, we call it a \typea base algorithm, while if $C_t\triangleq C_t^\rms$, we call it a \typer. 
Base algorithms are essentially corruption-robust algorithms that require the prior knowledge of the total corruption. Therefore, the algorithms developed by \citet[Appendix B]{lykouris2019corruption} or \citet{wu2021reinforcement} can be readily used as our base algorithms. 
For example, for tabular MDPs, a variant of the UCBVI algorithm \citep{azar2017minimax} satisfies \pref{assum: regret} with $C_t\triangleq C_t^\am$ and $\calR(t,\thres)=\poly(H, \log(SAT/\delta))(\sqrt{SAt} + S^2A + SA\thres)$; for linear MDPs, a variant of the LSVI-UCB algorithm \citep{jin2020provably} satisfies \pref{assum: regret} with $C_t\triangleq C_t^\rms$ and $\calR(t, \thres)=\poly(H, \log(dT/\delta))(\sqrt{d^3 t} + d\thres)$. More examples are provided in \pref{app: base algorithms}. 

A base algorithm with a higher hypothetical corruption level $\thres$ is more \emph{robust}, but incurs more regret overhead. In contrast, base algorithms with lower hypothetical corruption level introduce less overhead, but have higher possibility of mis-specifying the amount of corruption. When the true total corruption is unknown, just running a single base algorithm with a fixed $\thres$ is risky either way. 

The idea of our algorithm is to simultaneously run multiple base algorithms (in each round, sample one of the base algorithms and execute it), each with a different hypothesis on the total amount of corruption.
This idea is also used by \cite{lykouris2019corruption}. Intuitively, if two base algorithms have a valid hypothesis for the total corruption (i.e., their hypotheses upper bound the true total corruption), then the one with smaller hypothesis should learn faster than the larger one because its hypothesis is closer to the true value, and incurs less overhead.  Therefore, if at some point we find that the average performance of a base algorithm with a smaller hypothesis is significantly worse than that of a larger one, it is an evidence that the former has mis-specified the amount of corruption. If this happens, we simply stop running this base algorithm.  

There are two key questions to be answered. First, what distribution should we use to select among the base algorithms? Second, given this distribution, how should we detect mis-specification of the amount of corruption by comparing the performance of base algorithms? In \pref{sec: single epoch analysis}, we answer the second question. The first question will be addressed in \pref{sec: elimination alg} and \pref{sec: gap bound} slightly differently depending on our target regret bound.

\begin{algorithm}[t]
\SetKwIF{If}{ElseIf}{Else}{if}{}{else if}{else}{end if}%
    \caption{\textbf{B}ase \textbf{A}lgorithms run \textbf{S}imultaneously with m\textbf{I}s-specification \textbf{C}heck (\singleepoch)} \label{alg: regret balancing all}
    
    \textbf{input}:  base algorithm $\alg$ satisfying \pref{assum: regret},~~$L\in [T]$,~~$k\in [k_{\max}]$ where $k_{\max} \triangleq \lceil \log_2 (c_{\max}L) \rceil$,~~ $\delta\in(0,1)$,~ 
    and a distribution $\alpha=(\alpha_k ,\alpha_{k+1}, \ldots, \alpha_{k_{\max}})$ satisfying: 
    \vspace*{-2mm}
    \begin{align*}
        \alpha_k \geq \alpha_{k+1} \geq \cdots \geq \alpha_{k_{\max}} > 0 \qquad \text{and}\qquad \sum_{i=k}^{k_{\max}} \alpha_i = 1. 
    \end{align*}
    \For{$i=k, \ldots, k_{\max}$}{ Initiate an instance of $\alg$ with inputs $T, \delta$,
    and $\thres$ chosen as below:  
    \begin{align}
      \thres_i\triangleq
      \begin{cases}
          1.25\cdot\alpha_i2^{i}+21c_{\max}\log(T/\delta)  &\text{if $\alg$ is \typea}\\
          1.25\cdot\alpha_i2^{i} + 8c_{\max}\sqrt{\alpha_i L\log(T/\delta)} + 21 c_{\max}\log(T/\delta) &\text{if $\alg$ is \typer}
      \end{cases} 
      \label{eq: choice of thres}
    \end{align}
    (We call this instance $\alg_i$.)
    } 
    
    \ \\
    \For{$t=1,\ldots, L$}{
         Random pick an sub-algorithm $i_t\sim \alpha$, 
         receive the context $x_t$, and use $\alg_{i_t}$ to output $\pi_t$. 
         
         Execute $\pi_t$, receive feedback, and perform  update on $\alg_{i_t}$. 
         
         Define $N_{t,i}\triangleq \sum_{\tau=1}^t \one[i_\tau=i]$, $R_{t,i}\triangleq \sum_{\tau=1}^t \one[i_\tau=i]r_\tau$. \\
         \If{$\exists i,j\in[k, k_{\max}]$, $i<j$, \textnormal{such that}
         \vspace*{-3mm}
         \begin{align}
         \frac{R_{t,i}}{\alpha_i} +   \frac{\calR(N_{t,i}, \thres_i)}{\alpha_i} < \frac{R_{t,j}}{\alpha_j} - 8\left(\sqrt{\frac{t\log(T/\delta)}{\alpha_j}} + \frac{\log(T/\delta)+ \thres_j}{\alpha_j}\right),    \label{eq: terminate condition 1}
         \end{align}
         \vspace{-6mm}}
         {
         \Return \textbf{false}.   
         }
    }
    \textbf{return} \textbf{true}. 
\end{algorithm}

\subsection{Single Epoch Algorithm}\label{sec: single epoch analysis}

In this section, we analyze \singleepoch (\pref{alg: regret balancing all}), a building block of our final algorithms. In \singleepoch, the distribution over base algorithms is fixed and given as an input ($\alpha$ in \pref{alg: regret balancing all}). Other inputs include: a length parameter $L$ that specifies the maximum number of rounds (the algorithm might terminate before finishing all $L$ rounds though) and an index $k\in [k_{\max}]$ ($k_{\max}$ is defined in \pref{alg: regret balancing all})
that specifies the smallest index of base algorithms (the base algorithms are indexed by $k, k+1, \ldots, k_{\max}$). 

Below, we sometimes unify the statements for the two definitions of total corruption (see \pref{sec: two ways}). The notations $(C, C_t)$ refer to $(C^\am, C_t^\am)$ if the base algorithm is \typea, and refer to $(C^\rms, C_t^\rms)$ if it is \typer. We will explicitly write the superscripts if we have to distinguish them.

The base algorithm with index $i\in[k, k_{\max}]$ (denoted as $\alg_i$) hypothesizes that the total corruption $C$ is upper bounded by $2^i$. 
We say $\alg_i$ is \emph{well-specified} at round $t$ if $C_t\leq 2^i$; otherwise we say it is \emph{mis-specified} at round $t$.  Naively, we might want to set the $\thres$ parameter of $\alg_i$ to $2^i$. However, we can actually set it to be smaller to reduce the overhead, as explained below. 
Since each base algorithm is sub-sampled according to the distribution $\alpha$, the total corruption experienced by $\alg_i$ in $[1,t]$ is only roughly $\sum_{\tau\leq t} \alpha_i c_\tau\leq \alpha_i C^\am$ or $\sqrt{(\alpha_i t)\sum_{\tau\leq t} \alpha_i c_\tau^2}\leq \alpha_i C^{\rms}$ (for \typea and \typer base algorithms respectively). This means that $\alg_i$, which hypothesizes a total corruption of $2^i$, only needs to set the $\thres$ parameter in \pref{assum: regret} to roughly $\alpha_i2^i$, instead of $2^i$. Our choice of $\thres_i$ in \pref{eq: choice of thres} is slightly larger than $\alpha_i 2^i$ to accommodate the randomness in the sampling procedure.

Besides performing sampling over base algorithms, \singleepoch also compares the performance any two base algorithms using \pref{eq: terminate condition 1}. 
If all base algorithms hypothesize large enough corruption, then all of them enjoy the regret bound specified in \pref{assum: regret} in the subset of rounds they are executed. In this case, we can show that with high probability, the termination condition \pref{eq: terminate condition 1} will not hold. This is formalized in \pref{lem: istar = k case}. 

\begin{lemma}\label{lem: istar = k case}
   With probability at least $1-\order(k_{\max}\delta)$, 
   the termination condition \pref{eq: terminate condition 1} of the \singleepoch algorithm, does not hold in any round $t$, such that $C_{t}\leq 2^k$.
\end{lemma}
In other words, \pref{eq: terminate condition 1} is triggered only when $C_t > 2^k$, i.e., $\alg_k$ is mis-specified at round $t$. Once this happens, the \singleepoch algorithm terminates. 
%
Checking condition \pref{eq: terminate condition 1} essentially ensures that the quantity $\frac{R_{t,i}}{\alpha_i}$ of all base algorithms remain close. Notice that at all $t$, there is always a well-specified base algorithm $i^\star$ with $C_t\leq 2^{i^\star}$ which enjoys the regret guarantee of \pref{assum: regret}. Therefore, $\frac{R_{t,i^\star}}{\alpha_{i^\star}}$ is not too low, and thus, testing condition \pref{eq: terminate condition 1} prevents $\frac{R_{t,i}}{\alpha_i}$ of any $i$ from falling too low. This directly controls the performance of every base algorithms before termination. The following lemma bounds the learner's cumulative regret at termination. 
\begin{lemma}\label{lem: general regret}
Let $L_0\leq L$ be the round at which \singleepoch terminates, and let $i^\star$ be the smallest $i\in[k, k_{\max}]$ such that $C_{L_0}\leq 2^{i}$. Then with probability at least $1-\order(k_{\max}\delta)$, 
   \begin{align*}
    \sum_{t=1}^{L_0} (r_{t}^{\pi^\star} - r_t) \leq \sum_{i=k}^{k_{\max}} \calR(N_{L_0,i}, \thres_i) + \otil\left(\one\left[i^\star > k\right]\left(\sqrt{\frac{L_0}{\alpha_{i^\star}}} + \frac{\calR(N_{L_0,i^\star}, \theta_{i^\star})}{\alpha_{i^\star}}\right)\right).
\end{align*} 
where $N_{L_0, i}$ is the total number of rounds $\alg_i$ was played.
\end{lemma}

\subsection{Corruption-robust Algorithms with $\sqrt{T}+C^{\textsf{a}}$ or $\sqrt{T}+C^{\textsf{r}}$ Bounds}\label{sec: elimination alg}  
Next, we use \singleepoch to build a corruption-robust algorithm with a regret bound of either $\sqrt{T}+C^\am$ or $\sqrt{T}+C^{\rms}$ without prior knowledge of $C^\am$ or $C^\rms$. The algorithm is called \algname and presented in \pref{alg: meta alg}. We consider base algorithms with the following concrete form of $\calR(t,\thres)$: 
\begin{align}
    \calR(t,\thres) = \sqrt{\beta_1 t} + \beta_2 \theta + \beta_3  \label{eq: typeone}
\end{align}  
for some $\beta_1, \beta_2, \beta_3 \geq 1$. \algname
starts with $k=k_{\init}$ (defined in \pref{alg: meta alg}) and runs \singleepoch with inputs $k$ and $L=T$ and the following choice of $\{\alpha_i\}_{i=k}^{k_{\max}}$:   
\begin{align}
    \alpha_i = 
    \begin{cases}
        2^{k-i-1}  &\text{for\ } i>k, \\
        1-\sum_{i=k+1}^{k_{\max}}\alpha_i &\text{for\ } i=k.
    \end{cases}  \label{eq: type 2 alpha}
\end{align} 
Whenever the subroutine \singleepoch terminates before $T$, we eliminate the $\alg_k$ and start a new instance of \singleepoch with $k$ increased by $1$ (see the for-loop in \algname). This is because as indicated by \pref{lem: istar = k case}, early termination implies that $\alg_k$ mis-specifies the amount of corruption. 

Notice that $2^{k_{\init}}$ is roughly of order $\sqrt{T}$, i.e., we start from assuming that the total amount of corruption is $\sqrt{T}$. This is because we only target the worst-case regret rate of $\sqrt{T}+C$ here, so refinements for smaller corruption levels $C\leq \sqrt{T}$ do not improve the asymptotic bound. Our choice of $\alpha_i$ makes $\alpha_i 2^i\approx 2^k$ for all $i$, and this further keeps the magnitudes of $\calR(N_{t,i}, \thres_i)$ of all $i$'s  roughly the same. This conforms with the regret balancing principle by \cite{abbasi2020regret, pacchiano2020regret}, as well as the sub-sampling idea of \cite{lykouris2019corruption}. This makes the bound of the model selection algorithm only worse than the best base algorithm by a factor of $\order(k_{\max})=\otil(1)$ if all base algorithms are well-specified.  

In the following theorem, we show guarantees of \algname for both $C\triangleq C^\am$ and $C\triangleq C^\rms$. The proof essentially plugs the choices of parameters into \pref{lem: general regret}, and sum the regret over epochs. 

\SetKwBlock{Repeat}{repeat}{end}
\begin{algorithm}[t]
    \caption{\textbf{CO}rruption-robustness through \textbf{B}alancing and \textbf{E}limination (\algname)}  \label{alg: meta alg}
    \nl \textbf{input}: base algorithm $\alg$ satisfying \pref{assum: regret} with the form specified in \pref{eq: typeone}. \\
    \nl \textbf{define}: $Z\triangleq c_{\max}$ if $\alg$ is \typea, and $Z\triangleq c_{\max}\sqrt{T}$ if $\alg$ is \typer. \\
    \nl $k_{\init}\triangleq \max\left\{ \left\lceil \log_2  \frac{\sqrt{\beta_1 T} + \beta_2 Z + \beta_3}{\beta_2} \right\rceil, \ 0\right\}$ with $\beta_1, \beta_2, \beta_3$ defined in \pref{eq: typeone}.  
    \ \\
    \nl \For{$k=k_{\init}, \ldots$}{
    \nl    Run \singleepoch with input $k$ and $L=T$, and $\{\alpha_i\}_{i=k}^{k_{\max}}$ specified in \pref{eq: type 2 alpha}, until it terminates or the total number of rounds reaches $T$. 
    }  \label{line: for loop inside in COBE}
    

\end{algorithm}

\begin{theorem}\label{thm: form 1 regret}
    If $\alg$ satisfies \pref{assum: regret} and $\calR(t,\thres)$ in the form of \pref{eq: typeone}, then with $\alpha_i$'s specified in \pref{eq: type 2 alpha}, \algname guarantees with probability at least $1-\order(k_{\max}\delta)$ that
    \begin{align*}
        \Reg(T) = \otil\left(\sqrt{\beta_1 T} + \beta_2 (C+Z) + \beta_3 \right),  
    \end{align*}
     where $Z=c_{\max}$ if $C\triangleq C^\am$ and $Z=c_{\max}\sqrt{T}$ if $C\triangleq C^{\rms}$. 
\end{theorem}



%

\section{Gap-Dependent Bounds}\label{sec: gap bound}
In this section, the goal is to get instance-dependent bounds similar to those in \cite{lykouris2019corruption, chen2021improved}. 
There are extra assumptions to be made in this section. First, we only deal with the case without contexts, i.e., the following assumption holds: 
\begin{assumption}\label{assum: context free}
    Assume that $\mu^\pi(x_t)=\mu^\pi$. 
\end{assumption}
This covers linear bandits and MDPs with a fixed initial state. In fact, our approach can handle a slightly more general case where the context is i.i.d. generated in the uncorrupted case, and the non-iid-ness of the context distribution is considered as corruption (in contrast, in \pref{sec: main alg}, the non-iid-ness of contexts is not considered as corruption).  
Besides, our bound depends on the sub-optimality gap defined in the following: 
\begin{assumption}\label{assum: gap}
    There exists a policy $\pistar\in\Pi$ such that for all $\pi\in \Pi\backslash\{\pistar\}$, $\mu^{\pi} \leq \mu^{\pistar} - \Delta$.  
\end{assumption}
This gap assumption is in fact stronger than that made by \cite{chen2021improved}. In \citep{chen2021improved}, $\Delta:= \min_{\pi:~\Delta_{\pi}>0}\Delta_\pi$ where $\Delta_\pi=\mu^{\pistar}-\mu^\pi$. Their definition keeps $\Delta>0$ when there are multiple optimal policies, while our \pref{assum: gap} forces $\Delta=0$ if there are two optimal policies with the same expected reward. This kind of stronger gap assumption is similar to those in \citep{lee2021achieving, jin2021best}. 
Finally, we only focus on the case with $C=C^\am$ throughout this section. \footnote{When $C=C^\rms$, our approach produces a regret term of $c_{\max}\sqrt{T}$ as in \pref{thm: form 1 regret}, spoiling the gap-dependent bound. }

\setcounter{AlgoLine}{0}
\begin{algorithm}[t]
    \caption{Gap-bound enhanced \algname (\gapalgname)} \label{alg: meta alg gap}
    \nl $k_{\init}=\max\left\{\left\lceil \log_2\frac{\sqrt{\beta_1} + \beta_2c_{\max} + \beta_3 }{\beta_2} \right\rceil, \ \ 0\right\}$,\ \  $\beta_4 =10^4\left(2\beta_1 + 42\beta_2c_{\max}\log(T/\delta) + 2\beta_3\right)$.  \\ 
    \nl \For{$k=k_{\init}, \ldots$}{ \label{line: for loop}
        \nl \texttt{// \phaseone} \\
        \nl Let $L$ be the smallest integer such that $\sqrt{\beta_4 L} \geq \beta_2 2^k$. \label{line: phaseone start}\\
        \nl \lIf{$L > T$}{\textbf{break}}   \label{line: L>T condition}
        \nl Run \singleepoch with input $k$ and $L$, and $\{\alpha_i\}_{i=k}^{k_{\max}}$ specified in \pref{eq: type 2 alpha} until it terminates or the number of rounds reaches $T$. 
        Let $o$ be its output, and let $\hatpi$ be the policy that is executed the most number of times by the base algorithm $\alg_k$.   \label{line: phaseone end} \\
        \nl \If{$o=\textbf{true}$}{ 
        \nl    \texttt{// \phasetwo} \\
        \nl    Run \twomodel with $L, \hatpi, \calB_{\hatpi}$, until it terminates or the number of rounds reaches $T$.    \label{line: phasetwo end} 
        }
    }
    \nl \texttt{// \phasethree} \\
    \nl Run \algname in the remaining rounds. 
\end{algorithm}

\subsection{Algorithm Overview} 
Our algorithm \gapalgname (\pref{alg: meta alg gap}) consists of three phases where the first two phases are executed interleavingly. In \phaseone (\pref{line: phaseone start}-\pref{line: phaseone end} in \gapalgname), we run \singleepoch with a \typea base algorithm that satisfies \pref{assum: regret} with the following gap-dependent bound: 
\begin{align}
    \calR(t, \thres) = \min\left\{\sqrt{\beta_1 t}, \ \frac{\beta_1}{\Delta}\right\}+ \beta_2 \thres + \beta_3 \label{eq: typetwo}
\end{align}
for some $\beta_1, \beta_2, \beta_3$ satisfying $\beta_1\geq 16\log(T/\delta)$, $\beta_2\geq 1$, $\beta_3\geq 10\sqrt{\beta_1\log(T/\delta)}$. 

In every for-loop of $k$, if \singleepoch in \phaseone returns true, the algorithm proceeds to \phasetwo (\pref{line: phasetwo end} in \gapalgname). In \phasetwo we execute \twomodel (\pref{alg: corral}). \twomodel is a specially designed two-model selection algorithm that dynamically chooses between two instances. One of the them is $\hatpi$, a candidate optimal policy identified in \phaseone (defined in \pref{line: phaseone end} of \gapalgname); the other is an algorithm with $\hatpi$ as input (we call this algorithm $\calB_{\hatpi}$). 
We assume  that $\calB_{\hatpi}$ has the following property: 
\begin{assumption}\label{assum: corruption robust sub routine}
    $\calB_{\hatpi}$ is a corruption-robust algorithm over the policy set $\Pi\backslash\{\hatpi\}$ without the prior knowledge of total corruption. In other words, when running alone, in every round $t$, it chooses a policy $\pi_t\in\Pi\backslash\{\hatpi\}$ and receives $r_t$ with $\E[r_t]=\mu_t^{\pi_t}$. It ensures the following for all $t$ with probability at least $1-\delta$:
    \vspace*{-5pt}
    \begin{align}
        \max_{\pi\in \Pi\backslash\{\hatpi\}}\sum_{\tau=1}^t \left(r^\pi_\tau - r_\tau\right) \leq \calR_{\calB}(t,C_t) \triangleq \sqrt{\beta_1 t} + \beta_2 C_t + \beta_3.     \label{eq: guarantee of B}\\[-10pt]
        \nonumber
    \end{align}
\end{assumption}
Notice that in \pref{sec: elimination alg} we have already developed a corruption-robust algorithm \algname, whose guarantee is already in the form of \pref{eq: guarantee of B}, albeit over the original policy set $\Pi$ (see \pref{thm: form 1 regret}). In \pref{app: leave one out}, we describe how to implement $\calB_{\hatpi}$ through running \algname on a modified MDP. 

The \twomodel in \phasetwo might end earlier than time $T$. This happens only when $\frac{1}{\Delta}+C$ is larger than the order of $2^k$. 
In this case, the algorithm goes back to \phaseone with $k$ increased by $1$. When $2^k$ grows to the order of $\sqrt{T}$ (implying that $\sqrt{T}\gtrsim \frac{1}{\Delta}+C$), we instead proceed to \phasethree and simply run \algname in the remaining rounds (\pref{line: L>T condition} of \gapalgname). 



The regret guarantee of \gapalgname is summarized by the following theorem. 
\begin{theorem}\label{thm: gap bound thm}
    \gapalgname ensures that (with $\beta_4$ defined in \pref{alg: meta alg gap})
    \vspace*{-2pt}
    \begin{align*}
        \Reg(T) = \otil\left(\min\left\{\sqrt{\beta_4 T}, \ \frac{\beta_4}{\Delta}\right\} + \beta_2 C + \beta_4 \right). 
    \end{align*}
\end{theorem}
\pref{thm: gap bound thm} gives the first $\min\{\frac{1}{\Delta}, \sqrt{T}\}+C$ bound in the literature of corrupted MDPs without the knowledge of $C$. 
To show \pref{thm: gap bound thm}, we establish some key lemmas for \phaseone and \phasetwo in \pref{sec: phaseone} and \pref{sec: two-armed alg} respectively. The complete proof of \pref{thm: gap bound thm} is given in \pref{app: omitted proof for gap bound}. Note that within the sub-routines \singleepoch and \twomodel, we re-index the time so that they both start from $t=1$ for convenience. 

\subsection{\phaseone of \gapalgname}\label{sec: phaseone}
In \phaseone we run \singleepoch with base algorithms that achieve gap-dependent bounds \pref{eq: typetwo}. The regret bound of \singleepoch under general choices of $\calR(t, \thres)$ and $\alpha_i$ is already derived in \pref{lem: general regret}. Here, we apply it with the new form of $\calR(t, \thres)$ in \pref{eq: typetwo}, and the new choice of $\alpha_i$ as below: 
\begin{align}
    \alpha_i = 
    \begin{cases}
        \min\left\{ \frac{\sqrt{\beta_1 L}/\beta_2 + 2^k}{2^i} ,\ \frac{1}{2(k_{\max}-k)}\right\}  &\text{for\ } i>k \\
        1-\sum_{i=k+1}^{k_{\max}}\alpha_i &\text{for\ } i=k
    \end{cases} \label{eq: alpha for gap bound}
\end{align}
The regret bound of \singleepoch under such choices of parameters is summarized as the following:
\begin{lemma}\label{lem: form 1 regret with gap}
    Let $L_0\leq L$ be the round at which \singleepoch terminates. If $\calR(t,\thres)$ is in the form of \pref{eq: typetwo}, and $\alpha_i$'s follow \pref{eq: alpha for gap bound}, then with high probability, \singleepoch guarantees 
    \begin{align*}
        \sum_{t=1}^{L_0}(r^{\pistar}_t - r_t) = \otil\left(\sqrt{\beta_1 L_0} + \beta_2 C_{L_0} + \beta_2 c_{\max} + \beta_3\right). 
    \end{align*}
\end{lemma} 
We see that even though our base algorithms achieve a gap-dependent bound (\pref{eq: typetwo}), the advantage is not reflected on the final bound of \singleepoch (as can be seen in \pref{lem: form 1 regret with gap}, we still do not achieve a gap-dependent bound). 
This is due to the fundamental limitation of general model selection problems \citep{pacchiano2020model}. 
Therefore, \pref{lem: form 1 regret with gap} does not seem to give any advantage over \pref{thm: form 1 regret}. 
However, the hidden advantage of using base algorithms with gap-dependent bounds is that if a base algorithm well-specifies the total corruption, it will quicker concentrate on the best policy. This enables the learner to \emph{identify the best policy} faster. This is formalized in \pref{lem: execute most times}.  


\begin{lemma}\label{lem: execute most times}
    Suppose that we run \singleepoch with base algorithms satisfying \pref{eq: typetwo}. Let $L_0\leq L$ be the round at which \singleepoch terminates. 
    If 
    $32\left( \frac{\beta_4}{\Delta} + \beta_2 C_{L_0} \right)\leq \beta_2 2^{k} \leq \sqrt{\beta_4 L}$,  
    then with probability at least $1-\order(\delta)$, $L_0=L$, and the following holds: 
    \begin{align}
        \sum_{t=1}^{L}\one[i_t=k]\one[\pi_t=\pi^\star] > \frac{1}{2}\sum_{t=1}^{L}\one[i_t=k]. \label{eq: execute most times}
    \end{align}
\end{lemma}
\pref{lem: execute most times} ensures that if $ 2^k\gtrsim \frac{1}{\Delta}+C$, by looking at which policy is most frequently executed by $\alg_k$, the learner can correctly identify the best policy $\hatpi=\pistar$ with high probability (by \pref{eq: execute most times} and the definition of $\hatpi$ in \gapalgname).

\subsection{\phasetwo of \gapalgname}\label{sec: two-armed alg}
In \phasetwo, we execute \twomodel, which is a model selection algorithm between $\hatpi$ and $\calB_{\hatpi}$. The high-level goal is to make the learner concentrate on executing $\hatpi$ until the end of $T$ rounds if $\hatpi=\pistar$ and $\frac{1}{\Delta} + C$ is relatively small, and otherwise terminate the algorithm quickly before incurring too much regret. It proceeds in epochs of varying length, indexed with $j$. The quantity $\hatDelta_j$ is an estimator of the gap between the average performance of $\hatpi$ and $\calB_{\hatpi}$ at the beginning of epoch $j$; $M_j$ is the maximum possible length of epoch $j$, and $p_j$ is the probability that the learner chooses $\calB_{\hatpi}$ in epoch $j$. The learner constantly monitors the difference between the average performance of $\hatpi$ and $\calB_{\hatpi}$ (\pref{line: detection start}-\pref{line: detection end} in \twomodel). Whenever she finds that their performance gap is actually much smaller or larger than $\hatDelta_j$ (i.e., if \pref{eq: checkk 1} or \pref{eq: checkk 2} holds), she updates $\hatDelta_j, M_j$, and $p_j$, and restarts a new epoch. If at any time $\hatDelta_j$ becomes smaller than $\hatDelta_1$, or $j$ grows larger than $3\log^2 T$, she terminates $\twomodel$. We establish the following two key lemmas. 

\begin{lemma}\label{lem: corral regret 1}
    Let $T_0$ be the last round of \twomodel, then with probability at least $1-\order(\delta)$,  
    \begin{align*}
        \sum_{t=1}^{T_0}(r_t^{\pistar} - r_t) =\otil\left( \sqrt{\beta_4 L} + \beta_2 C_{T_0}+\beta_4\right). 
    \end{align*}
\end{lemma}

\begin{lemma}\label{lem: no end}
    Let $T_0$ be last round of \twomodel. If $\hatpi=\pi^*$ and $\sqrt{\beta_4 L}\geq 16\left(\frac{\beta_4}{\Delta}+\beta_2 C_{T_0}\right)$, then with probability at least $1-\order(\delta)$, it is terminated because the number of rounds reaches $T$.  
\end{lemma}
We combine \pref{lem: form 1 regret with gap}-\pref{lem: no end} to prove \pref{thm: gap bound thm} in \pref{app: omitted proof for gap bound}.

\setcounter{AlgoLine}{0}
\begin{algorithm}[t]
    \caption{\twomodel ($L$,  $\hatpi$, $\calB_{\hatpi}$)}\label{alg: corral}
\nl \textbf{initialization}: 
    $\hatDelta_1\leftarrow \min\Big\{\sqrt{\frac{\beta_4}{L}}, 1\Big\}$\,,\ \  $M_1\leftarrow \frac{\beta_4}{\hatDelta_1^2}$\,,\ \ $t\leftarrow 1$. \qquad \quad ($\beta_4$ defined in \pref{alg: meta alg gap})  \\
\nl    \For{$j=1, 2, \ldots, (3\log^2 T)$}{
\nl
$t_j\leftarrow t$, \ \ 
        $p_j\leftarrow  \frac{\beta_4}{2M_j\hatDelta_j^2}$, \ and re-initialize $\calB_{\hatpi}$.    \label{line: define pj} \\
\nl    \While{$t\leq t_j + M_j-1$}{
\nl       $Y_t\leftarrow \text{Bernoulli}\left(p_j\right)$. \\
\nl       \lIf{$Y_t=1$}{
                Execute $\calBpi$ for one round and update $\calBpi$  
            }
\nl       \lElse{
                Execute $\hatpi$ for one round 
            }
\nl        $t\leftarrow t+1$ \\ 
\nl        Let $\Rht_0 = \frac{1}{1-p_j}\sum_{\tau=t_j}^{t-1}r_\tau \one[Y_\tau=0]$, \ \ $\Rht_1 = \frac{1}{p_j}\sum_{\tau=t_j}^{t-1}r_\tau \one[Y_\tau=1]$. \label{line: detection start}
\vspace{-9pt}
\begin{flalign}
\nl        &\textbf{if} \
              \ \ \Rht_{0} \leq \Rht_{1} +  \textstyle\frac{1}{2} (t-t_j) \hatDelta_j - \frac{5}{p_j}\calR_{\calB}\left(p_j(t-t_j), \frac{p_j\sqrt{\beta_1 L}}{\beta_2}\right) \label{eq: checkk 1} && \\
\nl        &\textbf{then}\ \ \textstyle\hatDelta_{j+1}\leftarrow \frac{1}{1.25}\hatDelta_j\ \ \textbf{and break}  \nonumber\\
\nl &\textbf{if}\ \ \    \Rht_{0} \geq \Rht_{1} + 3 M_j\hatDelta_j +  8\sqrt{\beta_1 L} \label{eq: checkk 2} && \\
\nl       &\textbf{then}\ \  \hatDelta_{j+1} \leftarrow 1.25\hatDelta_j \ \ \textbf{and break} \nonumber &&
           \end{flalign} \label{line: detection end}
           \vspace{-20pt}
        } 
        \ \\
\nl        \lIf{$\hatDelta_{j+1}< \hatDelta_1$}{\textbf{return}} 
        \vspace{-20pt}
        \begin{flalign}
\nl            \textstyle M_{j+1}\leftarrow 2(t-t_j) + \frac{\beta_4}{\hatDelta_{j+1}^2}  && \label{eq: doubling epoch length}
        \end{flalign}
        \vspace{-15pt}
    }
\end{algorithm} 

\section{Applications to Different Settings}

In \pref{app: base algorithms}, we give examples of the base algorithms whose regret bound is of the form \pref{eq: typeone} or \pref{eq: typetwo}.  
For tabular MDPs, we directly use the Robust UCBVI algorithm by \cite{lykouris2019corruption} as our base algorithm (\pref{app: robust MVP}). For linear bandit, we adopt the Robust Phased Elimination algorithm developed by \cite{bogunovic2021stochastic}, and additionally prove a gap-dependent bound for it (\pref{app: robust phased elim}). For linear contextual bandits and linear MDPs, we modify the OFUL/LSVI-UCB algorithm to make them robust to corruption (\pref{app: robust OFUL}). Then we extend the VOFUL/VARLin algorithms by \cite{zhang2021varianceaware}, further improving the dependence on $C$ over the OFUL/LSVI-UCB approach (\pref{app: robust VOFUL}). Finally, we derive a corruption-robust variant of the GOLF algorithm by \cite{jin2021bellman} for the general function approximation setting (\pref{app: robust golf}).

\section{Conclusions and Future Work}
In this work, we develop a general model selection framework to deal with corruption in bandits and reinforcement learning. In the tabular MDP setting, without knowing the total corruption, our result is the first to achieve a worst-case optimal bound. This resolves open problems raised by \cite{lykouris2019corruption, chen2021improved, wu2021reinforcement}.  A general framework to obtain refined gap-dependent bounds is also developed. In linear bandits, linear contextual bandits, and linear MDPs, our bounds also improve those of previous works in various ways. 

However, our result is not the end of the story. There are many remaining open problems to be investigated in the future: 
\begin{itemize}
    \item For the tabular setting, our gap complexity measure is larger than those in \citep{simchowitz2019non, lykouris2019corruption, jin2021best}. It is an important future direction to further improve our gap-dependent bound without sacrificing the worst-case dependence on $T$ or $C$.  
    \item For linear contextual bandits and linear MDPs, a regret bound with additive dependence on $C^\am$ is only achieved through computationally inefficient algorithms (i.e., the variants of VOFUL and VARLin). These algorithms also have a bad dependence on the feature dimension $d$. Can we address these computational and statistical issues? 
    \item In the model mis-specification literature, \cite{agarwal2020pc, zanette2021cautiously} defines a new notion of \emph{local} model mis-specification for the state aggregation scenario. It is much smaller and more favorable than the notion of model mis-specification defined in \cite{jin2020provably, zanette2020learning}. Is there any counterpart for the corruption setting? If there is, how can we achieve robustness under such notion without prior knowledge?
\end{itemize}

\acks{The authors would like to thank Liyu Chen and Thodoris Lykouris for helpful discussions. }

\bibliography{ref}

\appendix

\section{Related Work}\label{app: related work}
Corruption-robust bandit/RL have been studied under various setting, and have many other closely related topics, as we discuss below. 
\paragraph{Corrupted multi-armed bandits and tabular MDPs} Corruption-robust multi-armed bandits have been studied by \cite{lykouris2018stochastic, gupta2019better, zimmert2019optimal} through three representative approaches. Interestingly, these three approaches have all been extended to the tabular MDP case by \cite{lykouris2019corruption, chen2021improved, jin2021best} respectively. However, the extensions by \cite{lykouris2019corruption, chen2021improved} produce a new $+C^2$ term in the regret bound, largely limiting the use case of their algorithms. Besides, the computational complexity of \cite{chen2021improved}'s algorithm scales with the number of policies, which is exponentially high. On the other hand, \cite{jin2021best} successfully achieves a near-optimal bound, but requires that the transition remains uncorrupted. 

\paragraph{Corrupted linear bandits}  \cite{li2019stochastic} and \cite{bogunovic2020corruption} extend the ideas of  \cite{gupta2019better} and \cite{lykouris2018stochastic} to linear bandits and Gaussian bandits respectively. Their regret bounds both have multiplicative dependence on $C$. \cite{bogunovic2021stochastic} considers a stronger corruption model where the adversary can observe the action in the current round. Their bound $\sqrt{T}+C^2$ additively depends on $C$, but can only tolerate $C\leq \sqrt{T}$.\footnote{In the stronger adversary setting consider by \cite{bogunovic2021stochastic}, however, the $C^2$ dependence is unavoidable. } Recently, \cite{lee2021achieving} established the first upper bound that has optimal dependence on the amount of corruption as well as a refined gap-dependent bound. However, their algorithm only handles a restricted form of corruption -- the corruption injected to action $a$ must be in the form of $a^\top c$ for some vector $c$ shared among all actions. In our work, we are able to get a similar bound but without this strong assumption. 

\paragraph{Corrupted MDPs with linear function approximation}  
\cite{lykouris2019corruption} studies corrupted linear MDPs and gets a bound of order $C^2\sqrt{T}$, which only tolerates $C\leq T^{\nicefrac{1}{4}}$. \cite{zhang2021robust} leverages the intrinsic robustness of policy gradient and tools in robust statistics to achieve an improved bound $\sqrt{(1+C)T}$ when the feature space has a bounded \emph{relative condition number}.  
\cite{zhang2021corruption} further studies offline RL in linear MDPs, showing that if the offline data has wide coverage, then there is an algorithm that can output a $\order(\sqrt{\nicefrac{1}{T}}+\nicefrac{C}{T})$-optimal policy after seeing $T$ samples with $C$ of them corrupted. Although this result indicates that $+C$ penalty in regret might be possible, their result heavily relies on the coverage assumption and does not apply to our setting. 

\paragraph{Robust statistics} The goal of robust statistics is to design estimators of some unknown quantity that are robust to data corruption. In several recent works, computationally efficient and highly robust estimators for linear regression that tolerate a constant fraction of data corruption have been designed \citep{bhatia2017consistent, Diakonikolas_2019,  chen2020online}. However, such strong guarantees usually require additional assumptions on the data generation process or the corruption process. Robust statistics has been used in corruption-robust RL under special cases. For example, \cite{zhang2021robust, zhang2021corruption} achieve robustness in MDPs with certain exploratory properties, and \cite{awasthi2020online} handles the case where the corrupted rounds are i.i.d. generated.   

\paragraph{Model mis-specification}  The notion of corruption we consider subsumes the notion of model mis-specification studied in many previous works \citep{jiang2017contextual, du2019good, jin2020provably, zanette2020learning, lattimore2020learning, wang2020reinforcement}. These works assume that the model class can only approximate the true world up to an order of $\order(\epsilon)$, and they establish regret bounds that have an additive $\order(\epsilon T)$ penalty. Clearly, one 
can also view the difference between the model and the true world as corruption, and as shown in \pref{tab: table of bounds}, all our bounds (both $\sqrt{T}+C^\am$ and $\sqrt{T}+C^{\rms}$) recover the $\order(\sqrt{T} + \epsilon T)$ bound in the mis-specification case. While many previous works assume a known $\epsilon$, there are also works dealing with the case of unknown $\epsilon$ \citep{takemura2021parameter, foster2021adapting, pacchiano2020regret}. 
 
\paragraph{Best-of-both-world bounds}  The best-of-both-world problem was studied by \cite{bubeck2012best, seldin2014one, auer2016algorithm, seldin2017improved, wei2018more, zimmert2019optimal, zimmert2019beating, jin2020simultaneously,  ito2021parameter, jin2021best, lee2021achieving} for various settings including multi-armed bandits, combinatorial semi-bandits, linear bandits, and tabular MDPs. The goal of this line of work is to have a single algorithm that achieves a $\otil(\sqrt{T})$ regret bound when the reward is adversarial and $\order(\log T)$ when the reward is stochastic, without knowing the type of reward in advance. Compared to our setting, their regret bound is always sub-linear in $T$ against a fixed policy, while ours is linear in the amount of corruption. However, their results usually rely on stronger structural assumptions than the corrupted setting we consider (e.g., the fixed transition assumption for MDPs or the linearized corruption assumption for linear bandits). 

\paragraph{Non-stationary RL} Non-stationary RL is another line of research that deals with non-static reward and transition \citep{cheung2020reinforcement, wei2021nonstationary}. In non-stationary RL, the difficulty of the problem is usually quantified by the number of times the reward or transition changes, or their fine-grained amount of variation. The corruption setting can be viewed as a special case of it, so existing algorithms for the latter can be readily applied. However, since non-stationary RL is more general, this reduction only leads to sub-optimal regret bounds. For example, the tight bound $\sqrt{T}+B^{\nicefrac{1}{3}}T^{\nicefrac{2}{3}}$ obtained in \cite{wei2021nonstationary}, where $B$ is the overall variation, only translates to a sub-optimal bound $\sqrt{T}+C^{\nicefrac{1}{3}}T^{\nicefrac{2}{3}}$ in the corruption setting. 


\paragraph{Model selection}  
Our approach is closely related to the \emph{regret balancing} technique developed by \cite{abbasi2020regret, pacchiano2020regret, cutkosky2021dynamic}. \cite{pacchiano2020regret, cutkosky2021dynamic} have applied regret balancing to tackle model mis-specification, but it remains unclear whether it also handles the more general corruption setting, where the adversary chooses which rounds to corrupt.
Existing techniques which choose base learners deterministically can only use a regret bound that includes the total corruption budget $\theta$ which leads to loose guarantees. Instead, our randomized choice allows us to scale regret bounds of base learners as $\alpha_i \theta$, where $\alpha_i$ is the probability of being selected. While \cite{pacchiano2020regret, cutkosky2021dynamic} also provide a version of their algorithm with a randomized learner choice for the special case of linear stochastic bandits with adversarial contexts, they resort to a weaker elimination test that requires additional information from base learners. Our work shows that this is indeed unnecessary by pairing a randomized learner selection with a simple elimination test. This may be of interest beyond the corruption setting.

In our work, we also develop a special model selection algorithm that achieves a gap-dependent bound that is better than $\sqrt{T}$. This kind of better-than-$\sqrt{T}$ bound is rare in the literature of model selection, and even proven to be impossible for general cases \citep{pacchiano2020model}. To our best knowledge, the only work on model selection that breaks the $\sqrt{T}$ barrier is \cite{arora2021corralling}, who considers a stationary multi-armed bandit setting where every base algorithm learns over a subset of arms, and the best arm is only controlled by one of the base algorithms. 
However, their stochastic bandit setting is less challenging than our adversarial/corrupted RL setting, so their techniques cannot be directly applied. We hope that our technique can also hint about how to achieve better-than-$\sqrt{T}$ bounds in more general model selection problems.

\section{The Non-robustness of Least Square Regression}\label{app: non-robust}
In this section, we show that for linear contextual bandits with non-i.i.d. contexts, the most natural extension from the standard OFUL algorithm to a corruption-robust version results in a regret bound of $\Omega(\sqrt{C^\am T})$, even if $C^\am$ is known. See \pref{sec: problem setting} for the definition of the linear contextual bandit framework that we consider. Recall that in the standard OFUL algorithm \citep{abbasi2011improved}, the learner constructs a confidence set for the underlying parameter: 
\begin{align}
    \calW_t = \left\{ w~:~ \|w-\widehat{w}_t\|_{\Lambda_t}^{2}\leq \iota_t \right\} \label{eq: confidence OFUL}
\end{align}
for some $\iota_t > 0$, where 
\begin{align*}
    \Lambda_t = \lambda I + \sum_{\tau=1}^{t-1}a_\tau a_\tau^\top, \qquad \widehat{w}_t = \Lambda_t^{-1}\left(\sum_{\tau=1}^{t-1}a_\tau r_\tau\right)   
\end{align*}
for some hyper-parameter $\lambda>0$ ($a_\tau$ is the action taken at round $\tau$, and $r_\tau$ is the reward received at round $t$). The action chosen at round $t$ is 
\begin{align}
    a_t = \argmax_{a\in\calA_t} \max_{w\in\calW_t} a^\top w.   \label{eq: action OFUL}
\end{align}

To make this algorithm robust to corruption, a natural modification is to widen the confidence set \pref{eq: confidence OFUL}. That is, the confidence set is changed to
\begin{align}
     \calW_t = \left\{ w~:~ \|w-\widehat{w}_t\|_{\Lambda_t}^{2}\leq \iota_t' \right\} \label{eq: robust OFUL conf set} 
\end{align}
for some $\iota_t' > \iota_t$. The definition of $\iota_t'$ may involve the knowledge of $C^\am$. Below we show a regret lower bound for this class of algorithms. 

We consider the following example for $d=1$. The action set in each round is the following: 
\begin{align*}
    \calA_t = \begin{cases}
        \{-1, 1\}   &\text{if } 1\leq t\leq C \\
        \{-\epsilon, \epsilon\} &\text{if } t>C 
    \end{cases}
\end{align*}
for some $C\in \mathbb{N}$. 
The true underlying parameter is $w^\star=1$, but in rounds $1, 2, \ldots, C$, the rewards are generated using $w'=-1$. We assume that there is no noise, i.e., $r_t=a_t w'$ for $t\in[1,C]$ and $r_t=a_t w^\star$ for $t>C$. 

In this case, the confidence set \pref{eq: robust OFUL conf set} can be written as 
\begin{align}
    \calW_t = \left\{w~:~ (w-\widehat{w}_t)^2 \leq \frac{\iota_t'}{\Lambda_t}\right\} \label{eq: confidence 1d}
\end{align}
where $\Lambda_t = \lambda + \sum_{\tau=1}^{t-1}a_\tau^2$ and 
\begin{align*}
    \widehat{w}_t &= \frac{1}{\Lambda_t} \sum_{\tau=1}^{t-1}a_\tau r_\tau. 
\end{align*}
By the reward generation process and the definition of action sets, we have that 
\begin{align*}
    \sum_{\tau=1}^{t-1}a_\tau r_\tau 
    &=  \sum_{\tau=1}^{t-1}  a_\tau \left(\one[\tau\leq C]a_\tau  w' + \one[\tau >C] a_\tau w^\star\right) \\
    &=
    \begin{cases}
         (t-1)w'=-t+1   &\text{if } t\leq C\\
         Cw' + (t-1-C)\epsilon^2 w^\star  = -C + (t-1-C)\epsilon^2 &\text{if } t> C
    \end{cases}
\end{align*}
Therefore, $\widehat{w}_t<0$ for $2\leq t \leq C\left(1+\frac{1}{\epsilon^2}\right)$. Since the confidence set $\calW_t$ (\pref{eq: confidence 1d}) is symmetric around $\widehat{w}_t$, by the action selection rule \pref{eq: action OFUL}, when  $\widehat{w}_t<0$, the learner will choose action $-1$ if $t\leq C$, and $-\epsilon$ if $t>C$. 

Therefore, the learner will choose sub-optimal actions in $2\leq t \leq C\left(1+\frac{1}{\epsilon^2}\right)$, and the regret is of order
\begin{align*}
    &\sum_{t=2}^{C} w^\star(1-(-1)) + \sum_{t=C+1}^{\min\{C(1+\nicefrac{1}{\epsilon^2)}, T\}}w^\star(\epsilon-(-\epsilon)) 
    \\ 
    &=(C-1)\times 2 + \min\left\{\frac{C}{\epsilon^2}, T-C\right\}\times 2\epsilon = \Theta\left(C+\min\left\{\frac{C}{\epsilon}, \epsilon T\right\}\right). 
\end{align*}
By picking $\epsilon=\sqrt{\frac{C}{T}}$, we see that the regret is at least of order $\sqrt{CT}$. Finally, notice that in the example we construct $C^\am=\Theta(C)$, hence proving our claim.

\section{Concentration Inequalities}\label{app: concentration}

\begin{lemma}[Freedman's inequality, Theorem 1 of \citep{beygelzimer2011contextual}]\label{lem: beygel freedman}
Let $\calF_0\subset \calF_1 \subset\cdots \subset \calF_{n}$ be a filtration, and $X_1, \ldots, X_n$ be real random variables such that $X_i$ is $\calF_i$-measurable, $\E[X_i|\calF_{i-1}]=0$, $|X_i|\leq b$, and $\sum_{i=1}^n \E[X_i^2|\calF_{i-1}]\leq V$ for some fixed $b\geq 0$ and $V\geq 0$. Then with probability at least $1-\delta$, 
\begin{align*}
    \sum_{i=1}^n X_i \leq 2\sqrt{V\log(1/\delta)} + b\log(1/\delta). 
\end{align*}

\end{lemma}

\begin{lemma}[Freedman's inequality, Lemma 4.4 of \citep{bubeck2012best}]\label{lem: freedman}
Let $\calF_0\subset \calF_1 \subset\cdots \subset \calF_{n}$ be a filtration, and $X_1, \ldots, X_n$ be real random variables such that $X_i$ is $\calF_i$-measurable, $\E[X_i|\calF_{i-1}]=0$, $|X_i|\leq b$, for some fixed $b\geq 0$. Let $V_n = \sum_{i=1}^n \E[X_i^2|\calF_{i-1}]$. Then with probability at least $1-\delta$, 
\begin{align*}
    \sum_{i=1}^n X_i \leq 2\sqrt{V_n\log(n/\delta)} + 3b\log(n/\delta). 
\end{align*}

\end{lemma}

\begin{lemma}\label{lem: aux0}
   Let $\calF_0\subset \calF_1 \subset\cdots \subset \calF_{T}$ be a filtration, and $X_1, \ldots, X_T$ be real random variables such that $X_t$ is $\calF_t$-measurable, $\E[X_t|\calF_{t-1}]=0$, $|X_t|\leq b$, for some fixed $b\geq 0$. Let $z_t\sim \text{Bernoulli}(\alpha)$ be an i.i.d. random variable independent of all other variables, and let $0\leq y_t\leq b$ be a deterministic scalar given $\calF_{t-1}$.  Then with probability at least $1-\delta$, the following holds for all $\calI=[t_1, t_2]\subseteq [1, T]$: 
   \begin{align*}
       \left|\sum_{t\in\calI}y_t(z_t-\alpha)\right| \leq \min\left\{4\sqrt{\alpha\sum_{t\in\calI} y_t^2 \log(T/\delta)} + 9b\log(T/\delta)
       , \ \ \ \frac{1}{4}\alpha \sum_{t\in\calI}y_t + 21b\log(T/\delta)\right\}. 
   \end{align*}

   
\end{lemma}
\begin{proof}
Fixing an interval $\calI\in[1, T]$, we apply \pref{lem: freedman} with $X_t=y_t(z_t-\alpha)$. Then we get that with probability at least $1-2\delta'$, 
\begin{align*}
    \left|\sum_{t\in\calI} y_t(z_t-\alpha)\right| 
    &\leq 2\sqrt{\sum_{t\in\calI} y_t^2\E_t[(z_t-\alpha)^2] \log(T/\delta')} + 3b\log(T/\delta')   \tag{define $\E_t[\cdot]=\E[\cdot|\calF_{t-1}]$} \\
    &\leq 2\sqrt{\alpha\sum_{t\in\calI} y_t^2 \log(T/\delta')} + 3b\log(T/\delta')   \tag{$\E_t[(z_t-\alpha)^2]=\alpha(1-\alpha)^2+(1-\alpha)\alpha^2\leq \alpha$} \\
    &\leq 2\sqrt{b\log(T/\delta')}\sqrt{\alpha\sum_{t\in\calI} y_t } + 3b\log(T/\delta')  \tag{$|y_t|\leq b$} \\
    &\leq \frac{1}{4}\alpha\sum_{t\in\calI} y_t + 7b\log(T/\delta')   \tag{AM-GM}
\end{align*}
Notice that there are $\frac{T(T-1)}{2}$ different $\calI$'s, so we pick $\delta'=\frac{\delta}{T(T-1)}$, and take an union bound over $\calI$'s. This gives the desired bound. 

\end{proof}

\section{Omitted Proofs in \pref{sec: main alg}}
We start with some extra notations to be used in this section. 
\begin{definition} \label{def: notation defs}
For any time $t$, base algorithm $i$, and policy $\pi$, define 
    $C_{t,i}^{\am}\triangleq \sum_{\tau=1}^t \one[i_\tau=i]c_\tau$ and $C^{\rms}_{t,i}\triangleq \sqrt{\left(\sum_{\tau=1}^t \one[i_\tau=i]\right)\left(\sum_{\tau=1}^t  \one[i_\tau=i]c_\tau^2\right)}$\,. Similarly, when we write $C_{t,i}$ to indicate either $C_{t,i}^\am$ or $C_{t,i}^\rms$, depending on the type of base algorithms we use. 
\end{definition}   
\begin{definition}
    For any time $t$, base algorithm $i$, and policy $\pi$, define $R_{t,i}^\pi=\sum_{\tau=1}^t \one[i_\tau=i] r_{\tau}^\pi$\,. 
\end{definition}
Next, we prove some lemmas to be used in the later analysis. 
\begin{lemma}\label{lem: relate N and alpha}
    In \singleepoch (\pref{alg: regret balancing all}), for any fixed $i$, with probability at least $1-2\delta$, the following holds for all $t$:  
   \begin{align*}
        \frac{3}{4}\alpha_i t - 21\log(T/\delta)\leq N_{t,i} \leq \frac{5}{4}\alpha_i t + 21\log(T/\delta). 
   \end{align*}
\end{lemma}
\begin{proof}
   This is by directly applying \pref{lem: aux0} with $y_t=1$ and $\alpha=\alpha_i$
\end{proof}

\begin{lemma}\label{lem: aux1}
   For any fixed $i$, with probability at least $1-3\delta$, the following holds for all $t$: 
   \begin{align*}
       C^\am_{t,i} &\leq 1.25\alpha_i C_t^\am+21c_{\max}\log(T/\delta), \\
       C^\rms_{t,i} &\leq 1.25\alpha_i C^\rms_t + 8c_{\max} \sqrt{\alpha_i t\log(T/\delta)} + 21c_{\max}\log(T/\delta). 
   \end{align*}
\end{lemma}


\begin{proof} 
We prove the lemma for $(C_{t,i}, C)=(C^\am_{t,i}, C^\am)$ and $(C_{t,i}, C)=(C^\rms_{t,i}, C^\rms)$
cases separately. 
\paragraph{Case 1.}$(C_{t,i}, C)=(C^\am_{t,i}, C^\am)$. \quad  
\begin{align*}
    C_{t,i} = \sum_{\tau=1}^t c_\tau\one[i_\tau=i] \leq \frac{5}{4}\alpha_i C +  21c_{\max}\log(T/\delta). \tag{holds w.p. $\geq 1-\delta$ by \pref{lem: aux0} with $y_\tau=c_\tau$, $z_\tau=\one[i_\tau=i]$}  
\end{align*}

\paragraph{Case 2.}$(C_{t,i}, C)=(C^\rms_{t,i}, C^\rms)$. \quad 
\begin{align*}
    C_{t,i} &= \sqrt{\left(\sum_{\tau=1}^t\one[i_\tau=i]\right)\left(\sum_{\tau=1}^t\one[i_\tau=i]c_\tau^2\right)}\\
    &\leq \sqrt{\left(\frac{5}{4}\alpha_i t +  21\log(T/\delta)\right)\left(\frac{5}{4}\alpha_i\sum_{\tau=1}^t c_\tau^2 + 21c_{\max}^2\log(T/\delta)\right)}  \tag{holds w.p. $\geq 1-2\delta$ by \pref{lem: aux0} with $y_\tau=1$ and $y_\tau=c_\tau^2$} \\
     &\leq \sqrt{\frac{25}{16}\alpha_i^2 t \sum_{\tau=1}^t c_\tau^2 + 52.5\alpha_i tc_{\max}^2\log(T/\delta) +  21^2c_{\max}^2\log^2(T/\delta)}   \\
    &\leq \frac{5}{4}\alpha_i \sqrt{t\sum_{\tau=1}^t c_\tau^2} + 8c_{\max}\sqrt{\alpha_i t\log(T/\delta)} + 21c_{\max}\log(T/\delta) \tag{$\sqrt{a+b+c}\leq \sqrt{a}+\sqrt{b}+\sqrt{c}$} \\
    &= \frac{5}{4}\alpha_i C + 8c_{\max}\sqrt{\alpha_i t\log(T/\delta)} + 21 c_{\max}\log(T/\delta). 
\end{align*}
\end{proof}
\begin{lemma}\label{lem: easy lemma 1}
    For any $i$, with probability at least $1-\order(\delta)$, the following holds for all $t$ such that $C_t\leq 2^i$:  
    \begin{align*}
        R_{t,i}^{\pistar} - R_{t,i} \leq \calR(N_{t,i}, \theta_i). 
    \end{align*}
\end{lemma}
\begin{proof}
    The total amount of corruption experienced by $\alg_i$ up to round $t$ is $C_{t,i}$, whose upper bound is given in \pref{lem: aux1} for both types of base algorithms. Comparing the upper bounds of $C_{t,i}$ with our choice of $\theta_i$ in \pref{eq: choice of thres}, we see that for a fixed $i$, under the condition $C_t \leq 2^i$, we have $C_{t,i}\leq \theta_i$ with probability $1-\order(\delta)$. In other words, the condition specified in \pref{assum: regret} is satisfied for $\alg_i$ in the rounds that it is executed. Therefore, by the regret bound in \pref{assum: regret}, we have 
    \begin{align*}
        R_{t,i}^{\pistar} - R_{t,i}  = \sum_{\tau=1}^t (r^{\pistar}_\tau - r_\tau)\one[i_\tau=i] \leq \calR(N_{t,i}, \theta_i). 
    \end{align*}
\end{proof}

\begin{lemma}\label{lem: easy lemma 2}
    For any fixed $i$, with probability at least $1-\delta$, the following holds for all $t$: 
   \begin{align*}
       \left|\frac{1}{\alpha_i}R^{\pistar}_{t,i} - \sum_{\tau=1}^t r^{\pistar}_{t,i}\right|
       \leq 2\sqrt{\frac{t\log(T/\delta)}{\alpha_i}} + \frac{\log(T/\delta)}{\alpha_i}. 
   \end{align*}
\end{lemma}
\begin{proof}
    By \pref{lem: beygel freedman}, for a fixed $i$, with probability $1-\delta$, for all $t$,  
   \begin{align*}
       \left|\frac{1}{\alpha_i}R^{\pistar}_{t,i} - \sum_{\tau=1}^t r^{\pistar}_{t,i}\right| 
       &= \left| \sum_{\tau=1}^t \left(\frac{\one[i_\tau=i]}{\alpha_i} - 1\right)r^{\pistar}_\tau \right| 
       \leq 2\sqrt{\frac{t\log(T/\delta)}{\alpha_i}} + \frac{\log(T/\delta)}{\alpha_i}. 
   \end{align*}
    
\end{proof}

\begin{proof}\textbf{of \pref{lem: istar = k case}. }
    Notice that $C_t\leq 2^k$ implies that $C_t\leq 2^i$ for all $i\in[k, k_{\max}]$. 
    Notice that 
   \begin{align}
       R_{t,i} - R_{t,i}^{\pi^\star}
       &= \sum_{\tau=1}^t \one[i_\tau=i](r_\tau - r_\tau^{\pi^\star})  \nonumber  \\
       &= \sum_{\tau=1}^t \one[i_\tau=i]\left( r_\tau - \mu^{\pi_\tau}(x_\tau) + \underbrace{\mu^{\pi_\tau}(x_\tau) - \mu^{\pi^\star}(x_\tau)}_{\leq 0 \text{ (by \pref{assum: realizable})}}  + \mu^{\pi^\star}(x_\tau)  - r_\tau^{\pistar}\right)  \nonumber\\
       &\leq \sum_{\tau=1}^t \one[i_\tau=i]\left(r_\tau - \mu^{\pi_\tau}_\tau(x_\tau) + \mu_\tau^{\pistar}(x_\tau) - r^{\pistar}_\tau\right) + 2C^{\am}_{t,i} \tag{$|\mu_\tau^{\pi}(x_\tau)-\mu^{\pi}(x_\tau)|\leq c_\tau$ for all $\pi$ and $\tau$} \\
       &\leq 2\sqrt{2\alpha_i t\log(T/\delta)} + 6\log(T/\delta) + 2C_{t,i} 
       \tag{by \pref{lem: freedman} with an union bound over $t$, and that $C^\am_{t,i}\leq C^{\rms}_{t,i}$}  \\
       &\leq 2\sqrt{2\alpha_i t\log(T/\delta)} + 6\log(T/\delta) + 2\thres_i.   \tag{$C_{t,i}\leq \thres_i$ with high probability by \pref{lem: aux1}} \\
       &  \label{eq: keyy eq 2}
   \end{align}


Combining \pref{lem: easy lemma 1} and \pref{lem: easy lemma 2}, we see that the performance of $\alg_i$ admits the following lower bound with probability at least $1-\order(\delta)$:
\begin{align}
    \frac{R_{t,i}}{\alpha_i} \geq \frac{R^{\pistar}_{t,i} - \calR(N_{t,i}, \thres_i)}{\alpha_i} \geq \sum_{\tau=1}^t r_\tau^{\pistar} - \frac{\calR(N_{t,i}, \thres_i)}{\alpha_i} - 2\sqrt{\frac{t\log(T/\delta)}{\alpha_i}} - \frac{\log(T/\delta)}{\alpha_i}.  \label{eq: lower eq}
\end{align}
Combining \pref{eq: keyy eq 2} and \pref{lem: easy lemma 2}, we also have the following with probability at least $1-\order(\delta)$: 
\begin{align}
    \frac{R_{t,i}}{\alpha_i} &\leq 
    \frac{R_{t,i}^{\pistar}}{\alpha_i} + 3\sqrt{\frac{t\log(T/\delta)}{\alpha_i}} + \frac{6\log(T/\delta) + 2\thres_i}{\alpha_i} \nonumber 
    \\
    &\leq \sum_{\tau=1}^t r_\tau^{\pistar} + 5\sqrt{\frac{t\log(T/\delta)}{\alpha_i}} + \frac{7\log(T/\delta)+ 2\thres_i}{\alpha_i}. \label{eq: upper eq}
\end{align}
The bounds \pref{eq: lower eq} and \pref{eq: upper eq} together with an union bound over $i$'s indicate that the following holds for all $i, j\in[k, k_{\max}]$ with probability $1-\order(k_{\max}\delta)$: 
\begin{align*}
     \frac{R_{t,i}}{\alpha_i} +  \frac{\calR(N_{t,i}, \thres_i)}{\alpha_i} + 2\sqrt{\frac{t\log(T/\delta)}{\alpha_i}} + \frac{\log(T/\delta)}{\alpha_i}
     \geq \sum_{\tau=1}^t r_\tau^{\pistar}  \geq \frac{R_{t,j}}{\alpha_j} -  5\sqrt{\frac{t\log(T/\delta)}{\alpha_j}} - \frac{7\log(T/\delta)+ 2\thres_j}{\alpha_j}.  
\end{align*}
Further combined with the fact that $\alpha_i\geq \alpha_j$ since $i\leq j$, the last inequality implies that the termination condition \pref{eq: terminate condition 1} will not hold.  

\end{proof}

\begin{proof}\textbf{of \pref{lem: general regret}. }
\begin{align}
    \sum_{t=1}^{L_0}\left(r_{t}^{\pi^\star} - r_t\right) \leq 1+ \sum_{i=k}^{k_{\max}} \sum_{t=1}^{L_0-1}\left(r_{t}^{\pi^\star} - r_t\right)\one[i_t=i] = 1 + \sum_{i=k}^{k_{\max}}  \left(R_{L_0-1,i}^{\pi^\star} - R_{L_0-1,i}\right). \label{eq: deompose}
\end{align}
For $i\geq i^\star$, since the corruption level is well-specified, by \pref{lem: easy lemma 1}, with probability at least $1-\order(\delta)$, 
\begin{align}
    R_{L_0-1,i}^{\pi^\star} - R_{L_0-1,i} \leq \calR(N_{L_0-1,i},\thres_i). \label{eq: well specified case}
\end{align}
For $i<i^\star$,  with probability $1-\order(\delta)$, 
\begin{align}
    \frac{R_{L_0-1, i}}{\alpha_i} 
    &\geq \frac{R_{L_0-1, i^\star}}{\alpha_{i^\star}} - \frac{\calR(N_{L_0-1,i},\thres_i)}{\alpha_i} - \order\left(\sqrt{\frac{(L_0-1)\log(T/\delta)}{\alpha_{i^\star} }} + \frac{\thres_{i^\star} + \log(T/\delta) }{\alpha_{i^\star}}\right)   \tag{by the termination condition \pref{eq: terminate condition 1}} \\
    &\geq \frac{R_{L_0-1, i^\star}^{\pi^\star}}{\alpha_{i^\star}} - \frac{\calR(N_{L_0-1,i},\thres_i)}{\alpha_i } - \otil\left(\sqrt{\frac{L_0}{\alpha_{i^\star}}} + \frac{\calR(N_{L_0-1,i^\star},\theta_{i^\star})}{\alpha_{i^\star}}\right) \tag{by \pref{lem: easy lemma 1} and that $\calR(\cdot, \theta)\geq \theta$} \\
    &\geq \frac{R_{L_0-1, i}^{\pi^\star}}{\alpha_{i^\star}} - \frac{\calR(N_{L_0-1,i},\thres_i)}{\alpha_i} - \otil\left(\sqrt{\frac{L_0}{\alpha_{i^\star}}} + \frac{\calR(N_{L_0-1,i},\theta_{i^\star})}{\alpha_{i^\star}}\right) \label{eq: i < istar case}
\end{align}
where the last inequality is because by \pref{lem: easy lemma 2} we have
\begin{align*}
    \left|\frac{1}{\alpha_i}R^{\pi^\star}_{L_0-1,i} - \frac{1}{\alpha_{i^\star}}R^{\pi^\star}_{L_0-1,i^\star}\right|\leq \otil\left(\sqrt{\frac{L_0}{\alpha_i}} + \frac{1}{\alpha_i} + \sqrt{\frac{L_0}{\alpha_{i^\star}}} + \frac{1}{\alpha_{i^\star}}\right) = \otil\left( \sqrt{\frac{L_0}{\alpha_{i^\star}}} + \frac{1}{\alpha_{i^\star}}\right). 
\end{align*}
Combining \pref{eq: i < istar case} with \pref{eq: deompose} and \pref{eq: well specified case} and an union bound over $i$'s, we get that with probability at least $1-\order(k_{\max}\delta)$, 
\begin{align} 
    &\sum_{t=1}^{L_0}\left(r_{t}^{\pi^\star} - r_t\right) \nonumber \\
    &\leq 1+\sum_{i=k}^{k_{\max}} \calR(N_{L_0-1, i}, \thres_i) + \sum_{i<i^\star} \alpha_{i} \times \otil\left(\sqrt{\frac{L_0}{\alpha_{i^\star}}} +  \frac{\calR(N_{L_0-1,i^\star},\theta_{i^\star})}{\alpha_{i^\star}}\right)  \nonumber  \\
    &\leq 1+\sum_{i=k}^{k_{\max}} \calR(N_{L_0-1, i}, \thres_i) + \otil\left( \one[i^\star > k] \left(\sqrt{\frac{L_0}{\alpha_{i^\star}}} + \frac{\calR(N_{L_0-1,i^\star},\theta_{i^\star})}{\alpha_{i^\star}}\right)\right)  \label{eq: combine regret meta}
\end{align}
where in the last inequality we use $\sum_{i<i^\star}\alpha_i\leq 1$. 


\end{proof}

\begin{proof}\textbf{of \pref{thm: form 1 regret}. }
    Recall that we define
    \begin{align*}
        Z = \begin{cases}
             c_{\max}   &\text{if $\alg$ is \typea},\\
             c_{\max}\sqrt{T}  &\text{if $\alg$ is \typer}.
        \end{cases}
    \end{align*}
    Let $i^\star$ be the smallest $i\in[k, k_{\max}]$ such that $C\leq 2^i$. 
    By \pref{lem: general regret} and by the choice of $\thres_i$ in \pref{eq: choice of thres}, with probability at least $1-\order(k_{\max}\delta)$, the regret within an epoch is upper bounded by 
    \begin{align*}
        &\otil\left(\sum_{i=k }^{k_{\max}}\left( \sqrt{\beta_1  \alpha_i T } + \beta_2 \alpha_i 2^i + \beta_2 Z + \beta_3\right) + \one[k<i^\star]\left(\sqrt{\frac{T}{\alpha_{i^\star}}} +  \frac{\sqrt{\beta_1\alpha_{i^\star}T}+ \beta_2 \alpha_{i^\star}2^{i^\star} + \beta_2 Z +  \beta_3}{\alpha_{i^\star}} \right) \right) \\
        &= \otil\left( \sqrt{\beta_1 T} + \beta_2 2^k + \beta_2z + \beta_3 + \one[k<i^\star]\left(\sqrt{\frac{\beta_1 T}{\alpha_{i^\star}}} + \beta_2 2^{i^\star} + \frac{ \beta_2z + \beta_3}{\alpha_{i^\star}}\right)  \right)  \tag{$\alpha_i 2^i\leq 2^k$ by the choice of $\alpha_i$} \\
        &= \otil\left( \sqrt{\beta_1 T} + \beta_2 (2^k + 2^{i^\star})  + \beta_2 Z  + \beta_3 + \sqrt{\frac{\beta_1 T\cdot 2^{i^\star}}{2^k}} + \frac{(\beta_2 Z  + \beta_3)  2^{i^\star}}{2^k} \right) \tag{by the choice of $\alpha_{i^\star}$}\\
        &= \otil\left( \sqrt{\beta_1 T} + \beta_2 (2^k + 2^{i^\star}) + \beta_2 Z  + \beta_3 + \beta_2 2^{i^\star} + \frac{\beta_1 T}{\beta_2 2^k} + \beta_2 2^{i^\star}\times \frac{ \beta_2 Z  + \beta_3}{\beta_2 2^k}\right)   \tag{AM-GM}\\
        &= \otil\left(\sqrt{\beta_1 T} + \beta_2 (2^k + 2^{i^\star}) + \beta_2 Z  + \beta_3 \right).  \tag{using $\beta_2 2^k\geq \sqrt{\beta_1 T} + \beta_2 Z + \beta_3$ by the choice of $k_{\init}$}
    \end{align*} 
    Notice that in \algname we start from $k=k_{\init}$. If $k_{\init}\geq i^\star$, then by \pref{lem: istar = k case}, the algorithm will run with $k=k_{\init}$ throughout all $T$ rounds. On the other hand, if $k_{\init} < i^\star$, the $k$ used in \algname might increase from $k_{\init}$. However, if $k=i^\star$ is ever reached, again by \pref{lem: istar = k case}, the learner will use this $k$ throughout the rest of the steps. In short, the $k$'s used in \algname are upper bounded by $\max\{k_{\init}, i^\star\}$ with high probability. Since there are at most $\order(k_{\max})$ epochs, the overall regret is upper bounded by 
    \begin{align*}
        &\order(k_{\max}) \times \otil\left(\sqrt{\beta_1 T} + \beta_2 (2^{\max\{k_{\init}, i^\star\}} + 2^{i^\star}) + \beta_2 Z + \beta_3 \right) = \otil\left(\sqrt{\beta_1 T} + \beta_2 (C + Z) + \beta_3\right)   
    \end{align*}
    with probability at least $1-\order(k_{\max}^2\delta)$
    Considering the difference definitions of $Z$ for \typea and \typer base algorithms finishes the proof. 
\end{proof}

\section{Omitted Proofs in \pref{sec: phaseone}}

\begin{proof}\textbf{of \pref{lem: form 1 regret with gap}. }
     By \pref{lem: general regret}, \begin{align*}
         &\sum_{t=1}^{L_0}(r^{\pistar}_t - r_t) \\ 
         &\leq  \sum_{i=k}^{k_{\max}} \left(\min\left\{\sqrt{\beta_1 N_{L_0,i}  }, \ \frac{\beta_1}{\Delta}\right\} + \beta_2 \alpha_i 2^i + \beta_2 c_{\max} + \beta_3 \right) \\
         &\qquad \qquad  + \otil\left( \sqrt{\frac{L}{\alpha_{i^\star}}} + \frac{\min\left\{ \sqrt{\beta_1 N_{L_0,i^\star}}, \frac{\beta_1}{\Delta}\right\} + \beta_2 \alpha_{i^\star}2^{i^\star} + \beta_2 c_{\max} + \beta_3}{\alpha_{i^\star}} \right)\\
         &\leq \otil\left(\sqrt{\beta_1 L} + \beta_2 (2^k + 2^{i^\star}) + \frac{\beta_2 c_{\max} + \beta_3}{\alpha_{i^\star}} + \sqrt{\frac{L}{\alpha_{i^\star}}} + \frac{\sqrt{\beta_1 N_{L_0, i^\star}}}{\alpha_{i^\star}} \right)  \tag{by the definition of $\alpha_i$, $\alpha_i 2^i=\order(2^k + \sqrt{\beta_1 L}/\beta_2)$} \\
         &= \otil\left(\beta_2 (2^k + 2^{i^\star}) + \frac{\sqrt{\beta_1} + \beta_2 c_{\max} + \beta_3}{\alpha_{i^\star}} + \sqrt{\frac{\beta_1 L}{\alpha_{i^\star}}} \right) \tag{using $N_{L_0, i^\star}=\otil\left(\alpha_{i^\star}L_0 + 1\right)$ by \pref{lem: relate N and alpha}} \\
         &= \otil\left( \beta_2(2^k + 2^{i^\star}) + \left(\sqrt{\beta_1} + \beta_2 c_{\max} + \beta_3\right)\left(1 + \frac{2^{i^\star}}{\frac{\sqrt{\beta_1L}}{\beta_2} + 2^k}\right) + \sqrt{\beta_1 L\left(1 + \frac{2^{i^\star}}{\frac{\sqrt{\beta_1L}}{\beta_2} + 2^k}\right)}  \right)   \tag{by the definition of $\alpha_{i^\star}$} \\
         &= \otil\left( \beta_2(2^k + 2^{i^\star}) + \left(\sqrt{\beta_1} + \beta_2 c_{\max} + \beta_3\right)\left(1 + \frac{\beta_2 2^{i^\star}}{\beta_2 2^k}\right) + \sqrt{\beta_1 L\left(1 + \frac{\beta_2 2^{i^\star}}{\sqrt{\beta_1L}}\right)}  \right) \\
         &= \otil\left( \sqrt{\beta_1 L} + \beta_2 2^{i^\star} + \beta_2 c_{\max} + \beta_3 +  \sqrt{\sqrt{\beta_1 L} \times \beta_2 2^{i^\star}} \right)    \tag{$\beta_2 2^k\geq \beta_22^{k_{\init}}\geq \sqrt{\beta_1} + \beta_2c_{\max} + \beta_3$ and $\beta_2 2^k\leq \sqrt{\beta_1 L}$ as chosen in \gapalgname}  \\
         &=\otil\left( \sqrt{\beta_1 L} + \beta_2 C_{L_0} + \beta_2 c_{\max} + \beta_3 \right).  \tag{AM-GM and $2^{i^\star}=\order(\max\{C_{L_0}, 2^k\})$ by the definition of $i^\star$ in \pref{lem: general regret}}   
     \end{align*}
\end{proof}

\begin{proof}\textbf{of \pref{lem: execute most times}. }
Since $2^k\geq 32C_{L_0} \geq 32C_t$ for all $t$ during execution, by \pref{lem: istar = k case}, with probability at least $1-\order(k_{\max}\delta)$, \pref{eq: terminate condition 1} will not hold. Therefore, \singleepoch will finish all $L$ steps (thus $L_0=L$) and return \textbf{true}. 
     By \pref{lem: easy lemma 1}, we have that for $\alg_k$, with probability $1-\order(\delta)$,  
\begin{align*}
    &\sum_{t=1}^L (r_t^{\pi^*} - r_{t})\one[i_t=k] \\
    &\leq  \frac{\beta_1}{\Delta} + \beta_2 \thres_k + \beta_3 \\
    &\leq  \frac{\beta_1}{\Delta} + \beta_2 \left(\frac{5}{4}\times 2^{k} + 21c_{\max}\log(T/\delta)\right) + \beta_3 \tag{by the definition of $\thres_k$} \\
    &\leq \frac{2\beta_4}{\Delta} + \frac{5}{4}\beta_22^k. \tag{by the definition of $\beta_4$ and that $\Delta\leq 1$}
\end{align*}
On the other hand, 
\begin{align*}
    &\sum_{t=1}^L (r_t^{\pi^*} - r_{t})\one[i_t=k] \\
    &\geq \sum_{t=1}^L (\mu^{\pi^*} - \mu^{\pi_t})\one[i_t=k] + \sum_{t=1}^L \left(r_t^{\pistar} - \mu_t^{\pistar}(x_t)+ \mu_t^{\pi_t}(x_t)-r_t^{\pi_t}\right)\one[i_t=k] - 2\sum_{t=1}^L  \max_{\pi}|\mu^\pi-\mu_t^\pi|\\
    &\geq \Delta \left(\sum_{t=1}^L\one[\pi_t\neq \pi^*]\one[i_t=k]\right) - 8\sqrt{\sum_{t=1}^L \one[\pi_t\neq \pi^*]\one[i_t=k]\log(T/\delta)} - 18\log(T/\delta) - 2C_L \tag{by \pref{lem: aux0} with $y_t=\one[i_t=k]\one[\pi_t\neq \pi^\star]$, $X_t=r_t^{\pistar} - \mu_t^{\pistar}(x_t)+ \mu_t^{\pi_t}(x_t)-r_t^{\pi_t}$}\\
    &\geq \frac{3}{4}\Delta \left(\sum_{t=1}^L\one[\pi_t\neq \pi^*]\one[i_t=k]\right) - \frac{34\log(T/\delta)}{\Delta} - 2C_L. \tag{AM-GM and that $\Delta\leq 1$} 
\end{align*}
Combining the two inequalities above, and using that $\beta_1\geq 5\log(T/\delta), \beta_2 \geq 1, \beta_3\geq 5\log(T/\delta)$, we get 
\begin{align}
    &\sum_{t=1}^L  \one[\pi_t\neq \pi^*]\one[i_t=k] \nonumber \\
    &\leq \frac{4}{3\Delta}\left( \frac{2\beta_4}{\Delta} + \frac{5}{4}\beta_2 2^{k} + \frac{34\log(T/\delta)}{\Delta} + 2C_L\right) \nonumber
    \\
    &\leq  \frac{4}{3\Delta}\left( \frac{2\beta_4}{\Delta} + \frac{7\beta_4}{\Delta} + 2C_L\right) + \frac{5}{3\Delta}\beta_22^k \tag{$\beta_4\geq \beta_1\geq 5\log(T/\delta)$} \\
    &\leq \frac{4}{3\Delta}\beta_2 2^k + \frac{5}{3\Delta}\beta_2 2^k   \tag{by the condition specified in the lemma and that $L=L_0$}\\
    &=   \frac{3\beta_2 2^k}{\Delta}.  \label{eq: upper bound alg k}
\end{align}
We also have
\begin{align}
    &\sum_{t=1}^L \one[i_t=k] \nonumber \\ 
    &\geq  \frac{3}{4}\alpha_k L - 21\log(T/\delta) \tag{by \pref{lem: relate N and alpha}} \nonumber \\
    &\geq \frac{3}{8} \frac{(\beta_2 2^{k})^2}{\beta_4} - 5\beta_3   \tag{$\alpha_k\geq \frac{1}{2}$ and $\sqrt{\beta_4 L}\geq \beta_2 2^k$ and $\beta_3 \geq 5\log(T/\delta)$}  \nonumber \\
    &\geq \frac{12\beta_2 2^k}{\Delta} - \beta_2 2^k   \tag{ $\beta_2 2^k \geq \frac{32\beta_4}{\Delta}\geq 32\beta_3$ by the condition specified in the lemma}\\
    &\geq \frac{11\beta_2 2^k}{\Delta}. \tag{$\Delta\leq 1$} \\
    &\label{eq: lower bound alg k}
\end{align}
Combining \pref{eq: upper bound alg k} and \pref{eq: lower bound alg k} proves the lemma.  
\end{proof}

\section{Omitted Proofs in \pref{sec: two-armed alg}} \label{app: omitted proof for gap bound}
For \twomodel, we define the following notations: 
\begin{definition}
    Let $\calE_j$ be the set of rounds in epoch $j$, i.e., $\calE_j\triangleq [t_j, t_{j+1}-1]$. Let $\calE_j'$ be the set of rounds in epoch $j$ except for the last round, i.e., $\calE_j'\triangleq [t_j, t_{j+1}-2]$ (might be empty if $t_{j+1}= t_j+1$). 
\end{definition}

\begin{definition}
    Let $\calI\subseteq \calE_j$ be any interval in epoch $j$. Define \begin{alignat*}{2}
        &\Nht_{\calI, 0} \triangleq \sum_{t\in\calI} \one[Y_t=0], \qquad 
        && \Nht_{\calI, 1}  \triangleq \sum_{t\in\calI} \one[Y_t=1],  \\
        &\Rht_{\calI, 0} \triangleq \frac{1}{1-p_j} \sum_{t\in\calI} \one[Y_t=0]r_t, \qquad 
        && \Rht_{\calI, 1} \triangleq \frac{1}{p_j} \sum_{t\in\calI} \one[Y_t=1]r_t, \\
        &\Rht_{\calI, 0}^\pi \triangleq \frac{1}{1-p_j} \sum_{t\in\calI} \one[Y_t=0]r_t^\pi, \qquad 
        && \Rht_{\calI, 1}^\pi \triangleq \frac{1}{p_j} \sum_{t\in\calI} \one[Y_t=1]r_t^\pi. 
    \end{alignat*}
\end{definition}
\begin{definition}
    With abuse of notations, define $C_{\calI}\triangleq \sum_{\tau\in\calI}c_\tau$ (recall that we also define $C_t=\sum_{\tau=1}^t c_\tau$). 
\end{definition}

\begin{definition}
    $\Theta_j\triangleq \frac{5}{4}p_jC_{\calE_j} + 21c_{\max}\log(T/\delta)$. 
\end{definition}
Below, we first establish some basic lemmas:
\begin{lemma}\label{lem: connecting Regret}
    Let $\calR(t,\theta)=\sqrt{\beta_1 t} + \beta_2\thres + \beta_3$ with $\beta_3\geq 10\sqrt{\beta_1\log(T/\delta)}$. Then with probability at least $1-\order(\delta)$, the following holds for all $\calI=[t_j, t]\subseteq \calE_j$ and any $\theta\geq 0$: 
    \begin{align*}
        \frac{1}{2}\calR(p_j |\calI|, \theta) \leq \calR(\Nht_{\calI,1}, \theta)\leq \frac{3}{2}\calR(p_j |\calI|, \theta).
    \end{align*}
\end{lemma}
\begin{proof}
    With probability at least $1-\order(\delta)$, 
    \begin{align*}
        \calR(\Nht_{\calI,1}, \theta) 
        &= \sqrt{\beta_1  \Nht_{\calI,1}} + \beta_2\theta + \beta_3 \\
        &\leq \sqrt{\frac{5}{4}\beta_1  p_j |\calI| + 21\beta_1 \log(T/\delta)} + \beta_2\theta + \beta_3   \tag{by the same argument as \pref{lem: relate N and alpha}} \\
        &\leq \frac{3}{2}\sqrt{\beta_1 p_j |\calI|} + \beta_2\theta + 5\sqrt{\beta_1\log(T/\delta)} + \beta_3   \\
        &\leq \frac{3}{2}\calR(p_j |\calI|, \theta). \tag{$\beta_3\geq 10\sqrt{\beta_1\log(T/\delta)}$} \\
    \end{align*}
    If $p_j |\calI|\geq 42\log(T/\delta)$, then
    \begin{align*}
        \calR(\Nht_{\calI,1}, \theta) 
        &= \sqrt{\beta_1  \Nht_{\calI,1}} + \beta_2\theta + \beta_3 \\
        &\geq \sqrt{\frac{3}{4}\beta_1  p_j |\calI| - 21\beta_1 \log(T/\delta)} + \beta_2\theta + \beta_3 \tag{by the same argument as \pref{lem: relate N and alpha}}  \\ 
        &\geq \frac{1}{2}\sqrt{\beta_1 p_j |\calI|} + \beta_2\theta + \beta_3 \\
        &\geq \frac{1}{2}\calR(p_j |\calI|, \theta);
    \end{align*}
    otherwise, we have $p_j |\calI|< 42\log(T/\delta)$ and 
    \begin{align*}
        \calR(\Nht_{\calI,1},\theta) 
        &\geq \beta_2\theta + \beta_3 \\
        &\geq 5\sqrt{\beta_1 \log(T/\delta)} + \beta_2\theta + \frac{1}{2}\beta_3 \tag{$\beta_3\geq 10\sqrt{\beta_1\log(T/\delta)}$} \\
        &\geq 5\sqrt{\beta_1 \times \frac{1}{42}p_j |\calI|} + \beta_2\theta + \frac{1}{2}\beta_3 \\
        &\geq \frac{1}{2}\calR(p_j |\calI|, \theta).  
    \end{align*}
\end{proof}

\begin{lemma}\label{lem: concentra}
    With probability at least $1-\order(\delta)$, for all interval $\calI=[t_j, t]\subseteq \calE_j$, 
    \begin{align} 
        \left|\Rht_{\calI,1} - |\calI| \mu^{\pi'} \right| 
        &\leq  \frac{3}{p_j}\calR_\calB\left(p_j|\calI|, \Theta_j\right),
        \label{eq: inequ1} \\
        \left|\Rht_{\calI,0} -  |\calI|\mu^{\hatpi}\right| &\leq 4\sqrt{|\calI|\log(1/\delta)} + 4\log(T/\delta) + C_\calI \leq \frac{1}{p_j}\calR_\calB(p_j |\calI|, \Theta_j).  \label{eq: inequ2}
    \end{align}
    where $\pi'=\argmax_{\pi'\in\Pi\backslash\{\hatpi\}}\mu^{\pi'}$. 
    
\end{lemma}
\begin{proof}
    By the same argument as \pref{lem: aux1}, the corruption experienced by $\calB_{\hatpi}$ in $\calI$ is upper bounded by $\Theta_j$.  By the regret guarantee of $\calB_{\hatpi}$, we have with probability at least $1-\order(\delta)$, 
    \begin{align*}
        \Rht_{\calI,1} 
        &\geq \Rht_{\calI,1}^{\pi'} - \frac{1}{p_j}\calR_\calB\left(\Nht_{\calI,1}, \Theta_j\right) \\
        &\geq \sum_{t\in\calI} \mu^{\pi'}_t - 2\sqrt{\frac{2|\calI|\log(T/\delta)}{p_j}} - \frac{2\log(T/\delta)}{p_j} - \frac{2}{p_j} \calR_\calB\left(p_j|\calI|, \Theta_j\right)  \tag{by \pref{lem: beygel freedman} with a union bound over $|\calI|$ and \pref{lem: connecting Regret}} \\
        &\geq |\calI| \mu^{\pi'} - C_{\calI} - 2\sqrt{\frac{2|\calI|\log(T/\delta)}{p_j}} - \frac{2\log(T/\delta)}{p_j} - \frac{2}{p_j} \calR_\calB\left(p_j|\calI|, \Theta_j\right) \\
        &\geq |\calI|\mu^{\pi'} - \frac{3}{p_j}\calR_\calB\left(p_j|\calI|, \Theta_j\right).  \tag{by the assumption $\beta_1\geq 8\log(T/\delta), \beta_2 \geq 1, \beta_3\geq 5\log(T/\delta)$}
    \end{align*}
    Again by \pref{lem: beygel freedman} with a union bound over $|\calI|$, we also have with probability $1-\order(\delta)$, 
    \begin{align*}
        \Rht_{\calI,1}  
        &\leq \sum_{t\in\calI} \mu_t^{\pi'} + 2\sqrt{\frac{2|\calI|\log(T/\delta)}{p_j}} + \frac{2\log(T/\delta)}{p_j} \\
        &\leq |\calI| \mu^{\pi'} + 2\sqrt{\frac{2|\calI|\log(T/\delta)}{p_j}} + \frac{2\log(T/\delta)}{p_j} + C_\calI\\
        &\leq |\calI| \mu^{\pi'} + \frac{1}{p_j}\calR_\calB\left(p_j|\calI|, \Theta_j\right).  
    \end{align*}
    Combining them, we get \pref{eq: inequ1}. The first inequality in \pref{eq: inequ2} can be obtained by \pref{lem: beygel freedman} with the fact that $1-p_j\geq \frac{1}{2}$ and $\left|\sum_{t\in\calI}\big(\mu_t^{\pi'} - \mu^{\pi'}\big)\right|\leq C_\calI$; the second inequality in \pref{eq: inequ2} can be obtained using the assumptions on $\beta_1, \beta_2, \beta_3$. 
\end{proof}

\begin{lemma}\label{lem: delta smaller than 1}
    For all $j$, $\hatDelta_j\leq 1$. 
\end{lemma}
\begin{proof}
    When $j=1$, $\hatDelta_1\leq 1$ by definition. Assume that $j+1\geq 2$ is the first $j$ such that $\hatDelta_{j+1} > 1$. By the way we update $\hatDelta_j$, it must be that $\hatDelta_{j}\geq \frac{1}{1.25}$ and that at the end of epoch $j$, \pref{eq: checkk 2} is triggered. 
    
    However, notice that in \pref{eq: checkk 2}, the left-hand side $\Rht_0 = \frac{1}{1-p_j}\sum_{\tau=t_{j}}^t r_\tau\one[Y_\tau=0]\leq \frac{1}{1-p_j}M_{j}\leq 2M_j$ since $p_j=\frac{\beta_4}{2M_j\hatDelta_j^2}\leq \frac{1}{2}$ by \pref{eq: doubling epoch length}, but the right-hand side of \pref{eq: checkk 2} involves a term $3M_j\hatDelta_j\geq 3M_j\times \frac{1}{1.25} > 2M_j$. Therefore, \pref{eq: checkk 2} is impossible to be triggered at this $j$, contradicting our assumption.  
\end{proof}

\begin{lemma}\label{lem: magic lemma}
    With probability at least $1-\order(\delta)$, for all interval $\calI=[t_j, t]\subseteq \calE_j$, 
    \begin{align*}
       \frac{1}{p_j}\calR_{\calB}\left(\Nht_{\calI, 1}, \Theta_j\right)\leq \frac{2}{p_j}\calR_{\calB}\left(p_j|\calI|, \Theta_j+ \frac{p_j\sqrt{\beta_1 L}}{\beta_2}\right) \leq 0.02M_j\hatDelta_j + 2.5\beta_2 C_{\calE_j}\log(T/\delta) +  2\sqrt{\beta_1 L}. 
    \end{align*}
\end{lemma}
\begin{proof}
    \begin{align*}
        &\frac{1}{p_j}\calR_{\calB}\left(\Nht_{\calI, 1}, \Theta_j\right)\\
        &\leq  \frac{2}{p_j}\calR_\calB\left(p_j|\calI|, \Theta_j + \frac{p_j\sqrt{\beta_1 L}}{\beta_2}\right) \tag{by \pref{lem: connecting Regret}}\\
        &= 2\left(\sqrt{\frac{\beta_1|\calI|}{p_j}} + \frac{\beta_2\Theta_j}{p_j} + \frac{\beta_3}{p_j}\right) + 2\sqrt{\beta_1 L} \\
        &\leq 2\left( \sqrt{\frac{2\beta_1}{\beta_4}|\calI|M_j}\hatDelta_j + \frac{5}{4}\beta_2C_{\calE_j} + \frac{21\beta_2 c_{\max}\log(T/\delta)+\beta_3}{\beta_4}\times 2 M_j\hatDelta_j^2 \right) + 2\sqrt{\beta_1 L} \\
        &\leq 0.02M_j\hatDelta_j + 2.5 \beta_2 C_{\calE_j} +  2\sqrt{\beta_1 L}.   \tag{by the definition of $\beta_4$ and that $\hatDelta_j\leq 1$ by \pref{lem: delta smaller than 1}}
    \end{align*}

\end{proof}

\begin{proof}\textbf{of \pref{lem: corral regret 1}. }
    Let $T_0$ be the round at which \twomodel terminates. In the following proof, we assume that the high-probability events defined in previous lemmas hold. 
    \paragraph{Case 1. $\hatpi\neq \pistar$. }
    \begin{align*}
        &\sum_{t\in\calE_j'} (r_t^{\pi^\star} - r_t) \\
        &= (1-p_j)\left(\Rht^{\pistar}_{\calE_j',0} - \Rht^{\hatpi}_{\calE_j',0}\right) + p_j\left(\Rht^{\pistar}_{\calE_j',1} - \Rht_{\calE_j',1}\right) \\
        &\leq (1-p_j)\left(\Rht^{\pistar}_{\calE_j',1} - \Rht_{\calE_j',0}\right) + (1-p_j)\left(2\sqrt{\frac{2|\calE_j'|\log(T/\delta)}{p_j}} + \frac{2\log(T/\delta)}{p_j}\right) + p_j\left(\Rht^{\pistar}_{\calE_j',1} - \Rht_{\calE_j',1}\right)   \tag{when $Y_t=0$ we execute $\hatpi$, and by \pref{lem: beygel freedman}} \\
        &\leq (1-p_j)\left(\Rht_{\calE_j',1} - \Rht_{\calE_j',0}\right) + \frac{1}{p_j}\calR_\calB\left(p_j|\calE_j'|, 0\right) + \left(\Rht^{\pistar}_{\calE_j',1} - \Rht_{\calE_j',1}\right)  \tag{by the assumptions on $\beta_1, \beta_3$}  \\
        &\leq \frac{5}{p_j}\calR_\calB\left(p_j|\calE_j'|,  \frac{p_j\sqrt{\beta_1 L}}{\beta_2}\right) - \frac{1}{2}|\calE_j'|\hatDelta_j + \frac{1}{p_j}\calR_\calB\left(p_j|\calE_j'|, 0\right) + \frac{1}{p_j}\calR_\calB\left(\Nht_{\calE_j', 1}, \Theta_j\right)    \tag{by \pref{eq: checkk 1}, and \pref{assum: corruption robust sub routine} with the assumption that $\pi^\star\in \Pi\backslash \{\hatpi\}$}\\
        &\leq \frac{8}{p_j}\calR_\calB\left(p_j|\calE_j'|, \Theta_j + \frac{p_j\sqrt{\beta_1 L}}{\beta_2}\right) - \frac{1}{2}|\calE_j|\hatDelta_j   \tag{by \pref{lem: connecting Regret}} \\
        &\leq 0.16M_j\hatDelta_j + 20\beta_2 C_{\calE_j} + 16\sqrt{\beta_1 L} - 0.5|\calE_j|\hatDelta_j.   \tag{by \pref{lem: magic lemma}}
    \end{align*}
    For $j\geq 2$, the last expression is further upper bounded by
    \begin{align*}
        &\left(0.32 |\calE_{j-1}| + \frac{\beta_4}{\hatDelta_j^2}\right)\hatDelta_j + 20\beta_2 C_{\calE_j} + 16\sqrt{\beta_1 L} - 0.5|\calE_j|\hatDelta_j   \tag{by the definition of $M_j$} \\
        &\leq \frac{\beta_4}{\hatDelta_j} + 20\beta_2 C_{\calE_j} + 16\sqrt{\beta_1 L} + 0.4|\calE_{j-1}|\hatDelta_{j-1} - 0.5|\calE_j|\hatDelta_j    \tag{$\hatDelta_{j}\leq 1.25\hatDelta_{j-1}$}\\
        &\leq \otil\left(\sqrt{\beta_4 L} + \beta_2 C_{\calE_j}+\beta_4\right) + 0.4|\calE_{j-1}|\hatDelta_{j-1} - 0.5|\calE_j|\hatDelta_j; \tag{$\hatDelta_j\geq \hatDelta_1=\min\Big\{\sqrt{\frac{\beta_4}{L}}, 1\Big\}$}
    \end{align*} 
    for $j=1$, it is upper bounded by
    \begin{align*}
        \frac{0.16\beta_4}{\hatDelta_1} + 20\beta_2 C_{\calE_1} + 16\sqrt{\beta_1 L} - 0.5|\calE_1|\hatDelta_1  = \otil\left(\sqrt{\beta_4 L} + \beta_2 C_{\calE_1} + \beta_4\right) - 0.5|\calE_1|\hatDelta_1.  
    \end{align*}
    
    Summing up the above bound over $j$ (and noticing that the number of epochs is upper bounded by $3\log^2 T$), we see that 
    \begin{align*}
        &\sum_{t=1}^{T_0} (r_t^{\pi^\star} - r_t) \\
        &\leq \otil\left(\sqrt{\beta_4 L} + \beta_2 C_{\calE_1}+ \beta_4\right) - 0.5|\calE_1|\hatDelta_1 + \sum_{j=2}^{3\log^2 T} \left(\otil\left(\sqrt{\beta_4 L} + \beta_2 C_{\calE_j}+ \beta_4\right) + 0.4|\calE_{j-1}|\hatDelta_{j-1} - 0.5|\calE_j|\hatDelta_j\right)  \\
        & = \otil\left(\sqrt{\beta_4 L} + \beta_2 C_{T_0}+ \beta_4\right). 
    \end{align*}

    \paragraph{Case 2. $\hatpi= \pistar$. }
    \begin{align*}
        \sum_{t\in\calE_j'}\left( r^{\pistar}_t- r_t \right) 
        &= p_j\left(\Rht_{t,1}^{\pistar} - \Rht_{t,1}\right) \\
        &\leq  p_j\left(\Rht_{t,0}^{\pistar} - \Rht_{t,1}\right) +  \otil\left( \sqrt{p_j|\calE_j'|} + 1\right) \tag{\pref{lem: beygel freedman}} \\
        &\leq p_j\left(\Rht_{t,0} - \Rht_{t,1}\right) +  \otil\left( \sqrt{p_jM_j} + 1\right)  \tag{since $\hatpi=\pistar$}\\
        &\leq \otil\left( p_j M_j\hatDelta_j + \sqrt{\beta_1 L} + \sqrt{p_j M_j} \right)   \tag{by \pref{eq: checkk 2}}\\
        &\leq \otil\left(\frac{\beta_4}{\hatDelta_j} + \sqrt{\beta_1 L}\right)    \tag{by the definition of $p_j$}\\
        &= \otil\left(\sqrt{\beta_4 L} + \beta_4 \right).  \tag{$\hatDelta_j\geq \hatDelta_1=\min\Big\{\sqrt{\frac{\beta_4}{L}}, 1\Big\}$}
    \end{align*}
    Similarly, summing over epochs and using the fact that the number of epochs is upper bounded by $\order(\log^2 T)$ we get the desired bound. 
    
    

\end{proof}

\begin{proof}\textbf{of \pref{lem: no end}. }
     The condition in the lemma implies $\Delta\geq 16\sqrt{\frac{\beta_4}{L}} = 16\hatDelta_1$. 
     Below we prove by induction that $\Delta\geq \hatDelta_j$ for all $j$. This holds for $j=1$. Notice that $\hatDelta_j$ only increases when the second break condition \pref{eq: checkk 2} holds. 
     If \pref{eq: checkk 2} holds, we have 
    \begin{align*}
         |\calE_j|\Delta 
         &= |\calE_j|(\mu^{\pi^*} -\mu^{\pi'})\\ 
         &\geq \Rht_{\calE_j, 0} - \Rht_{\calE_j, 1} - \frac{4}{p_j}\calR_\calB\left(\Nht_{\calE_j,1}, \Theta_j\right) \tag{\pref{lem: concentra}} \\
         &\geq 3M_j\hatDelta_j + 9\sqrt{\beta_1 L}  - 4\left(0.02M_j\hatDelta_j + 2.5\beta_2 C_{\calE_j} +  2\sqrt{\beta_1 L}\right).  \tag{by \pref{eq: checkk 2} and \pref{lem: magic lemma}} \\
         &\geq 2.5 M_j\hatDelta_j. \tag{by the condition specified in the lemma, we have $\sqrt{\beta_1 L}\geq 10\beta_2 C_{\calE_j}$} 
    \end{align*} 
    Because $|\calE_j|\leq M_j$, we have $\hatDelta_j\leq \frac{1}{2.5}\Delta$. Therefore, after the update, $\hatDelta_{j+1}=1.25\hatDelta_j\leq \Delta$ still holds. 

    Next, we show that the first break condition \pref{eq: checkk 1} will not hold with high probability: at any time $t$ within epoch $j$,  
    \begin{align*}
        &\Rht_{[t_j, t-1], 0} - \Rht_{[t_j, t-1],1} - \frac{1}{2}(t-t_j)\hatDelta_j\\
        &= \left(\Rht_{[t_j,t-1],0} - (t-t_j) \mu^{\pi^*}\right) +  (t-t_j)\left(\mu^{\pi^*} - \mu^{\pi'} - \frac{1}{2}\hatDelta_j\right) + \left((t-t_j) \mu^{\pi'} - \Rht_{[t_j, t],1}\right)  \\
        &\geq  -  \frac{4}{p_j}\calR_\calB\left(p_j(t-t_j), p_j C_{\calE_j}\right) + (t-t_j)\left(\Delta - \frac{1}{2} \hatDelta_j\right) \tag{\pref{lem: concentra}}\\
        &\geq  -  \frac{4}{p_j}\calR_\calB\left(p_j(t-t_j), \frac{p_j\sqrt{\beta_1 L}}{\beta_2} \right). \tag{$\beta_2 C_{\calE_j}\leq \sqrt{\beta_1 L}$ as assummed in the lemma; $\Delta \geq \hatDelta_j$ for all $j$ as we just showed above} 
    \end{align*}
    Therefore, the first break condition will not be triggered. Overall, with high probability, $\hatDelta_j$ is non-decreasing with $j$.  
    
    Under this high-probability event, since $\hatDelta_j$ never decreases, the number of times $\hatDelta_j$ increases is upper bounded by $\log_{1.25}\frac{\Delta}{\hatDelta_1}\leq \log_{1.25}\sqrt{\frac{L}{\beta_4}}\leq \frac{1}{2}\log_{1.25} T\leq 2\log_2 T$. Furthermore, between two times $\hatDelta_j$ increases, since \pref{eq: checkk 1} and \pref{eq: checkk 2} are not triggered, the epoch length is at least two times the previous one (by \pref{eq: doubling epoch length}). Therefore, between two times $\hatDelta_j$ increases, the number of epochs is upper bounded by $\log_2 T$.  Overall, the total number of epochs is upper bounded by $2\log_2 T\times \log_2 T=2\log^2 T$. Since we allow the maximum number of epochs to be $3\log^2 T$ in \pref{alg: corral}, it will not end before the number of rounds reaches $T$. 
\end{proof}

\begin{proof}\textbf{of \pref{thm: gap bound thm}. }
Let $L^\star$ be the smallest $L$ such that 
\begin{align*}
    32\left(\frac{\beta_4}{\Delta} + \beta_2 C \right)\leq \sqrt{\beta_4 L}.  
\end{align*}
In the for-loops where the learner uses $L\leq L^\star$, by \pref{lem: form 1 regret with gap} and \pref{lem: corral regret 1}, the sum of regret in \phaseone and \phasetwo is upper bounded by 
\begin{align*}
    \otil\left(\sqrt{\beta_4 L} + \beta_2 C + \beta_4\right) = \otil\left(\sqrt{\beta_4 L^\star} + \beta_2 C + \beta_4\right) = \otil\left(\frac{\beta_4}{\Delta} + \beta_2 C \right). 
\end{align*}
In the for-loop where the learner first time uses $L > L^\star$, we have 
\begin{align*}
    \beta_2 2^k \geq \sqrt{\beta_4(L-1)}\geq \sqrt{\beta_4 L^\star} \geq 32\left(\frac{\beta_4}{\Delta} + \beta_2 C  \right) 
\end{align*}
where the first inequality is by the choice of $L$ in \algname. By \pref{lem: execute most times}, with probability at least $1-\order(\delta)$,  $\hatpi=\pistar$. Further by \pref{lem: no end}, with high probability, \phasetwo will continue until the total number of rounds reaches $T$. In this case, using \pref{lem: form 1 regret with gap} and \pref{lem: corral regret 1}, we can still bound the regret in the remaining steps by
\begin{align*}
    \otil\left(\sqrt{\beta_4 L} + \beta_2 C + \beta_4\right) = \otil\left(\sqrt{\beta_4 L^\star} + \beta_2 C + \beta_4\right) = \otil\left(\frac{\beta_4}{\Delta} + \beta_2 C\right). 
\end{align*}

By the discussions above, we also see that with high probability, in all for-loops, the learner uses $L < 2L^\star$ (because the algorithm will be locked in Phase~2 when the first time $L>L^\star$ happens). Therefore, by the condition of starting \phasethree, \phasethree can only be reached when $L^\star=\Omega(T)$. In this case, the regret incurred in \phasethree, by \pref{thm: form 1 regret}, is upper bounded by 
\begin{align*}
    \otil\left(\sqrt{\beta_1 T} + \beta_2 C + \beta_3\right) = \otil\left( \sqrt{\beta_1 L^\star} + \beta_2 C + \beta_3 \right) = \otil\left(\frac{\beta_4}{\Delta} + \beta_2 C + \beta_4\right) = \otil\left(\frac{\beta_4}{\Delta} + \beta_2 C\right). 
\end{align*}
Overall, after summing the regret in all phases and using the fact that the for-loop only repeat $\otil(1)$ times, we see that the total regret can be upper bounded by $\otil\left(\frac{\beta_4}{\Delta} + \beta_2 C\right)$. To show that the algorithm also simultaneously guarantees a bound of $\otil\left(\sqrt{\beta_4 T} + \beta_2 C + \beta_4\right)$, simply bound the regret in all phases by \sloppy$\otil\left(\sqrt{\beta_4 L} + \beta_2 C + \beta_4\right)=\otil\left(\sqrt{\beta_4 T} + \beta_2 C + \beta_4\right)$.

\end{proof}

\section{The Implementation of the Leave-one-policy-out MDP}\label{app: leave one out}
We consider a tabular MDP $\calM=(\calS, \calA, r, p, H)$ with a fixed initial state $s_1\in\calS$. Let $\Pi_\calM$ denote the set of all deterministic policies in $\calM$. 
Now, given a deterministic policy $\hatpi\in\Pi_\calM$, our goal is to construct another MDP $\calM'$, such that the policy set of $\calM'$ includes all policies in $\calM$ except for $\hatpi$, and that for any $\pi\in \Pi_\calM\backslash\{\hatpi\}$, the expected reward in $\calM$ and $\calM'$ is the same.  

MDP $\calM'$ has state space $\{s_0\} \cup \mathcal S \times \mathcal S$ and horizon $H+1$. In $\calM'$, the agent starts in the initial state $s_0$ and takes one of $S$ actions which makes it transition to one of $S$ copies of the original MDP $\calM$. The $s$-th copy of $\calM$ is denoted by $\calM_s$ and is identical to $\calM$ except that the agent is not allowed to take the actions prescribed by $\widehat \pi$ in  state $s$.

Note that we can obtain samples for $\calM'$ by playing in $\calM$, and that 
\begin{align}
    \max_{\pi\in \Pi_{\calM'} }\mu^\pi_{\calM'} = \max_{\pi\in \Pi_\calM \setminus \{\hatpi\}} \mu^\pi_{\calM}   \label{eq: desired identity}
\end{align}
where $\mu_{\calM}^{\pi}$ denotes the expected reward of policy $\pi$ under MDP $\calM$. To see this, simply notice that for any $\pi\in\Pi_\calM \setminus \{\hatpi\}$ which differs from $\pistar$ on state $s$, one can find a policy $\pi'\in\Pi_{\calM'}$ that first goes to $\calM_s$ in $\calM'$ in the first step, and then follow $\pi$ in the rest of the steps.  This policy $\pi'$ gives the same expected reward as $\pi$. Conversely, for any $\pi'\in\Pi_{\calM'}$, there is a policy $\pi\in\Pi_{\calM}\setminus\{\hatpi\}$ which simply equals to $\pi'$ on its $2$ to $H+1$ steps. This $\pi$ gives the same expected reward as $\pi'$. 

Although $\calM'$ has $S^2+1$ states, and the total number of actions is $(SA-1)\times S + S$ (where $SA-1$ is the total number of actions in each copy of $\calM$, and the additional $S$ is the number of actions on $s_0$), running UCBVI on $\calM'$ can in fact yield the  same gap-independent bound as running it in $\calM$ if we share the samples among different copies of $\calM$. To see this in the uncorrupted case, notice that in the analysis of the UCBVI algorithm (see, e.g., \citep{azar2017minimax}, or Chapter~7 of \citep{agarwal2020}), the regret bound is a sum of terms of the form $\sum_{t}\sum_h\frac{\poly(S, A, H)}{\sqrt{n_t(s^t_h,a^t_h)}}$ or $\sum_{t}\sum_h\frac{\poly(S, A, H)}{n_t(s^t_h,a^t_h)}$. When the samples of the $S$ copies of $\calM$ are shared, these sum will only scale with the original number of states and actions.
This can also be proved formally through the use of feedback graphs \citep{dann2020reinforcement}.
In the corrupted case, the amount of corruption (i.e., $c_t=H\cdot \sup_{s,a,V}|(\calT V - \calT_t V)(s,a)|$) remains the same in $\calM$ and in $\calM'$.
Therefore, the overall regret bound in $\calM'$ under corruption remains the same order as that in $\calM$. 

\section{Base Algorithms}\label{app: base algorithms}
In this section, we describe and analyze the base algorithms for all settings considered in \pref{tab:my_label}. At the end of each subsection, we calculate the precise regret bounds achieved by our approaches and compare them with those in previous works to complement \pref{tab:my_label}. The proofs are sometimes brief since they mostly follow standard analysis appeared in previous works. More details can be found in the references. 

\subsection{Robust UCBVI for tabular MDPs}\label{app: robust MVP}
A Robust UCBVI algorithm is presented by \citet{lykouris2019corruption} in their Appendix B. We translate it to our setting (i.e., our trajectory reward is bounded in $[0, 1]$, and our definition of $C^\am$ already includes an $H$ factor). The resulting algorithm essentially runs the standard UCBVI algorithm \citep{azar2017minimax} with enlarged bonuses 
\begin{align*}
    b_t(s,a) = \min\left\{2\sqrt{\frac{2\ln(64SAHT^2/\delta)}{n_t(s,a)}} + \frac{C^\am}{n_t(s,a)}, \ \ 1\right\}
\end{align*}
where $n_t(s,a)$ is the number of visits to $(s,a)$ before episode $t$, and $C^\am$ is a given upper bound of the total corruption. This algorithm achieves the following bound (c.f. Eq. (B.1) in \citet{lykouris2019corruption}): 
\begin{align}
    \sum_{\tau=1}^t \left(V_1^\star(s^t_1) - V_1^{\pi_t}(s^t_1)\right) \leq  \poly(H)\times\otil\left(\min\left\{\textsf{GapComplexity}, \sqrt{SAt}\right\} + S^2A + C^\am SA\right). \label{eq: Lykouris bound UCBVI}
\end{align}
Furthermore, by their definition of \textsf{GapComplexity} (with proper scaling for our setting), it holds that $\textsf{GapComplexity}\leq \frac{SA}{\textsf{gap}_{\min}}\leq \frac{SA}{\Delta}$, where $\textsf{gap}_{\min}\triangleq \min_{s,a\neq \pistar_h(s),h}\left(V^\star_h(s) - Q^\star_h(s,a)\right)$ and the second inequality is by the performance difference lemma, 
\begin{align*}
    \Delta = \min_{\pi\neq \pistar}\Delta_{\pi}
    &=\min_{\pi\neq \pistar}\E\left[\sum_{h,s,a} \Pr[s_h=s](V_h^\star(s)-Q^\star_h(s,a))~\bigg|~\pi\right]\\
    &\leq \min_{h,s,a\neq \pistar_h(s)} \E\left[ \Pr[s_h=s](V_h^\star(s)-Q_h^\star(s,a))~\big|~\pi: \pi_h(s')=\pistar_h(s') \forall s'\neq s, \pi_h(s)=a\right] \tag{Let $\pi$ be the policy that only differs from $\pistar$ on state $s$ at level $h$} \\
    &\leq \textsf{gap}_{\min}. 
\end{align*}
Below we use these facts to derive our bound. 

\begin{theorem}
    For finite-horizon tabular MDPs, $\algname$ with Robust UCBVI as the base algorithm guarantees $\Reg(T)=\otil\left(\textup{poly}(H)\times \left(\sqrt{SAT} + S^2A + SAC^\am\right)\right)$; \gapalgname with Robust UCBVI guarantees $\Reg(T)=\otil\left(\textup{poly}(H)\times \left(\min\left\{\frac{S^2A}{\Delta}, \sqrt{S^2AT}\right\} + S^2A + SAC^\am\right)\right)$.  
\end{theorem}
\begin{proof}
    By \pref{eq: Lykouris bound UCBVI} and Azuma's inequality, we have 
    \begin{align*}
    \sum_{\tau=1}^t \left(r_\tau^{\pistar} - r_\tau\right) \leq \poly(H)\times \otil\left( \sqrt{SAt} + S^2A +  SAC^\am\right),  
\end{align*}
which satisfies \pref{eq: typeone} with $\beta_1=\widetilde{\Theta}(\poly(H) SA)$, $\beta_2=\widetilde{\Theta}(\poly(H)SA)$, $\beta_3=\widetilde{\Theta}(\poly(H)S^2A)$. Applying \pref{thm: form 1 regret} with these parameters we get the bound for \algname. 

Let $N_t^{\neq \pistar}=\sum_{\tau=1}^t \one[\pi_\tau\neq \pistar]$ be the number of times the learner chooses sub-optimal policies. Using \pref{eq: Lykouris bound UCBVI} and noticing that the left-hand side of it is lower bounded by $N_t^{\neq \pistar}\Delta$, we get 
\begin{align*}
    N_t^{\neq \pistar}\Delta \leq 
    \poly(H)\times \otil\left(\frac{SA}{\Delta} + S^2 A + SAC^\am\right). 
\end{align*}
Therefore, by Azuma's inequality, we have with probability $1-\order(\delta)$, 
\begin{align*}
    \sum_{\tau=1}^t (r^{\pistar}_\tau - r_\tau) &\leq  \poly(H)\times\otil\left(\min\left\{\frac{SA}{\Delta}, \sqrt{SAt}\right\} + S^2A + C^\am SA + \sqrt{N_t^{\neq\pistar}}\right)\\
    &\leq  \poly(H)\times\otil\left(\min\left\{\frac{SA}{\Delta}, \sqrt{SAt}\right\} + S^2A + C^\am SA \right). 
\end{align*}
This satisfies \pref{eq: typetwo} with $\beta_1=\widetilde{\Theta}(\poly(H) SA)$, $\beta_2=\widetilde{\Theta}(\poly(H)SA)$, $\beta_3=\widetilde{\Theta}(\poly(H)S^2A)$ (therefore, $\beta_4=\widetilde{\Theta}(\poly(H)S^2A)$). Applying \pref{thm: gap bound thm} with these parameters we get the bound for \gapalgname. 
\end{proof}

\paragraph{Comparison with previous bounds } For corrupted tabular MDPs, the bound of \cite{lykouris2019corruption} is $\poly(H)\times \otil\left( (1+C^\am)\min\left\{\textsf{GapComplexity}, \sqrt{SAT}\right\} + S^2AC^\am + SA(C^\am)^2 \right)$; the bound of \cite{chen2021improved} is $\poly(H)\times \otil\left( \min\left\{\frac{S^2A^{\nicefrac{3}{2}}}{\Delta}, \sqrt{S^4A^3T}\right\} + S^2 A^2 C^{\am} + (C^{\am})^2\right)$; the bound of \cite{jin2021best} is $\poly(H)\times \otil\left(\min\left\{\textsf{GapComplexity}, \sqrt{S^2A^2 T}\right\} + C^{\am}\right)$ (under the uncorrupted transition assumption).

\subsection{Robust Phased Elimination for linear bandits}\label{app: robust phased elim}
\begin{algorithm}[t]
    \caption{Robust Phased Elimination}\label{alg: robust phased elimination}
    \textbf{input}: $C^\am$ \\
    \textbf{define}: $m_0=4d([\log\log d]_+ + 18)$ \\
    \textbf{initialize}: $\calA_0\leftarrow \calA$ \\
    \For{$k=0, 1, 2, \ldots$}{
        Let $m_k=2^{k-1}m_0$. \\ 
        Compute $\zeta_k: \calA_k\rightarrow [0,1]$ such that 
        \begin{align*}
            \max_{a\in\calA_k} \|a\|_{\Gamma(\zeta_k)^{-1}}\leq 2d  \quad \text{and} \quad   |\text{supp}(\zeta_k)|\leq m_0
        \end{align*}
        where $\Gamma(\zeta_k) = \sum_{a \in \calA_k} \zeta_k(a) a a^\top$.\\
        Set 
        \begin{align*}
            u_k(a) = \begin{cases}
                0 &\text{if}~\zeta_k(a)=0 \\
                \left\lceil m_k\max\{\zeta_k(a), \frac{1}{m_0}\} \right\rceil &\text{otherwise}
            \end{cases}
        \end{align*}
        Draw each action $a\in\calA_k$ exactly $u_k(a)$ times, and get action-reward pairs $(a_\tau, r_\tau)_{\tau=1}^{u_k}$, where $u_k=\sum_{a\in\calA_k}u_k(a)$. \\
        Estimate parameter:  
        \begin{align*}
            w_k = \Gamma_k^{-1} \sum_{\tau=1}^{u_k}a_\tau r_\tau, \qquad \text{where\ } \Gamma_k = \sum_{a\in\calA_k} u_k(a)aa^\top. 
        \end{align*}
        Update the active set 
        \begin{align*}
            \calA_{k+1}\leftarrow \left\{ a\in \calA_k~:~ \max_{a'\in\calA_k} w_k^\top (a'-a) \leq 4d\sqrt{\frac{1}{m_k}\log(T/\delta)} +  \frac{4 \sqrt{2d}m_0}{m_k}C^{\am}\right\}. 
        \end{align*}
    }
    
\end{algorithm}
The Robust Phased Elimination (\pref{alg: robust phased elimination}) is exactly the Algorithm 1 of \cite{bogunovic2021stochastic} with the choice of parameters specified in their Theorem 1 (i.e., $\nu=\frac{1}{m_0}$ in their notations). Its gap-independent bound is shown below: 
\begin{lemma}\label{lem: linear bandit gap-independent}
    With probability at least $1-\order(\delta)$, Robust Phased Elimination ensures 
    \begin{align*}
        \sum_{\tau=1}^t (r_\tau^{a^\star}-r_\tau)  = \otil\left(d\sqrt{t \log(T/\delta)} + C^{\am}d^{\nicefrac{3}{2}}\log T\right). 
    \end{align*}
\end{lemma}
\begin{proof}
    By Theorem 1 of \cite{bogunovic2021stochastic}, we have 
    \begin{align*}
        \sum_{\tau=1}^t (\mu^{a^\star} - \mu^{a_t}) = \otil\left(\sqrt{dt \log(T|\calA|/\delta)} + C^{\am}d^{\nicefrac{3}{2}}\log T\right) = \otil\left(d\sqrt{t \log(T/\delta)} + C^{\am}d^{\nicefrac{3}{2}}\log T\right)
    \end{align*}
    where for simplicity we assume $|\calA|=\order(T^d)$ without loss of generality.   
    The conclusion follows by noticing that $\sum_{\tau=1}^t  |\mu^{a^\star} - \mu_\tau^{a^\star}|\leq C^\am$, $\sum_{\tau=1}^t |\mu^{a_\tau} - \mu_\tau^{a_\tau}|\leq C^\am$ by the definition of corruption, and that with probability $1-\order(\delta)$,  $\left|\sum_{\tau=1}^t (r_\tau^{a^\star} - \mu_\tau^{a^\star})\right|=\order(\sqrt{t\log(T/\delta)})$, and $\left|\sum_{\tau=1}^t (r_\tau - \mu_\tau^{a_\tau})\right|=\order(\sqrt{t\log(T/\delta)})$ by Azuma's inequality. 
\end{proof}

Next, we further show that the same algorithm achieves a gap-dependent bound. We first restate an intermediate result of \cite{bogunovic2021stochastic}. 
\begin{lemma}[Appendix A.2 of \cite{bogunovic2021stochastic}]
    Robust Phased Elimination ensures that with probability at least $1-\order(\delta)$, $a^\star\in \calA_k$ for all $k$. 
\end{lemma}
The gap-dependent bound of Robust Phased Elimination is then given by the following lemma.
\begin{lemma}\label{lem: linear bandit gap-dependent}
    With probability $1-\order(\delta)$, Robust Phased Elimination ensures 
    \begin{align*}
        \sum_{\tau=1}^t \left(r_\tau^{a^\star}-r_\tau \right) = \otil\left( \frac{d^2\log (T/\delta)}{\Delta} + C^\am d^{\nicefrac{3}{2}}\log T \right). 
    \end{align*}
\end{lemma}
\begin{proof}
    By Eq.~(58) of \citep{bogunovic2021stochastic}, for all $a\in\calA_k$, $a\neq a^\star$, we have 
    \begin{align*}
        \Delta \leq w^{\star\top}(a^\star - a) \leq 8d\sqrt{\frac{\log(T/\delta)}{m_k}} +  \frac{8 \sqrt{2d}m_0}{m_k}C^{\am}
    \end{align*}
    where the first inequality is by our assumption. Solving the inequality we get 
    \begin{align}
        m_k\leq  \frac{256 d^2\log(T/\delta)}{\Delta^2} + \frac{16\sqrt{2d}m_0C^\am}{\Delta}.  \label{eq: m k bound} 
    \end{align}
    This means that as long as $m_k$ grows larger than the right-hand side of \pref{eq: m k bound}, no sub-optimal arm can remain in $\calA_k$. Let $k^\star$ be the smallest $k$ such that $m_{k}$ is larger than the right-hand side of \pref{eq: m k bound}. Then we only need to calculate the regret incurred in epochs $1, \ldots, k^\star -1$. By the same calculation as Eq.~(48)-(55) in \citep{bogunovic2021stochastic}, we get that with probability at least $1-\order(\delta)$, 
    \begin{align*}
        \sum_{\tau=1}^t (r_\tau^{a^\star}-r_\tau)
        &\leq\sum_{k=1}^{k^\star-1}\sum_{\tau\in \text{epoch}(k)}(r_\tau^{a^\star}-r_\tau) \\ 
        &\leq \sum_{k=1}^{k^\star-1}\sum_{\tau\in \text{epoch}(k)}(\mu^{a^\star}-\mu^{a_\tau}) + \order\left(\sum_{k=1}^{k^\star-1}\sqrt{m_k\log(T/\delta)}+ C^{\am}\right) \tag{Azuma's inequality}\\ 
        &=\order\left( u_0 + \sum_{k=1}^{k^\star-1}\left(d\sqrt{m_k\log(T/\delta)} + C^\am m_0\sqrt{d}\right) \right) \tag{By Eq.~(48)-(55) in \citep{bogunovic2021stochastic}}\\
        &= \order\left(m_0 + d\sqrt{m_{k^\star}\log(T/\delta)} + C^\am m_0\sqrt{d}\log T\right) \\
        &= \order\left(\frac{d^2\log(T/\delta)}{\Delta} + d^{\nicefrac{3}{2}}C^{\am}\log T \right).
    \end{align*}
\end{proof}
\begin{theorem}
    For linear bandits, \algname with Robust Phased Elimination as the base algorithm guarantees $\Reg(T)=\otil\left(d\sqrt{T} + d^{\nicefrac{3}{2}}C^\am + d^{\nicefrac{3}{2}}\right)$; \gapalgname with Robust Phased Elimination guarantees $\Reg(T)=\otil\left(\min\left\{\frac{d^2}{\Delta}, d\sqrt{T}\right\} + d^{\nicefrac{3}{2}}C^\am + d^{2}\right)$.
\end{theorem}
\begin{proof}
    By \pref{lem: linear bandit gap-independent} and \pref{lem: linear bandit gap-dependent}, we see that Robust Phased Elimination satisfies \pref{eq: typeone} and  \pref{eq: typetwo} with $\beta_1=\widetilde{\Theta}(d^2), \beta_2=\widetilde{\Theta}(d^{\nicefrac{3}{2}}), \beta_3=\Theta(d)$ (thus, $\beta_4=\widetilde{\Theta}(d^2)$). Using them in \pref{thm: form 1 regret} and \pref{thm: gap bound thm} gives the desired bounds. 
\end{proof}

\paragraph{Comparison with previous bounds } For corrupted linear bandits, the bound of \cite{li2019stochastic} is  $\otil\left(\frac{d^6}{\Delta^2} + \frac{d^{\nicefrac{5}{2}}C^\am}{\Delta}\right)$; the bound of \cite{bogunovic2021stochastic} is $\otil\left(  d\sqrt{T} + d^2(C^\am)^2\right)$ (against a stronger adversary); the bound of \cite{lee2021achieving} is $\otil\left(\min\left\{\frac{d^2}{\Delta}, d\sqrt{T}\right\} + C^{\am}\right)$ (under the linearized corruption assumption). 

\subsection{Robust OFUL for linear contextual bandits / Robust LSVI-UCB for linear MDPs}
\label{app: robust OFUL}
From this section, we denote the state, the action, and the reward at the $h$-th step of the $t$-th episode as $s^t_h, a^t_h$, and $\sigma^t_h$ respectively (same as the $s_{t,h}, a_{t,h}, \sigma_{t,h}$ defined in \pref{sec: problem setting}). 

Below we restate the linear MDP assumption in \citep{jin2020provably} (adapted to our case where the per-step reward lies in $[0, \frac{1}{H}]$): 
\begin{assumption}[Finite-horizon Linear MDP]\label{assum: linear MDP assumptio}
    Let $\phi(s,a)\in\mathbb{R}^d$ be known feature vector for the state-action pair $(s,a)$. Assume that for all $(s,a)$, the reward function can be represented as $\sigma(s,a)=\phi(s,a)^\top \rho$, and the transition kernel can be represented as $p(s'|s,a)=\phi(s,a)^\top\nu(s')$ for some $\nu(s')\in\mathbb{R}^d$. Without loss of generality, we assume that $\|\phi(s,a)\|\leq 1$, $\|\rho\|\leq \frac{1}{H}\sqrt{d}$, and $
    \left\|\int\nu(s')\mathrm{d}s'\right\|\leq \sqrt{d}$.  
\end{assumption}
In \pref{alg: robust LSVI-UCB}, we present a corruption robust version of LSVI-UCB \citep{jin2020provably} that takes $C^\rms$ as input. Since linear contextual bandit is a special case of linear MDP with $H=1$, we can use the same algorithm to deal with it. 

\begin{algorithm}[t]
    \caption{Robust OFUL / Robust LSVI-UCB}\label{alg: robust LSVI-UCB}
    \textbf{input}: $C^\rms$ \\
    \textbf{define}: $\zeta=\zeta_{0}\cdot \sqrt{d\log(dT/\delta)}$ for $H=1$ (linear contextual bandit case), or $\zeta=\zeta_{0}\cdot d\sqrt{\log(dHT/\delta)}$ for $H>1$ (linear MDP case), where $\zeta_{0}$ is a universal constant.  \\
    \For{$t=1, \ldots, T$}{
        \begin{align*}
            \Lambda^t = \sum_{\tau=1}^{t-1}\sum_{h=1}^H  \phi^\tau_h\phi^{\tau\top}_h +  I, \qquad \quad \text{where\ } \phi^\tau_h\triangleq \phi(s^\tau_h, a^\tau_h). 
        \end{align*}
        \For{$h=H, \ldots, 1$}{
            \begin{align*}
                w^t_h &= (\Lambda^t)^{-1}\sum_{\tau=1}^{t-1}\sum_{k=1}^H \phi^{\tau\top}_k\left(\sigma^\tau_k + V_{h+1}^t(s^\tau_{k+1})\right) \\
                Q^{t}_h(s,a) &\leftarrow \min\left\{w_h^{t\top}\phi(s,a) + \left(4\zeta + C^{\rms}\sqrt{\frac{d}{Ht}}\right)\|\phi(s,a)\|_{(\Lambda^t)^{-1}}, \ \ 1\right\}\\
                V^t_h(s)&\leftarrow \max_a Q^t_h(s,a) 
            \end{align*}
           
        }
        \For{$h=1,\ldots, H$}{
            Observe $s^t_h$,  choose $a^t_h=\argmax_a Q^t_h(s^t_h, a)$, and receive $\sigma^t_h$. 
        }
    }
\end{algorithm}
\allowdisplaybreaks
The following lemma is adapted from \citep[Lemma B.4]{jin2020provably}. 
\begin{lemma} \label{lem: linear MDP decomposition}
    For any $\pi$, with probability at least $1-\order(\delta)$, the following holds for all $s,a$: 
\begin{align*}
    \phi(s,a)^\top (w^t_h - w^\pi_h) = \E_{s'\sim p_t(\cdot|s,a)}\left[V^t_{h+1}(s') - V^{\pi}_{h+1}(s')\right] + \varepsilon^{t}_h(s,a) 
\end{align*}
for some $\varepsilon^{t}_h(s,a)$ that satisfies 
\begin{align}
    |\varepsilon^{t}_h(s,a)|\leq \left(4\zeta + C^{\rms}{\sqrt{\frac{d}{Ht}}}\right)\|\phi(s,a)\|_{(\Lambda^t)^{-1}} + \frac{2}{H}c_t  \label{eq: epsilon range}
\end{align}
\end{lemma}
\begin{proof}
For any $(s,a)$ and any $\pi$, 
\begin{align*}
    &\phi(s,a)^\top (w^t_h-w^\pi_h) \\
    &=\phi(s,a)^\top \left((\Lambda^t)^{-1}\sum_{\tau=1}^{t-1}\sum_{k=1}^H \phi^\tau_k\Big(\sigma^\tau_k+V_{h+1}^t(s_{k+1}^\tau)\Big)-w^\pi_h\right)\\
    &=\phi(s,a)^\top (\Lambda^t)^{-1}\left(\sum_{\tau=1}^{t-1}\sum_{k=1}^H \phi^\tau_k\left(\sigma^\tau_k+V_{h+1}^t(s_{k+1}^\tau) - \phi^{\tau\top}_k\rho - \phi^{\tau\top}_k \int\nu(s')V_{h+1}^{\pi}(s')\mathrm{d}s'\right) -  w^{\pi}_h\right)\\
    &=\phi(s,a)^\top (\Lambda^t)^{-1}\left(\sum_{\tau=1}^{t-1}\sum_{k=1}^H \phi^\tau_k\left(\sigma^\tau_k+V_{h+1}^t(s_{k+1}^\tau) - \sigma(s^\tau_k, a^\tau_k) - \E_{s'\sim p(\cdot|s^\tau_k, a^\tau_k)}V_{h+1}^{\pi}(s')\right) -  w^{\pi}_h\right)\\
    &= \underbrace{\phi(s,a)^\top (\Lambda^t)^{-1}\sum_{\tau=1}^{t-1}\sum_{k=1}^H  \phi^\tau_k\left(\sigma^\tau_k - \sigma_\tau(s^\tau_{k}, a^\tau_k) \right)}_{\term_1} \\
    &\quad + \underbrace{\phi(s,a)^\top(\Lambda^t)^{-1}\sum_{\tau=1}^{t-1}\sum_{k=1}^H  \phi^\tau_k\left(V^t_{h+1}(s^\tau_{k+1}) - \E_{s'\sim p_\tau(\cdot|s^\tau_k, a^\tau_k)} V^{t}_{h+1}(s')\right)}_{\term_3}   \\
    &\quad +  \underbrace{\phi(s,a)^\top(\Lambda^t)^{-1}\sum_{\tau=1}^{t-1}\sum_{k=1}^H  \phi^\tau_k\left(\sigma_\tau(s^\tau_{k}, a^\tau_k) +  \E_{s'\sim p_\tau(\cdot|s^\tau_k, a^\tau_k)} V^{t}_{h+1}(s')  -  \sigma(s^\tau_{k}, a^\tau_k) - \E_{s'\sim p(\cdot|s^\tau_k, a^\tau_k)} V^{t}_{h+1}(s') \right)}_{\term_3}   \\
    &\quad + \underbrace{\phi(s,a)^\top(\Lambda^t)^{-1}\sum_{\tau=1}^{t-1}\sum_{k=1}^H  \phi^\tau_k\left( \E_{s'\sim p(\cdot|s^\tau_k, a^\tau_k)} V^{t}_{h+1}(s') - \E_{s'\sim p(\cdot|s^\tau_k, a^\tau_k)} V^{\pi}_{h+1}(s') \right)}_{\term_4} 
    \ \ \ \underbrace{-  \phi(s,a)^\top(\Lambda^t)^{-1} w^\pi_h}_{\term_5} \\ 
\end{align*}
$\term_1$ and $\term_2$ are of the same form, Their absolute values $|\term_1|$ and $|\term_2|$ can both be upper bounded by $\zeta\|\phi(s,a)\|_{(\Lambda^t)^{-1}}$ with a similar proof as Lemma B.3 and Lemma D.4 of \citep{jin2020provably} (notice that our range of reward is smaller than theirs by a $\frac{1}{H}$ factor). 

Then notice that 
\begin{align*}
    \term_3 = \phi(s,a)^\top (\Lambda^t)^{-1}\sum_{\tau=1}^{t-1}\sum_{k=1}^H \phi^\tau_k \left(\calT_\tau V^t_{h+1}(s^\tau_k, a^\tau_k) - \calT V^t_{h+1}(s^\tau_k, a^\tau_k)\right)   \tag{$\calT_t$ and $\calT$ are Bellman operators defined in \pref{sec: problem setting}}
\end{align*}
and 
$|\term_3|$ is upper bounded by 
\begin{align*}
    &\|\phi(s,a)\|_{(\Lambda^t)^{-1}}\sum_{\tau=1}^{t-1}\sum_{k=1}^H  \|\phi^\tau_k\|_{(\Lambda^t)^{-1}} \times \frac{1}{H}c_\tau \tag{recall that $c_t\triangleq H\cdot \sup_{s,a}\sup_{V\in[0,1]^\calS}|(\calT V-\calT_t V)(s,a)|$}\\
    &\leq \frac{1}{H}\|\phi(s,a)\|_{(\Lambda^t)^{-1}}\sqrt{\sum_{\tau=1}^{t-1}\sum_{k=1}^H c_\tau^2}\sqrt{\sum_{\tau=1}^{t-1}\sum_{k=1}^H \|\phi^\tau_k\|_{(\Lambda^t)^{-1}}^2} \\
    &\leq  \frac{C^{\rms}}{\sqrt{Ht}}\times \sqrt{d} \|\phi(s,a)\|_{(\Lambda^t)^{-1}}.    \tag{by Lemma D.1 of \citep{jin2020provably}}
\end{align*}

\begin{align*}
    \term_4 &= \phi(s,a)^\top (\Lambda^t)^{-1}\sum_{\tau=1}^{t-1}\sum_{k=1}^H \phi^\tau_k\phi^{\tau\top}_k \left(\rho + \int \nu(s')\left(V^t_{h+1}(x')-V^{\pi}_{h+1}(s')\right)\mathrm{d}s'\right) \\
    &= \phi(s,a)^\top\left(\rho + \int \nu(s')\left(V^t_{h+1}(s')-V^{\pi}_{h+1}(s')\right)\mathrm{d}s'\right) \\
    &\qquad \qquad \underbrace{-  \phi(s,a)^\top (\Lambda^t)^{-1}\left(\rho + \int \nu(s')\left(V^t_{h+1}(s')-V^{\pi}_{h+1}(s')\right)\mathrm{d}s'\right)}_{\term_5} \\
    &= \E_{s'\sim p(\cdot|s, a)}\left[V^t_{h+1}(s') - V^{\pi}_{h+1}(s')\right] + \term_5 \\
    &= \E_{s'\sim p_t(\cdot|s, a)}\left[V^t_{h+1}(s') - V^{\pi}_{h+1}(s')\right] + \term_5 \\ 
    &\qquad + \underbrace{\left(\sigma(s,a) + \E_{s'\sim p(\cdot|s, a)}\left[V^t_{h+1}(s')\right]-\sigma_t(s,a) - \E_{s'\sim p_t(\cdot|s, a)}\left[V^t_{h+1}(s')\right]\right)}_{\term_6}\\
    &\qquad - \underbrace{\left(\sigma(s,a) + \E_{s'\sim p(\cdot|s, a)}\left[ V^{\pi}_{h+1}(s')\right]-\sigma_t(s,a) - \E_{s'\sim p_t(\cdot|s, a)}\left[ V^{\pi}_{h+1}(s')\right]\right)}_{\term_7}
\end{align*}
Furthermore, 
\begin{align*}
    |\term_5|
    &\leq \|\phi(s,a)\|_{(\Lambda^t)^{-1}}\left\|\rho + \int \nu(s')\left(V^t_{h+1}(s')-V^{\pi}_{h+1}(s')\right)\mathrm{d}s' \right\|_{(\Lambda^t)^{-1}} \\
    &\leq  \|\phi(s,a)\|_{(\Lambda^t)^{-1}} \left\|\rho + \int \nu(s')\left(V^t_{h+1}(s')-V^{\pi}_{h+1}(s')\right)\mathrm{d}s' \right\| \\
    &\leq 2\sqrt{ d}\|\phi(s,a)\|_{(\Lambda^t)^{-1}} \tag{by \pref{assum: linear MDP assumptio}}\\
    &\leq \zeta\|\phi(s,a)\|_{(\Lambda^t)^{-1}}, \\
    |\term_6|&\leq \sup_{s',a', V\in [0,1]^{\calS}} \left|\left(\calT_t V-\calT V\right)(s',a')\right|=\frac{1}{H}c_t \\
    |\term_7|&\leq \sup_{s',a', V\in [0,1]^{\calS}} \left|\left(\calT_t V-\calT V\right)(s',a')\right|=\frac{1}{H}c_t \\
\end{align*}
Finally, 
\begin{align*}
    \term_5\leq  \|\phi(s,a)\|_{(\Lambda^t)^{-1}}\|w^\pi_h\|_{(\Lambda^t)^{-1}}\leq \sqrt{ d}\|\phi(s,a)\|_{(\Lambda^t)^{-1}} \leq \zeta \|\phi(s,a)\|_{(\Lambda^t)^{-1}}. 
\end{align*}
Combining all terms above finishes the proof. 

\end{proof}

\begin{lemma}\label{lem: LSVIUCB optimism}
With probability at least $1-\order(\delta)$,
$V^t_h(s)\geq V^\star_h(s) - 2c_t$ for all $t, h, s$. 
\end{lemma}

\begin{proof}
We use induction to show that with probability at least $1-\order(\delta)$,  $Q_h^t(s,a)\geq Q^\pi_h(s,a) - \frac{2(H+1-h)}{H}c_t$ for all $h, s, a$ and any $\pi$. Consider the case $h=H$, 
\begin{align*}
    &Q^t_H(s,a) \\
    &= \min\left\{w_H^{t\top}\phi(s,a) + \left(4\zeta + C^{\rms}{\sqrt{\frac{d}{Ht}}}\right)\|\phi(s,a)\|_{(\Lambda^t)^{-1}}, \ \ 1\right\} \\
    &= \min\left\{w_H^{\pi\top}\phi(s,a) + \left(4\zeta + C^{\rms}{\sqrt{\frac{d}{Ht}}}\right)\|\phi(s,a)\|_{(\Lambda^t)^{-1}} + \varepsilon^t_H(s,a), \ \ Q^{\pi}_H(s,a)\right\}  \tag{by \pref{lem: linear MDP decomposition}} \\
    &\geq \min\left\{Q^{\pi}_H(s,a) -  \frac{2}{H}c_t, \ \ Q^{\pi}_H(s,a)\right\} \tag{by \pref{lem: linear MDP decomposition}} \\
    &= Q^{\pi}_H(s,a) -  \frac{2}{H}c_t. 
\end{align*}
Suppose that the induction hypothesis holds for $h+1$, then
\begin{align*}
    &Q^t_h(s,a) \\
    &= \min\left\{w_h^{t\top}\phi(s,a) + \left(4\zeta + C^{\rms}{\sqrt{\frac{d}{Ht}}}\right)\|\phi(s,a)\|_{(\Lambda^t)^{-1}}, \ \ 1\right\} \\
    &= \min\left\{w_h^{\pi\top}\phi(s,a) + \E_{s'\sim p_t(\cdot|s,a)}[V^t_{h+1}(s')-V^{\pi}_{h+1}(s')] + \left(4\zeta + C^{\rms}{\sqrt{\frac{d}{Ht}}}\right)\|\phi(s,a)\|_{(\Lambda^t)^{-1}} + \varepsilon^t_h(s,a), \ \ 1\right\} \\
    &\geq \min\left\{Q_h^{\pi}(s,a) + \E_{s'\sim p_t(\cdot|s,a)}[V^t_{h+1}(s')-V^{\pi}_{h+1}(s')] - \frac{2}{H}c_t, \ \ Q^{\pi}_h(s,a)\right\}  \tag{by \pref{lem: linear MDP decomposition}} 
\end{align*}
Notice that by the induction hypothesis, we have for any $s$, $V^{\pi}_{h+1}(s) \leq \max_a Q^{\pi}_{h+1}(s,a)\leq \max_{a}Q^t_{h+1}(s,a) + \frac{2(H-h)}{H}c_t=V^t_{h+1}(s) + \frac{2(H-h)}{H}c_t$. Therefore, the last expression can further be lower bounded by 
\begin{align*}
    &\min\left\{Q_h^{\pi}(s,a) - \frac{2(H-h)}{H}c_t - \frac{2}{H}c_t, \ \ Q^{\pi}_h(s,a)\right\} \\
    &=Q_h^{\pi}(s,a) - \frac{2(H-h+1)}{H}c_t,  
\end{align*}
which finishes the induction. Note that $Q_h^t(s,a)\geq Q^\pi_h(s,a) - \frac{2(H-h+1)}{H}c_t$ implies the lemma since 
\begin{align*}
    V^{t}_h(s) 
    &= \max_a Q^t_h(s,a) \\
    &\geq \max_a Q^\star_h(s,a) -  \frac{2(H-h+1)}{H}c_t   \tag{let $\pi=\pi^\star$}\\
    &\geq V^{\star}_h(s) - 2c_t. 
\end{align*}
\end{proof}

\begin{lemma}\label{lem: robust LSVI bound}
     Robust OFUL / Robust LSVI-UCB ensures with probability at least $1-\order(\delta)$
    \begin{align*}
        \sum_{\tau=1}^t \left(r^{\pistar}_\tau- r_\tau\right) = \otil\left(\zeta\sqrt{dHt} + dC^\rms\right). 
    \end{align*}    
\end{lemma}

\begin{proof}
Note that for all $t, h$, 
\begin{align}
    &\sum_{\tau=1}^t \left(V^\tau_h(s^\tau_h) - V^{\pi_\tau}_h(s^\tau_h)\right) = \sum_{\tau=1}^t \left(Q^\tau_h(s^\tau_h, a^\tau_h) - Q^{\pi_\tau}_h(s^\tau_h, a^\tau_h)\right)   \nonumber \\
    &\leq \sum_{\tau=1}^t \phi^\tau_h(w^\tau_h - w^{\pi_\tau}_h) + \otil\left(\sum_{\tau=1}^t \left(\zeta + C^{\rms}{\sqrt{\frac{d}{H\tau}}}\right) \|\phi^\tau_h\|_{(\Lambda^\tau)^{-1}} \right)  \nonumber \\
    &\leq \sum_{\tau=1}^t \E_{s'\sim p_\tau(\cdot|s^\tau_h, a^\tau_h) }\left[V^\tau_{h+1}(s') - V^{\pi_\tau}_{h+1}(s')\right]   \nonumber \\
    &\qquad \quad  + \otil\left(\sum_{\tau=1}^t \left(\zeta + C^{\rms}{\sqrt{\frac{d}{H\tau}}}\right) \|\phi^\tau_h\|_{(\Lambda^\tau)^{-1}} +  \frac{1}{H}\sum_{\tau=1}^t c_\tau\right)    \tag{by \pref{lem: linear MDP decomposition}} \\
    &= \sum_{\tau=1}^t \left(V^\tau_{h+1}(s^\tau_{h+1}) - V^{\pi_\tau}_{h+1}( s^\tau_{h+1})\right)  \nonumber   \\
    &\qquad \quad + \otil\left(\sum_{\tau=1}^t \left(\zeta + C^{\rms}{\sqrt{\frac{d}{H\tau}}}\right) \|\phi^\tau_h\|_{(\Lambda^\tau)^{-1}} +  \frac{1}{H}\sum_{\tau=1}^t c_\tau + \sum_{\tau=1}^t \epsilon^\tau_h\right), \label{eq: LSVI reget analysis}
\end{align}
where we define $\epsilon^\tau_h\triangleq \E_{s'\sim p_\tau(\cdot|s^\tau_h, a^\tau_h) }\left[V^\tau_{h+1}(s') - V^{\pi_\tau}_{h+1}(s')\right] - \left(V^\tau_{h+1}(s^\tau_{h+1}) - V^{\pi_\tau}_{h+1}( s^\tau_{h+1})\right)$. 
Thus, 
\begin{align*}
    \sum_{\tau=1}^t \left(V^\star_1(s^\tau_1) - V^{\pi_\tau}_1(s^\tau_1)\right) 
    &\leq \sum_{\tau=1}^t \left(V^\tau_1(s^\tau_1) - V^{\pi_\tau}_1(s^\tau_1)\right)  + \order\left(\sum_{\tau=1}^t c_\tau\right)  \tag{by \pref{lem: LSVIUCB optimism}} \\
    &\leq \otil\left(\sum_{\tau=1}^t \sum_{h=1}^H  \left(\zeta + C^{\rms}{\sqrt{\frac{d}{Ht}}}\right) \|\phi^\tau_h\|_{(\Lambda^\tau)^{-1}} + \sum_{\tau=1}^t c_\tau + \sum_{\tau=1}^t \sum_{h=1}^H \epsilon^\tau_h\right) \tag{by \pref{eq: LSVI reget analysis}} \\
    &\leq \otil\left( \sqrt{\sum_{\tau=1}^t \sum_{h=1}^H \left(\zeta^2 + {\frac{(C^{\rms})^2d}{Ht}}\right) }\sqrt{\underbrace{\sum_{\tau=1}^{t}\sum_{h=1}^H \|\phi^\tau_h\|^2_{(\Lambda^\tau)^{-1}}}_{=\otil(d)}}  + C^\am + \sqrt{Ht}\right)  \tag{Cauchy-Schwarz and Azuma's inequality}\\
    &\leq \otil\left( \zeta\sqrt{dHt} + dC^{\rms}\right). 
\end{align*}
Finally, by Azuma's inequality, we get 
\begin{align*}
    \sum_{\tau=1}^t \left(r^{\pistar}_\tau- r_\tau\right) = \otil\left(\zeta\sqrt{dHt} + dC^\rms\right). 
\end{align*}
\end{proof}
\begin{theorem}
    For linear contextual bandits, \algname with Robust OFUL as the base algorithm guarantees $\Reg(T)=\otil\left(d\sqrt{T} + dC^\rms \right)$. For linear MDPs, \algname with Robust LSVI-UCB as the base algorithm guarantees $\Reg(T)=\otil\left(\sqrt{d^3 HT} + H\sqrt{T} + dC^\rms\right)$. 
\end{theorem}
\begin{proof}
    By \pref{lem: robust LSVI bound}, we see that Robust OFUL (with $\zeta=\widetilde{\Theta}(\sqrt{d})$) satisfies \pref{eq: typeone} with $\beta_1=\widetilde{\Theta}(d^2), \beta_2=\widetilde{\Theta}(d), \beta_3=\Theta(1)$. Using them in \pref{thm: form 1 regret} gives the desired bound for linear contextual bandits. For Robust LSVI-UCB, we pick $\zeta=\widetilde{\Theta}(d)$ and thus $\beta_1=\widetilde{\Theta}(d^3H), \beta_2=\widetilde{\Theta}(d),   \beta_3=\Theta(1)$. Using them in \pref{thm: form 1 regret} together with the fact that $c_{\max}=\order(H)$, we get the desired bound for linear MDPs. 
\end{proof}

\paragraph{Comparison with previous bounds } For linear contextual bandits, the bound of \citep{foster2021adapting} is $\otil\left(d\sqrt{T} + \sqrt{d}C^\rms\right)$. For linear MDPs, the bound of \citep{lykouris2019corruption} is $\poly(H)\times \otil\left(C^\am \sqrt{(d^3+dA)T} + (C^\am)^2\sqrt{dT}\right)$.

\subsection{Robust-VOFUL for linear contextual bandits / Robust-VARLin for linear MDPs}\label{app: robust VOFUL}
In this section, we develop a variant of the algorithm of \citep{zhang2021varianceaware} that is robust to corruption (\pref{alg: robust voful linear mdp}). Notice that their original algorithm is for a different linear model called linear mixture MDP, but we carry the similar idea to the linear MDP setting. Again, the same algorithm works for linear contextual bandits. 

\begin{algorithm}[t]
    \caption{Robust VOFUL / Robust VARLin}   \label{alg: robust voful linear mdp}
    \textbf{input}: $C^\am$ \\
    \textbf{define}: $\ell_j\triangleq 2^{-j}$ and $\clip_j(v) \triangleq \max(\min(v, \ell_j), -\ell_j)$. Let $\calB(r)$ be Euclidean ball with radius $r$. \\
    \For{$t=1, 2, \ldots, T$}{
\begin{align*}
    \calW^{t} 
    &= \Bigg\{w=(w_1, w_2, \ldots, w_H)\in \mathcal{B}(\sqrt{d})^H:~~  \\
    & \left| \sum_{\tau=1}^{t-1}\clip_j\left( (\phi^\tau_h)^\top \xi \right)\bigg( \left(\phi^\tau_h\right)^\top w_h - \sigma^\tau_h -  V_{h+1}(w_{h+1})(s^\tau_{h+1})\bigg) \right| \\
    &\qquad \qquad \qquad \qquad \leq 200 \ell_j\left( \sqrt{dHt\log(dTH/\delta)} + C^{\am}\right) \\
    & \forall \xi\in \mathcal{B}\left(2\sqrt{d}\right), \ \ \forall j\in\left[\left\lceil \log_2 T\right\rceil\right], \ \ \forall h\in[H]\Bigg\}
\end{align*} 
where $Q_h(w_h)(s,a)\triangleq w_h^\top \phi(s,a)$ and $V_h(w_h)(s)\triangleq \max_a Q_h(w_h)(s,a)$. 

Let 
\begin{align*}
    w^t = \argmax_{w\in\calW^t} V_1(w_1)(s^t_1),  
\end{align*}
and define $
    Q_h^t(s,a) \triangleq Q_h(w_h^t)(s,a)$ and 
    $V_h^t(s)\triangleq V_h(w_h^t)(s)
$. \\ 
\For{$h=1, \ldots, H$}{
    Observe $s^t_h$, choose $a^t_h=\argmax_a Q^t_h(s^t_h, a)$, and observe $\sigma^t_h$. 
}

}
\end{algorithm}

\begin{lemma}\label{lem: wstar in confidence set}
With probability at least $1-\delta$, the following holds for all $t\in[T]$, $h\in[H]$, $j\in [\left\lceil\log_2 T\right\rceil]$, $\xi\in \mathcal{B}\big(2\sqrt{d}\big)$, $w_{h+1}\in \mathcal{B}\big(\sqrt{d}\big)$: 
\begin{align}
    &\left|\sum_{\tau=1}^{t-1} \clip_j\left(\phi^{\tau\top}_h \xi \right) \left(\phi^{\tau\top}_h \left(\rho + \int \nu(s')V(w_{h+1})(s')\mathrm{d}s' \right) - \left(\sigma^\tau_h + V(w_{h+1})(s_{h+1}^\tau)\right)\right)\right| \nonumber \\
    &\qquad \quad \leq 200\ell_j\left(\sum_{\tau=1}^{t-1}c_\tau + \sqrt{dt\log(dHT/\delta)}\right)
    \label{eq: general bound for VOFUL}
\end{align} 
\end{lemma}
\begin{proof}
   For a fixed tuple of $t, h, j, \xi, w_{h+1}$, recall that $\E[\sigma^\tau_h|s^\tau_h, a^\tau_h] = \sigma_\tau(s^\tau_h, a^\tau_h)$ and $s^\tau_{h+1}\sim p_\tau(\cdot|s^\tau_h, a^\tau_h)$. Therefore, 
    \begin{align*}
        &\left|\E\left[ \clip_j\left(\phi^{\tau\top}_h \xi \right) \left(\phi^{\tau\top}_h \left(\rho + \int \nu(s')V(w_{h+1})(s')\mathrm{d}s' \right) - \left(\sigma^\tau_h + V(w_{h+1})(s_{h+1}^\tau)\right)\right)~\bigg|~  s^\tau_h, a^\tau_h \right]\right| \\
        &=  \left|\clip_j\left(\phi^{\tau\top}_h \xi \right)\left(\sigma_\tau\left(s^\tau_h, a^\tau_h\right) + \E_{s'\sim p_\tau(\cdot|s^\tau_h, a^\tau_h)}V(w_{h+1})(s') - \sigma\left(s^\tau_h, a^\tau_h\right) + \E_{s'\sim p(\cdot|s^\tau_h, a^\tau_h)}V(w_{h+1})(s')\right)\right| 
        \\
        &\leq \ell_j c_\tau. 
    \end{align*}
    By Azuma's inequality, for a fixed tuple $(t, h, j, \xi, w_{h+1})$, with probability at least $1-\delta'$, 
    \begin{align}       
        &\left|\sum_{\tau=1}^{t-1} \clip_j\left(\phi^{\tau\top}_h \xi \right) \left(\phi^{\tau\top}_h \left(\rho + \int \nu(s')V(w_{h+1})(s')\mathrm{d}s' \right) - \left(\sigma^\tau_h + V(w_{h+1})(s_{h+1}^\tau)\right)\right)\right|\nonumber \\
        &\qquad \quad \leq \ell_j \left(\sum_{\tau=1}^{t-1}c_\tau + 2\sqrt{t\log(T/\delta')}\right). \label{eq: VOFUL wstar}
    \end{align}
    Next, we take a union bound for \pref{eq: VOFUL wstar} over $t\in[T]$, $h\in[H]$, $j\in[\lceil\log_2 T\rceil]$, and $\xi, w_{h+1}$ in an $\frac{\ell_j}{2T}$-cover of $\mathcal{B}(2\sqrt{d})$ and $\calB(\sqrt{d})$ respectively. By \citep{wu2016}, the $\epsilon$-covering number of a $d$-dimensional unit ball is upper bounded by $(3/\epsilon)^d$. Therefore, we get that with probability at least \sloppy$1-TH\lceil\log_2 T\rceil \left(3\times 4\sqrt{d}T/\ell_j\right)^{2d}\delta'\geq 1-HT^2(12dT^2)^{2d}\delta'$, \pref{eq: VOFUL wstar} holds for all possible $t,h,j$, and $\xi, w_{h+1}$ in the $\frac{\ell_j}{2T}$-cover. 
    
    Therefore, for all possible $t, h, j, \xi$, and $w_{h+1}$, with probability at least $1-HT^2(12dT^2)^{2d}\delta'$, the left-hand side of \pref{eq: VOFUL wstar} is upper bounded by 
    \begin{align*}
        \ell_j\left(\sum_{\tau=1}^{t-1}c_\tau + 2\sqrt{t\log(T/\delta')}\right) + \frac{\ell_j}{2T}\times t\times 2 + \ell_j \times \frac{\ell_j}{2T}\times t \times 2 \leq \ell_j\left(\sum_{\tau=1}^{t-1}c_\tau + 4\sqrt{t\log(T/\delta')}\right)
    \end{align*}
    where we use the fact that $|\clip_j(\phi^{\tau\top}_h \xi) - \clip_j(\phi^{\tau\top}_h \xi')|\leq |\phi^\tau_h(\xi-\xi')|\leq \|\xi-\xi'\|$ and $|V(w_{h+1})(s)-V(w_{h+1}')(s)|=|\max_a w_{h+1}^\top \phi(s,a) - \max_{a} w_{h+1}'^\top \phi(s,a)| \leq \|w_{h+1}-w_{h+1}'\|. $
    Choosing $\delta'=\delta\Big/ \left(T^2H(12dT^2)^{2d}\right)$ finishes the proof. 
\end{proof}

\begin{corollary}
    With probability at least $1-\delta$, $w^\star\in \calW^t$ for all $t$. 
\end{corollary}
\begin{proof}
    It suffices to show that with probability at least $1-\delta$, 
    \begin{align*}
        &\left|\sum_{\tau=1}^{t-1} \clip_j\left(\phi^{\tau\top}_h \xi \right) \left(\phi^{\tau\top}_h w_h^{\star} - \left(\sigma^\tau_h + V(w_{h+1}^\star)(s_{h+1}^\tau)\right)\right)\right|  
        \leq 200\cdot \ell_j\left(\sum_{\tau=1}^{t-1}c_\tau + \sqrt{dt\log(dHT/\delta)} \right)
    \end{align*}
    for all $t, h, \xi, j$. This can be obtained by \pref{lem: wstar in confidence set} with the fact that $w^\star_h = \rho + \int \nu(s')V_{h+1}^\star(s')\mathrm{d}s' = \rho + \int \nu(s')V(w_{h+1}^\star)(s')\mathrm{d}s'$.
\end{proof}

\begin{definition}
    $\xi^t_h \triangleq w^t_h - \left(\rho + \int \nu(s')V_{h+1}^t(s')\mathrm{d}s'\right)$. 
\end{definition}

\begin{lemma}\label{lem: useful linear MDP}
With probability at least $1-\delta$, the following holds for all $t, h$, and $j$: 
\begin{align*}
    \sum_{\tau=1}^{t-1} \clip_j\left(\phi^{\tau\top}_h \xi^t_h\right) \phi^{\tau\top}_h \xi^t_h \leq 400 \cdot \ell_j\left( \sqrt{dt\log(dTH/\delta)} + C^\am\right). 
\end{align*}
\end{lemma}

\begin{proof}
By the definition of $\xi^t_h$, 
\begin{align*}
    &\sum_{\tau=1}^{t-1} \clip_j\left(\phi^{\tau\top}_h \xi^t_h\right) \phi^{\tau\top}_h \xi^t_h \\
    &= \sum_{\tau=1}^{t-1} \clip_j\left(\phi^{\tau\top}_h \xi^t_h \right) \bigg(\phi^{\tau\top}_h w^t_h - \sigma^\tau_h -  V_{h+1}^t(s^\tau_{h+1})\bigg)\\
    &\quad \qquad + \sum_{\tau=1}^{t-1} \clip_j\left(\phi^{\tau\top}_h \xi^t_h \right)\bigg(\sigma^\tau_h + V_{h+1}^t(s^\tau_{h+1}) - \phi^{\tau\top}_h\left(\rho + \int \nu(s')V_{h+1}^t(s')\mathrm{d}s'\right) \bigg) \\
    &\leq  400\ell_j\left(C^\am +  \sqrt{dt\log(dTH/\delta)}\right)
\end{align*}
where we use the fact that $w^t\in \calW^t$, and \pref{lem: wstar in confidence set} with $w=w^t$. 
\end{proof}

\begin{lemma}
With probability at least $1-\order(\delta)$, 
\begin{align*}
    \sum_{\tau=1}^t \left( r^{\pistar}_\tau - r_\tau \right) \leq \otil\left(Hd^{4.5}\sqrt{t} + Hd^4 C^\am \right). 
\end{align*}
\end{lemma}

\begin{proof}
    Notice that 
    \begin{align}
         V^\star_1(s^t_1) - V^{\pi_t}(s^t_1) \leq V^t_1(s^t_1)  - V^{\pi_t}(s^t_1)
         = \phi_1^{t\top}(w_1^t - w_1^{\pi_t}). \label{eq: regret optimism} 
    \end{align}
    where in the first equality we use the optimism of $w_1^t$. 
    For any $h$, 
    \begin{align}
        &\phi^{t\top}_h(w^t_h - w^{\pi_t}_h) \nonumber \\
        &= \phi^{t\top}_h w^t_h - \phi^{t\top}_h \rho - \phi^{t\top}_h \int \nu(s')V_{h+1}^{\pi_t}(s')\mathrm{d}s' \nonumber \\
        &= \phi^{t\top}_h \xi^t_h + \phi^{t\top}_h \int \nu(s')\left(V_{h+1}^{t}(s')-V_{h+1}^{\pi_t}(s')\right)\mathrm{d}s' \nonumber \\
        &= \phi^{t\top}_h \xi^t_h+ \E_{s'\sim p(\cdot|s^t_h, a^t_h)}\left[V_{h+1}^t(s')-V_{h+1}^{\pi_t}(s')\right] \nonumber \\
        &\leq \phi^{t\top}_h \xi^t_h + \E_{s'\sim p_t(\cdot|s^t_h, a^t_h)}\left[V_{h+1}^t(s')-V_{h+1}^{\pi_t}(s')\right] + \frac{2}{H}c_t \nonumber \\
        &= \phi^{t\top}_h \xi^t_h + \E\left[ V^t_{h+1}(s^t_{h+1}) - V^{\pi_t}_{h+1}(s^t_{h+1})~\big|~s^t_h, a^t_h \right] + \frac{2}{H}c_t  \nonumber \\
        &= \phi^{t\top}_h \xi^t_h + \frac{2}{H}c_t + \E\left[\phi^t_{h+1}(w^t_{h+1} - w^{\pi_t}_{h+1})~\big|~s^t_h, a^t_h\right] \label{eq: recursion VOFUL}
    \end{align}
    where in the inequality we use 
    \begin{align*}
         &\E_{s'\sim p(\cdot|s^t_h, a^t_h)}\left[V_{h+1}^t(s')-V_{h+1}^{\pi_t}(s')\right]  \\
         &=  \E_{s'\sim p_t(\cdot|s^t_h, a^t_h)}\left[V_{h+1}^t(s')-V_{h+1}^{\pi_t}(s')\right]\\
         &\qquad +\left(\sigma(s^t_h, a^t_h) + \E_{s'\sim p(\cdot|s^t_h, a^t_h)}\left[V^t_{h+1}(s')\right]-\sigma_t(s^t_h, a^t_h) - \E_{s'\sim p_t(\cdot|s^t_h, a^t_h)}\left[V^t_{h+1}(s')\right]\right)\\
         &\qquad - \left(\sigma(s^t_h, a^t_h) + \E_{s'\sim p(\cdot|s^t_h, a^t_h)}\left[ V^{\pi}_{h+1}(s')\right]-\sigma_t(s^t_h, a^t_h) - \E_{s'\sim p_t(\cdot|s^t_h, a^t_h)}\left[ V^{\pi}_{h+1}(s')\right]\right) \\
         &\leq \E_{s'\sim p_t(\cdot|s^t_h, a^t_h)}\left[V_{h+1}^t(s')-V_{h+1}^{\pi_t}(s')\right] + 2\sup_{s',a'}\sup_{V\in [0,1]^{\calS}}|\calT_t V(s',a') - \calT V(s',a')| \\
         &= \E_{s'\sim p_t(\cdot|s^t_h, a^t_h)}\left[V_{h+1}^t(s')-V_{h+1}^{\pi_t}(s')\right] + \frac{2}{H}c_t.
    \end{align*}
    Combining \pref{eq: regret optimism} and \pref{eq: recursion VOFUL}, we get 
    \begin{align*}
        \sum_{\tau=1}^{t-1} \E\left[V_1^\star(s^\tau_1) - V^{\pi_\tau}(s^\tau_1) \right] \leq \sum_{\tau=1}^{t-1}\sum_{h=1}^H  \phi^{\tau\top}_h\xi^\tau_h + 2\sum_{\tau=1}^{t-1}c_\tau
    \end{align*}
    Applying Azuma-Hoeffding's inequality, we further get that with probability at least $1-\order(\delta)$,  
    \begin{align}
        \sum_{\tau=1}^{t-1} \left(r^{\pistar}_\tau - r_\tau\right)\leq \sum_{\tau=1}^{t-1}\sum_{h=1}^H  \phi^{\tau\top}_h\xi^\tau_h + 2\sum_{\tau=1}^{t-1}c_\tau + \otil(\sqrt{t}).  \label{eq: VOFUL key eq}
    \end{align}
    It remains to bound $\sum_{\tau=1}^{t-1}\phi^{\tau\top}_h \xi^\tau_h$ for all $h$: 
    \begin{align*}
        &\sum_{\tau=1}^{t-1}  \phi^{\tau\top}_h \xi^\tau_h \\
        &= \sum_{\tau=1}^{t-1}  \phi^{\tau\top}_h \xi^\tau_h \one[|\phi^{\tau\top}_h \xi^\tau_h| \geq \nicefrac{1}{(2t)}] + \sum_{\tau=1}^{t-1}  \phi^{\tau\top}_h \xi^\tau_h \one[|\phi^{\tau\top}_h \xi^\tau_h| < \nicefrac{1}{(2t)}] \\
        &\leq \sum_{\tau=1}^{t-1} \phi^{\tau\top}_h \xi^\tau_h \times  \frac{\sum_{s=1}^{\tau-1} \clip_{j_\tau}\left(\phi^{s\top}_h \xi^\tau_h\right) \phi^{s\top}_h \xi^\tau_h + \ell_{j_\tau}}{\sum_{\tau=1}^{\tau-1} \clip_{j_\tau}\left(\phi^{s\top}_h \xi^\tau_h\right) \phi^{s\top}_h \xi^\tau_h +  \ell_{j_\tau}} + 1 \tag{$j_\tau$ is such that $\frac{1}{2}\ell_{j_\tau} \leq |\phi^{\tau\top}_h \xi^\tau_h| \leq \ell_{j_\tau}$} \\
        &\leq \sum_{\tau=1}^{t-1}  |\phi^{\tau\top}_h \xi^\tau_h|  \times \frac{ \ell_{j_\tau}  \times \otil\left(\sqrt{d\tau} +C^\am \right)}{\sum_{s=1}^{\tau-1} \clip_{j_\tau}\left(\phi^{s\top}_h \xi^\tau_h\right) \phi^{s\top}_h \xi^\tau_h + \ell_{j_\tau}} + 1\tag{\pref{lem: useful linear MDP}} \\
        &\leq \left(\sum_{\tau=1}^{t-1} \frac{2\left(\clip_{j_\tau} \left(\phi^{\tau\top}_h \xi^\tau_h\right)\right)^2 }{\sum_{s=1}^{\tau-1} \clip_{j_\tau}\left(\phi^{s\top}_h \xi^\tau_h\right) \phi^{s\top}_h \xi^\tau_h + \ell_{j_\tau}}\right) \times \otil\left(\sqrt{dt} + C^\am\right) + 1\\
        &\leq \otil(d^4)  \times \otil\left(\sqrt{dt}  + C^\am\right) \tag{by \pref{lem: zhang et al. lemma}} 
    \end{align*}
    Combining this with \pref{eq: VOFUL key eq} finishes the proof. 
\end{proof}

\begin{lemma}[Lemma 20 of \citep{zhang2021varianceaware}]\label{lem: zhang et al. lemma}
    \begin{align*}
        \sum_{\tau=1}^{t-1} \frac{2\left(\clip_{j_\tau} \left(\phi^{\tau\top}_h \xi^\tau_h\right)\right)^2 }{\sum_{s=1}^{\tau-1} \clip_{j_\tau}\left(\phi^{s\top}_h \xi^\tau_h\right) \phi^{s\top}_h \xi^\tau_h + \ell_{j_\tau}} \leq \order\left(d^4 \log^3(t)\right). 
    \end{align*}    
\end{lemma}

\begin{theorem}
    For linear contextual bandits, \algname with Robust VOFUL as the base algorithm guarantees $\Reg(T)=\otil\left(d^{4.5}\sqrt{T} + d^4C^\am \right)$. For linear MDPs, \algname with Robust VARLin as the base algorithm guarantees $\Reg(T)=\otil\left(Hd^{4.5}\sqrt{T} + Hd^{4}C^\am\right)$. 
\end{theorem}
\begin{proof}
    By \pref{lem: robust LSVI bound}, we see that Robust OFUL satisfies \pref{eq: typeone} with $\beta_1=\widetilde{\Theta}(d^9), \beta_2=\widetilde{\Theta}(d^4), \beta_3=\Theta(1)$. Using them in \pref{thm: form 1 regret} gives the desired bound for linear contextual bandits. For Robust LSVI-UCB,  $\beta_1=\widetilde{\Theta}(d^9H^2), \beta_2=\widetilde{\Theta}(d^4H),  \beta_3=\Theta(1)$. Using them in \pref{thm: form 1 regret}, we get the desired bound for linear MDPs. 
\end{proof}

\subsection{Robust GOLF}\label{app: robust golf}
In this section, we adapt the GOLF algorithm by \citet{jin2021bellman} to the corruption setting. For simplicity, we assume that the function class $\calF$ is finite (the extension to infinite case is straightforward through a discretization step, as shown in \citep{jin2021bellman}). The algorithm is presented in \pref{alg:golf}. 
\begin{algorithm}
\caption{Robust GOLF}
\label{alg:golf}
\textbf{input}: $C^\rms$ \\
\textbf{parameter}: $\zeta=16\log(TH|\calF|/\delta)$.  \\
\textbf{Initialize:} $\mathcal B^1 \gets \mathcal F$\\
\For{$t = 1, 2, \dots, T$}{
Choose policy: $\pi^t = \pi_{f^t}$, where $f^t \in \argmax_{f \in \mathcal B^{t}} f(s_1, \pi_f(s_1))$\\
Collect a trajectory $(s^t_1, a^t_1, \sigma^t_1, \dots, s^t_H, a^t_H, \sigma^t_H, s^t_{H+1})$ by following $\pi^t$.\\
\text{Update}
\begin{align}
    \mathcal B^{t+1}& = \left\{ f \in \mathcal F \colon \mathcal L^t_h(f_h, f_{h+1}) \leq \inf_{g \in \mathcal F_h} \mathcal L^t_h(g, f_{h+1}) + \left(\zeta + \frac{2C^\rms}{H^2 t}\right) \textrm{ for all } h \in [H] \right\},\\
    & \textrm{where \ \ } \mathcal L^t_h(f_h, f_{h+1}) = \sum_{\tau=1}^t \left(f_h(s^\tau_h,a^\tau_h) - \sigma^\tau_h - \max_{a'} f_{h+1}(s^\tau_{h+1}, a')\right)^2. 
\end{align}
}

\end{algorithm}

\begin{lemma}[\textit{c.f.} Lemma 39 of \citep{jin2021bellman}]\label{lem: X_t bound}
    With probability at least $1-\delta$, we have 
    \begin{align*}
        &\text{(a)}\qquad  \sum_{\tau=1}^{t-1} \E\left[\left(f^t_h(s_h, a_h) - (\calT f^t_{h+1})(s_h,a_h)\right)^2~\big|~s_h, a_h \sim \pi_\tau \right] \leq \order\left(\zeta + \frac{C^\rms}{H^2 t}\right)  \\
        &\text{(b)}\qquad  \sum_{\tau=1}^{t-1} \left(f^t_h(s_h^\tau, a_h^\tau) - (\calT f^t_{h+1})(s_h^\tau,a_h^\tau)\right)^2  \leq\order\left(\zeta + \frac{C^\rms}{H^2 t}\right) 
    \end{align*}
\end{lemma}
\begin{proof}
    Define for any $h\in[H], g\in\calF$,
    \begin{align*}
        X^t_h(g) 
        \triangleq \left(g_h(s^t_h, a^t_h) - \sigma^t_h - g_{h+1}(s^t_{h+1}, \pi_g(s^t_{h+1}))\right)^2 - \left((\calT g_{h+1})(s^t_h, a^t_h) - \sigma^t_h - g_{h+1}(s^t_{h+1}, \pi_g(s^t_{h+1}))\right)^2.  
    \end{align*}
    Then we have 
    \begin{align}
        &X^t_h(g)  \nonumber \\
        &= (g_h(s^t_h, a^t_h)-(\calT g_{h+1})(s^t_h, a^t_h))^2   \nonumber\\
        &\ \   + 2\left((\calT g_{h+1})(s^t_h, a^t_h) - \sigma^t_h - g_{h+1}(s^t_{h+1}, \pi_g(s^t_{h+1}))\right)\left(g_h(s^t_h, a^t_h)-(\calT g_{h+1})(s^t_h, a^t_h)\right) \nonumber\\
        &= (g_h(s^t_h, a^t_h)-(\calT g_{h+1})(s^t_h, a^t_h))^2 \nonumber\\
        &\ \ + \scalebox{0.95}{$\displaystyle 2\left((\calT g_{h+1})(s^t_h, a^t_h) - \sigma_t(s^t_h, a^t_h) - \E_{s'\sim p_t(\cdot|s^t_h, a^t_h)}\left[g_{h+1}(s', \pi_g(s'))\right]\right)\left(g_h(s^t_h, a^t_h)-(\calT g_{h+1})(s^t_h, a^t_h)\right)$}  \nonumber\\
        &\ \ + \scalebox{0.95}{$\displaystyle\underbrace{2\left( \sigma_t(s^t_h, a^t_h) -\sigma^t_h + \E_{s'\sim p_t(\cdot|s^t_h, a^t_h)}\left[g_{h+1}(s', \pi_g(s'))\right] - g_{h+1}(s^t_{h+1}, \pi_g(s^t_{h+1})) \right)}_{\triangleq ~\epsilon^t_h}$}\left(g_h(s^t_h, a^t_h)-(\calT g_{h+1})(s^t_h, a^t_h)\right)  \nonumber \\
        &\geq \left(g_h(s^t_h, a^t_h)-(\calT g_{h+1})(s^t_h, a^t_h)\right)^2 - \frac{2}{H} c_t \left|g_h(s^t_h, a^t_h)-(\calT g_{h+1})(s^t_h, a^t_h)\right| + \epsilon^t_h \left(g_h(s^t_h, a^t_h)-(\calT g_{h+1})(s^t_h, a^t_h)\right)   \tag{by the definition of $c_t$}\\
        &\geq  \frac{1}{2} \left(g_h(s^t_h, a^t_h)-(\calT g_{h+1})(s^t_h, a^t_h)\right)^2 - \frac{2c_t^2}{H^2} + \epsilon^t_h \left(g_h(s^t_h, a^t_h)-(\calT g_{h+1})(s^t_h, a^t_h)\right)   \tag{AM-GM}   \\
        &\label{eq: GOLD eq 1}
    \end{align}
    Notice that $\epsilon^t_h$ is a zero-mean random variable. 
    By the definition of $\calB^t$ and that $f^t\in\calB^t$, we have 
    \begin{align*}
        &\sum_{\tau=1}^{t-1} X_h^\tau(f^t)\\
        &= \sum_{\tau=1}^{t-1}\left[ \left(f_h^t(s^\tau_h, a^\tau_h) - \sigma^\tau_h - f_{h+1}^t(s^\tau_{h+1}, \pi_{f^t}(s^\tau_{h+1}))\right)^2 - \left(\calT f_{h+1}^t(s^\tau_h, a^\tau_h) - \sigma^\tau_h - f_{h+1}^t(s^\tau_{h+1}, \pi_{f^t}(s^\tau_{h+1}))\right)^2 \right] \\ &\leq \sum_{\tau=1}^{t-1} \left[\left(f_h^t(s^\tau_h, a^\tau_h) - \sigma^\tau_h - f_{h+1}^t(s^\tau_{h+1}, \pi_{f^t}(s^\tau_{h+1}))\right)^2 - \min_{g\in\calF_h}\left(g(s^\tau_h, a^\tau_h) - \sigma^\tau_h - f_{h+1}^t(s^\tau_{h+1}, \pi_{f^t}(s^\tau_{h+1}))\right)^2\right] \tag{by the closeness of $\calF$}\\
        &\leq \zeta + \frac{2C^\rms}{H^2 t}. 
    \end{align*}
    Combining this with \pref{eq: GOLD eq 1}, we get 
    \begin{align*}
        &\frac{1}{2}\sum_{\tau=1}^{t-1} \left(f^t_h(s^\tau_h, a^\tau_h)-(\calT f^t_{h+1})(s^\tau_h, a^\tau_h)\right)^2 \\
        &\leq \sum_{\tau=1}^{t-1}X^\tau_h(f^t) +  \frac{2}{H^2}\sum_{\tau=1}^{t-1} c_\tau^2 - \sum_{\tau=1}^{t-1}\epsilon^\tau_h \left(f_h^t(s^\tau_h, a^\tau_h)-(\calT f^t_{h+1})(s^\tau_h, a^\tau_h)\right)  \\
        &\leq \sum_{\tau=1}^{t-1}X^\tau_h(f^t) + \frac{2C^\rms}{H^2 t} + 2\sqrt{\sum_{\tau=1}^{t-1} \left(f^t_h(s^\tau_h, a^\tau_h)-(\calT f^t_{h+1})(s^\tau_h, a^\tau_h)\right)^2 \log(TH|\calF|/\delta)}   \tag{Freedman's inequality} \\
        &\leq \zeta  + \frac{4C^\rms}{H^2 t} + \frac{1}{4}\sum_{\tau=1}^{t-1} \left(f^t_h(s^\tau_h, a^\tau_h)-(\calT f^t_{h+1})(s^\tau_h, a^\tau_h)\right)^2 + 4\log(TH|\calF|/\delta)   \tag{AM-GM}
    \end{align*}
    The above inequality implies 
    \begin{align*}
        \sum_{\tau=1}^{t-1} \left(f^t_h(s^\tau_h, a^\tau_h)-(\calT f^t_{h+1})(s^\tau_h, a^\tau_h)\right)^2 \leq (4\zeta + 16\log(TH|\calF|/\delta))+ \frac{16C^\rms}{H^2 t}= \order\left(\zeta + \frac{C^\rms}{H^2 t}\right), 
    \end{align*}
    proving (b). (a) can be proven by the same approach (see also \citep{jin2021bellman}). 
\end{proof}

\begin{lemma}
With probability at least $1 - \delta$, the optimal Q-function of the uncorrupted MDP is always feasible, that is, $Q^\star \in \mathcal B^t$ for all $t \in [T]$.
\label{lem:golf_opt}
\end{lemma}
\begin{proof}
Following the proof of Lemma~40 in \citet{jin2021bellman}, we define for any $h \in [H]$, $t \in [T]$ and $g \in \mathcal F$, 
\begin{align*}
    W^t_h(g) &\triangleq \left(g_h(s_h^t,a_h^t)-\sigma_h^t - V^\star(s_{h+1}^t)\right)^2-\left(Q^\star(s_h^t,a_h^t)-\sigma_h^t - V^\star(s_{h+1}^t)\right)^2\\
    &= \left(g_h(s_h^t,a_h^t)-Q^\star(s_h^t,a_h^t)\right)^2 
    +  2\left(Q^\star(s_h^t,a_h^t)-\sigma_h^t - V^\star(s_{h+1}^t)\right)\left(g_h(s_h^t,a_h^t)-Q^\star(s_h^t,a_h^t)\right)\\
    &= (g_h(s_h^t,a_h^t)-Q^\star(s_h^t,a_h^t))^2 \\
    &\qquad +  2\left(Q^\star(s_h^t,a_h^t)-\sigma_t(s^t_h, a^t_h) - \E_{s'\sim p_t(s'|s^t_h,a^t_h)}[V^\star(s')]\right)\left(g_h(s_h^t,a_h^t)-Q^\star(s_h^t,a_h^t)\right) \\
    &\qquad + \underbrace{2\left(\sigma_t(s^t_h, a^t_h) - \sigma^t_h + \E_{s'\sim p_t(s'|s^t_h,a^t_h)}[V^\star(s')] - V^\star(s^t_{h+1})\right)}_{\epsilon^t_h}(g_h(s_h^t,a_h^t)-Q^\star(s_h^t,a_h^t)) \\
    &\geq (g_h(s_h^t,a_h^t)-Q^\star(s_h^t,a_h^t))^2 - \frac{2}{H}c_t|g_h(s_h^t,a_h^t)-Q^\star(s_h^t,a_h^t)| + \epsilon^t_h(g_h(s_h^t,a_h^t)-Q^\star(s_h^t,a_h^t)) \\
    &\geq \frac{1}{2}(g_h(s_h^t,a_h^t)-Q^\star(s_h^t,a_h^t))^2 - \frac{2}{H^2}c_t^2 + \epsilon^t_h(g_h(s_h^t,a_h^t)-Q^\star(s_h^t,a_h^t))  \tag{AM-GM}
\end{align*}
Let $\mathfrak F^t_{h}$ be the sigma-field induced by all samples up to $s_h^t, a_h^t$ (but not $\sigma_h^t$ or $s_{h+1}^t$). Then 
\begin{align}
    \E[W^t_h(g) ~ | ~ \mathfrak F^t_{h}] \geq \frac{1}{2}(g_h(s_h^t,a_h^t)-Q^\star(s_h^t,a_h^t))^2 -\frac{2}{H^2}c_t^2\,.  \label{eq: exp W}
\end{align}
and by the definition of $W^t_h(g)$, the variance is bounded by
\begin{align}
    \E[W^t_h(g)^2~|~ \mathfrak F_{t, h}] 
    \leq  
    4(g_h(s_h^t,a_h^t)-Q^\star(s_h^t,a_h^t))^2\,.  \label{eq: var W}
\end{align}
By Freedman's inequality, we have with probability at least $1-\delta$, 
\begin{align*}
    &\frac{1}{2}\sum_{\tau=1}^{t-1}(g_h(s_h^\tau,a_h^\tau)-Q^\star_h(s_h^\tau,a_h^\tau))^2 - \sum_{\tau=1}^{t-1} W_h^\tau(g)
    \\
    & \leq  \frac{2}{H^2}\sum_{\tau=1}^{t-1} c_\tau^2 +2\sqrt{4\sum_{\tau=1}^{t-1}(g_h(s_h^\tau,a_h^\tau)-Q^\star_h(s_h^\tau,a_h^\tau))^2\log(TH|\calF|/\delta)} \tag{by \pref{eq: exp W} and \pref{eq: var W}}\\
    &\leq \frac{2C^\rms}{H^2 t} + \frac{1}{4}\sum_{\tau=1}^{t-1}(g_h(s_h^\tau,a_h^\tau)-Q^\star_h(s_h^\tau,a_h^\tau))^2 + 16\log(TH|\calF|/\delta)  \tag{AM-GM}
\end{align*}
which implies
\begin{align*}
    -\sum_{\tau=1}^{t-1} W_h^\tau(g) \leq \frac{2C^\rms}{H^2 t} + 16\log(TH|\calF|/\delta).
\end{align*}
This implies that $Q^\star\in\calB^t$. 
\end{proof}

\begin{lemma}\label{lem: GOLF lemma}
    With probability at least $1-\order(\delta)$, 
    \begin{align*}
        \sum_{\tau=1}^t (r_\tau^{\pistar} - r_\tau) = \otil\left(H\sqrt{\zeta\DE\cdot t} + \sqrt{\DE}C^\rms \right).  
    \end{align*}
\end{lemma}
where $\zeta$ is defined in \pref{alg:golf}, and $\DE$ is the Bellman eluder dimension. We refer the reader to \citep{jin2021bellman} for the precise definition of the Bellman eluder dimension.  
\begin{proof}
\begin{align*}
    &\sum_{\tau=1}^t (r_\tau^{\pistar} - r_\tau)\\
    &\leq 
    \sum_{\tau=1}^t (V^{\pistar}(s^\tau_1) - V^{\pi_\tau}(s^{\tau}_1)) + \otil\left(\sqrt{t} + C^{\am}\right)   \tag{Azuma's inequality}\\
    &\leq \sum_{\tau=1}^t (\max_a f^\tau(s^\tau_1, a) - V^{\pi_\tau}(s^{\tau}_1)) + \otil\left(\sqrt{t} + C^{\am}\right) \\
    &\leq \sum_{\tau=1}^t \sum_{h=1}^H \E\left[f^\tau_h(s_h,a_h) - \calT f_{h+1}^{\tau}(s_h,a_h)~\big|~ (s_h,a_h)\sim \pi_\tau\right]  + \otil\left(\sqrt{t} + C^{\am}\right) \tag{by \citep[Eq.(4)]{jin2021bellman}} \\
    &\leq \sum_{h=1}^H \otil\left(\sqrt{\DE \cdot t} \sqrt{\zeta + \frac{1}{H^2}\frac{(C^\rms)^2}{t}}\right) +   \otil\left(\sqrt{t} + C^{\am}\right) \tag{using \pref{lem: X_t bound} together with \citep[Lemma 17]{jin2021bellman}}\\
    &= \otil\left(H\sqrt{\zeta\DE\cdot t} + \sqrt{\DE}C^\rms \right). 
\end{align*}
\end{proof}

\begin{theorem}
    For MDPs with low Bellman-eluder dimension, \algname with Robust GOLF as the base algorithm guarantees $\Reg(T)=\otil\left(H\sqrt{\zeta\DE\cdot T} + \sqrt{\DE}C^\rms \right)$.
\end{theorem}
\begin{proof}
    By \pref{lem: GOLF lemma}, we see that Robust GOLF satisfies \pref{eq: typeone} with $\beta_1=\widetilde{\Theta}(H^2\zeta\DE), \beta_2=\widetilde{\Theta}(\sqrt{\DE}), \beta_3=\Theta(1)$. Using them in \pref{thm: form 1 regret} gives the desired bound.   
\end{proof}

\end{document}